\RequirePackage{fix-cm}
\documentclass[twocolumn]{svjour3}          % twocolumn
\smartqed  % flush right qed marks, e.g. at end of proof
\usepackage{graphicx,amsfonts}
\usepackage{url}
\usepackage{soul,color}
 \usepackage{mathptmx}      % use Times fonts if available on your TeX system
 \usepackage{amssymb}
\newcommand{\nat}{\mathbb{N}}
\newcommand{\erre}{\mathbb{R}}
\usepackage[english]{babel}

\begin{document}

\title{An image structure model for exact edge detection %\thanks{Grants or other notes
}


\author{Alessandro Dal Pal\`u}


\institute{Alessandro Dal Pal\`u \at
              University of Parma,
              Parco Area delle Scienze 53/A
              43124 Parma, Italy \\
              Tel.: +39-0521-906962 
              Fax: +39-0521-906950\\
              \email{alessandro.dalpalu@unipr.it}\\
              ORCID:0000-0003-0353-158X
                         %  \\
}

\date{Received: date / Accepted: date}

\maketitle
%\tableofcontents

\begin{abstract}
The paper presents a new model for single channel images low-level interpretation. The
image is decomposed into a graph which captures a complete set of structural features. 
The description allows to accurately identify every edge location and its correct connectivity. 
The key features of the method are: vector description of the edges, subpixel precision,
and parallelism of the underlying algorithm.
The methodology outperforms classical and state of the art edge detectors at both conceptual and experimental levels. It
also enables graph based algorithms for higher-level feature extraction. Any image processing pipeline can benefit from such results: e.g., controlled denoising, edge preserving filtering, upsampling, compression, vector and graph based pattern matching, neural network training.

\end{abstract}

\section{Introduction}
\label{intro}

Edge detection is one of the most active research fields in computer vision. It represents the first step of image features abstraction and interpretation process. Edges are related to significant changes in image values that witness the presence of higher level properties (object boundaries, change in reflectance, reflections, etc.). The quality of any image processing pipeline relies on the underlying edge detector and on its effectiveness in removing false positives while minimizing false negatives. The 
presence of noise, due to acquisition, compression and filtering, affects the detection. Commonly, a smoothing filter preprocesses the image, in order to reduce the effect of the noise, while preserving the main features. Local edge detection is strictly related to a differential or gradient operator that highlights the presence of strong changes in the image along a specific direction. Eventually, the detection combines them into a set of lines/list of pixels/raster that describes a coherent disposition of each edge according to the maximal change regions of the image. 

In  literature, several surveys on edge detectors cover last decades of research directions~\cite{arbelaez2011contour,papari2011edge,ziou1998edge,shapiro2001computer}. We refer the reader to such references for further details. Let us shortly cover some approaches that are related to our method. The primordial convolutional masks approaches (Roberts~\cite{roberts1963machine}, Prewitt~\cite{prewitt1970object} and Sobel~\cite{duda1973pattern}) computed 
first-order derivatives with kernel convolutions. The maximal intensity pixels in the resulting image indicate the presence of an edge. Such output contains only raster information, which is visually pleasing but of limited use for automatic processing, since no edge labeling and actual interpretation is performed. Canny's edge detector~\cite{canny1987computational} builds on such (smoothed) output, selects local maxima pixels along gradient directions that correspond to edges and connects them through a chaining process. Such detector and its legacy over the last 30 years represent a landmark for edge detection. Canny's detector is optimal with respect to three criteria: it provides a low error rate, a minimal error in predicted position w.r.t. the actual edge and a single response to a single edge.
Many approaches in literature provide as output of edge detection a raster matrix, where edges are identified by pixels, usually at the same resolution of the original image. %Pixels can be linked and for a discrete polyline.

\begin{figure*}[ht]
\begin{center}
\includegraphics[width=0.32\textwidth]{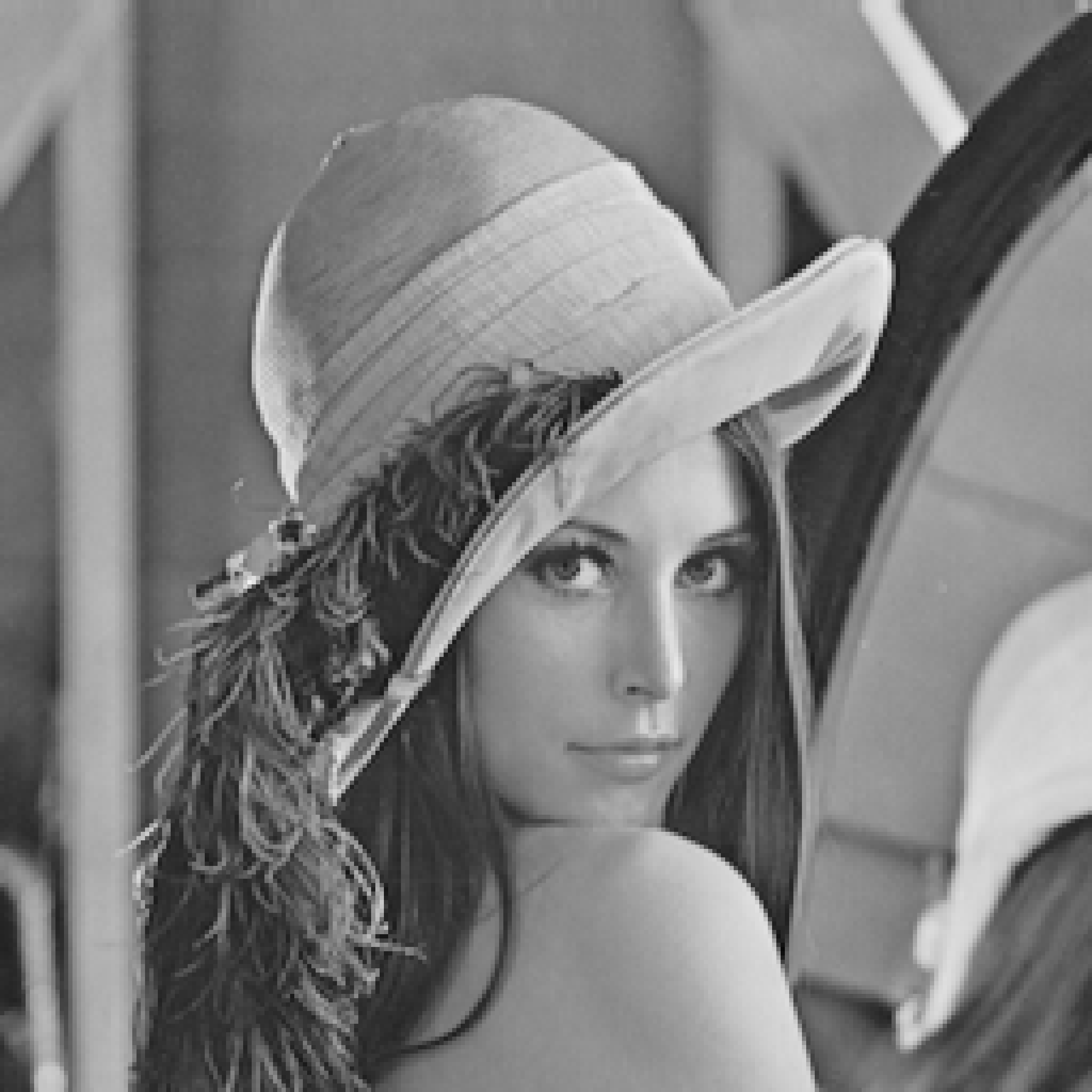}
\includegraphics[width=0.32\textwidth]{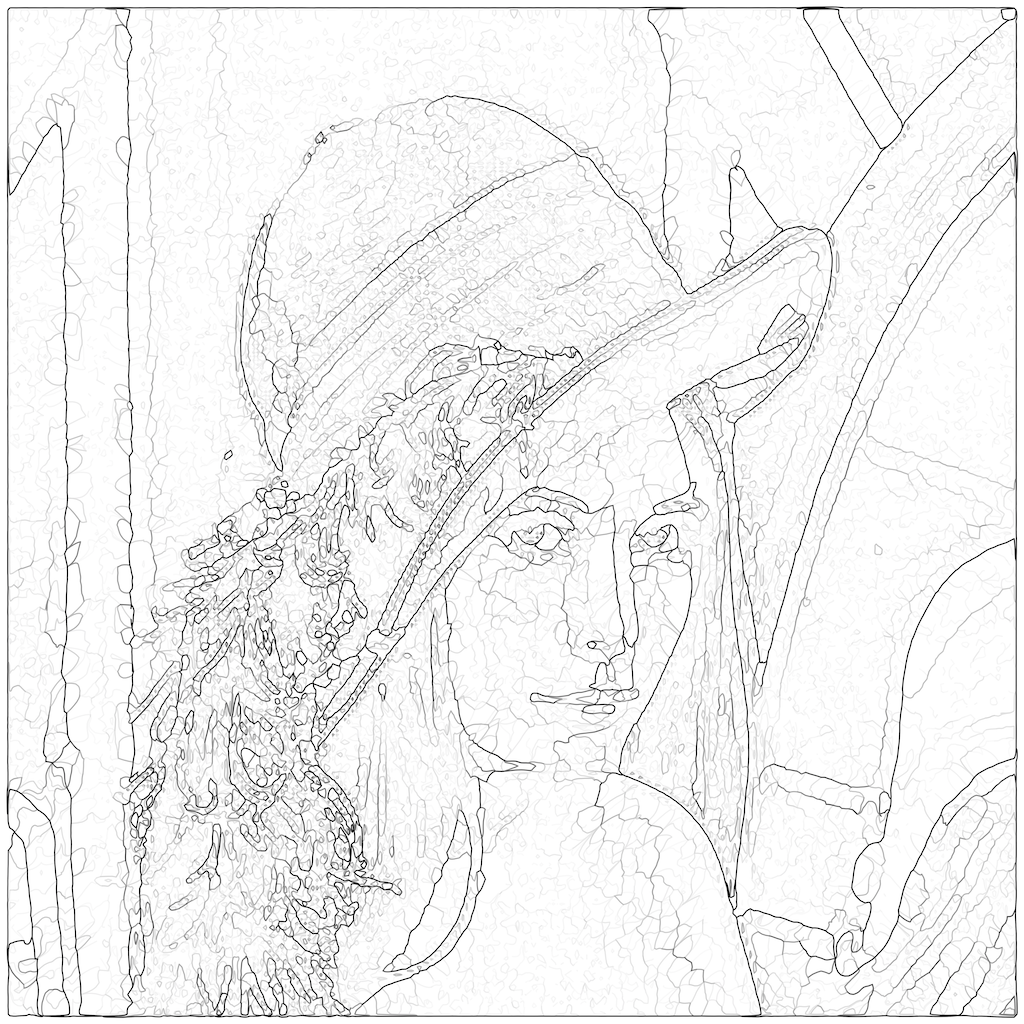}
\includegraphics[width=0.32\textwidth]{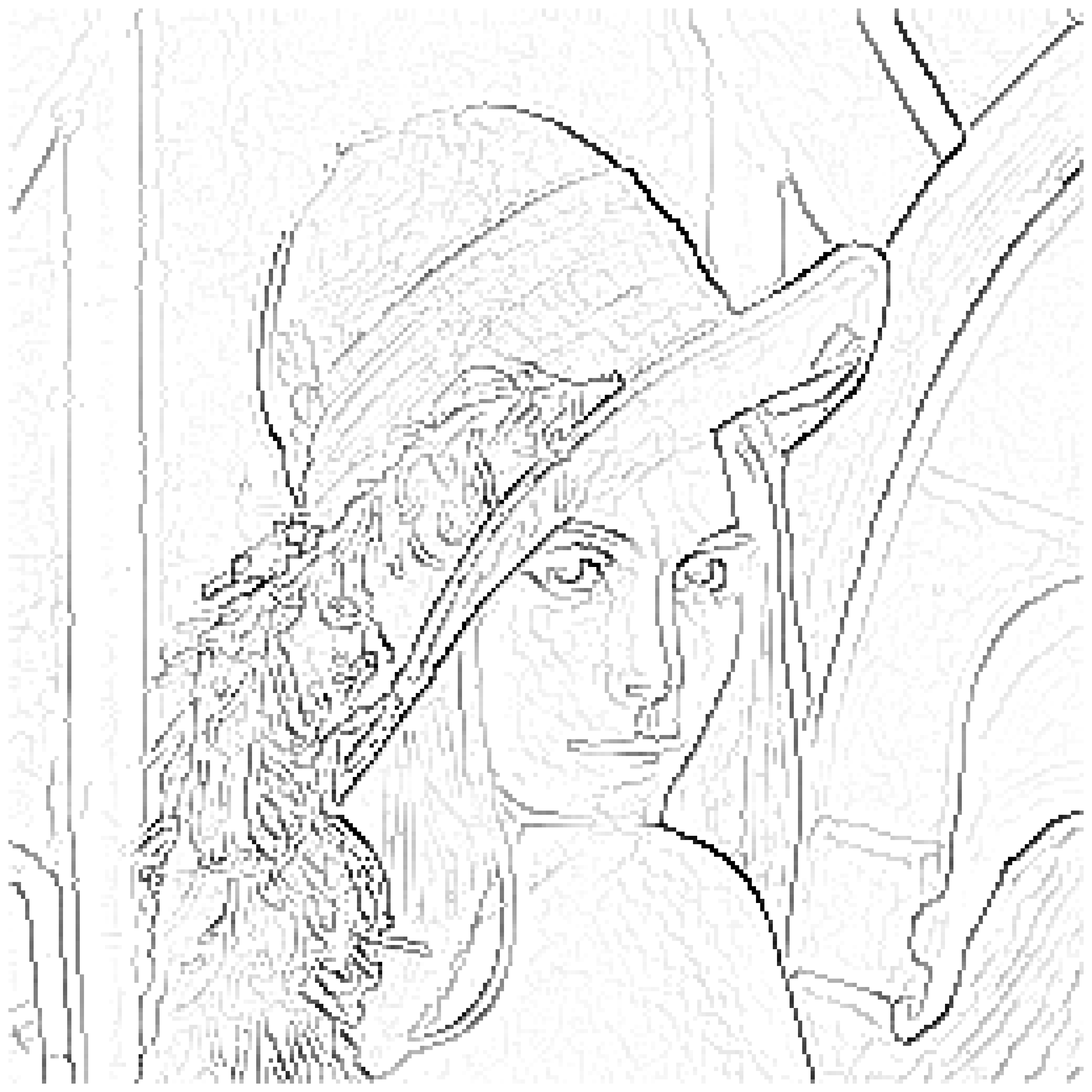}

\caption{256x256 Lena. Original image (left), proposed edge detection (center), Canny's detection (right) \label{fig:test3}}
\end{center}
\end{figure*}

Color and texture information~\cite{cremers2007review,malik2001contour} are considered in the context of image segmentation. Subpixel precision~\cite{hermosilla2008non,jensen1995subpixel} has also been included in the detection process. Global interpretations~\cite{shotton2008multiscale,catte1992image} and statistical approaches~\cite{konishi2003statistical} have been presented. In recent years, the advent of new parallel architectures (GPGPU) allowed a rise of Neural Network based methods~\cite{xie2015holistically,dollar2013structured}, where a training with manually annotated edges is performed. 
Global minimization algorithms (active snakes~\cite{snakes}) have been proposed, where the output is a vector polyline that converges to the edge on a continuous domain. %The only limitation is that in the presence of edge bifurcation, the model is not capable of automatically tracking a complex landscape.

%%%%%%% junctions
Junctions of edges (the location where multiple edge lines intersect and/or sharply bend)
are usually not computed by the low-level edge detection process.
Starting from Harris' corner detector~\cite{harris1988combined}, where abrupt changes in gradient directions are identified, 
significant work aimed at including local and global information for accurate positioning and detection of intersections has been carried on~\cite{ziou1998edge}.
The task is difficult because high quality edge information is required for a correct interpretation of the junction. Moreover, edge strength and gradients become weak at junctions. Contextual information and higher level interpretation~\cite{mcdermott2004psychophysics} have been proposed to increase accuracy.
Finally, perceptual detection of junctions have been proposed by means of ground truth neural network training~\cite{maire2008using}
A comprehensive model for edge junctions did not emerge yet, since it combines perceptual interpretation of the image and context information. As already suggested by~\cite{harris1988combined}, edge and junction detection should be performed by a single detector that takes advantage of both kinds of information.

This paper proposes a new formalization of edges and their connectivity with the goal of overcoming the two common challenges of state-of-the-art detectors: edge tracking and edge connectivity. Edge tracking is the process of identifying all maximal gradient magnitude points (or second derivative zeroes) and linking them in order to follow the edge line across the image. This task is impaired by noise and can easily lead to misinterpretation. Moreover, if performed at discrete level, it is heavily affected by spatial quantization of pixels and gradients. To the best of our knowledge, no global relationships between edges and junctions are modelled in a graph like structure.

The proposed model is based on a new point of view:
 an edge is modelled as a \emph{closed} and \emph{oriented} line that traverses the image and splits and merges into other edges. A single connectivity graph contains all possible edges (or cycles) and describes the correct topology of edges. Rather than tracking the accurate disposition of each edge throughout the image, the focus is to identify a set of key points that must be traversed by the edges (associated to the nodes of the graph) and to deduce the topology of connections by means of the analysis of the relationships between neighbor nodes. Image edges can be drawn, without tracking, as a rendering of graph edges and lines interpolation.
The proposed edge model natively combines edge lines localization and connectivity detection into a single process. 
It also decouples signal raw information from high-level interpretation.
%The image is processed to output edge information and edge graph from signal analysis point of view. 
Higher level interpretations can be built on such graph: e.g., perceptual faithfulness, corner detection, basic geometry (as opposed to Hough transform), classification of edge types, segmentation.

The main features offered by the proposed detector are:

\begin{itemize}
\item{\bf accuracy} Every edge, its position and strenght are detected. 
%, and also the width of the edge (namely how far it takes to the maximal gradient to find its next local minima on both sides). 

\item{\bf connectivity} The detector outputs an edge graph, where
nodes are junctions points and graph edges represent parts of image edge lines. % has no leaves and any edge form a cycle in the graph. 

\item{\bf no preprocessing} No image pre-filtering is required. The edge extraction directly descends  from the image structure and topology.

\item{\bf no thresholds nor parameters} The model is independent on image based thresholds, since it recovers structural properties. %The edge data structure can be further post-processed depending on the perceptual model in use.

\item {\bf vectorial output} The edge features are in vector form, in order to adapt to the finest details of the image. Vectors are oriented, so that directional graph-based algorithms can be used.

\item {\bf subpixel precision} Vector spatial precision should be unlimited.

\item {\bf no prior knowledge} The detector is not informed  on domain related knowledge (i.e. trained neural networks, registration and calibration). 

\item {\bf global} The algorithm performs a global analysis. Local information is not able to capture the potential edge relations with the context. Any application of the algorithm on a subimage may lead to poorer results.

\item {\bf small images} The algorithm is able to accurately process any image size, starting from 2x2 pixel images.

\item {\bf arbitrary image depth} The algorithm accepts any pixel channel depth. In particular, even if continuous values (modeled as floating point) are used, no loss of information (e.g., quantization) is caused be the algorithm.

\item {\bf spatial domain} The algorithm deals with original image information: no transformed spaces are used in the process.

\item {\bf efficiency and parallelism} The detector is compute efficient and the algorithm is suitable for parallel execution.

\item {\bf perceptual correspondence} Edges are drawn by perceptually coherent lines.

\end{itemize}

The paper is organized as follows. In Section~\ref{sec:not} some notation is provided. Section~\ref{sec:structure} presents the structural model that defines image features and
Section~\ref{sec:em} shows how to extract edge information from it. In Section~\ref{sec:res} we present some edge detection examples for a set of significant synthetic and natural images. In Section~\ref{sec:appl} we highlight some of the potential benefits of such method for most demanding applications.
Finally in Section~\ref{sec:conc} we draw the conclusions and final remarks.

\section{Notation and definitions}
\label{sec:not}
A single-channel two-dimensional image is composed of pixels (squares of unit size) tiled on a regular grid on a cartesian space. Each pixel is associated to a value, which usually represents the average measurement of the phenomena of interest, projected among the pixel surface.

Formally, a discrete image is  a function $I: [0,n-1]\subset\nat \times [0,m-1]\subset\nat \longrightarrow \erre^+$.
The codomain values can be either quantized gray levels (e.g. 1, 8 or 16 bit representations) or a continuous value (e.g. represented by a floating point value). In the paper no assumptions about value type is made. The syntax $I(i,j)=v$ (or $I(p)=v$), where I is the image and $v\in \erre$, describes the value $v$ at position $p=(i,j)$ in the image. The \emph{pixel} $(i,j)$, with value $I(i,j)$, is the square defined by the corners $(i+i',j+j'), \mbox{with } i',j'\in\{0,1\}$. Practically, pixels are arranged on a Cartesian space and each pixel value is stored in its lower-left corner.
%Without loss of generality we assume that $I(i',j)=I(i,j')=0$, with $i'\in{0,n-1},j'\in{0,m-1}$. Informally, we assume to have a border of width 1 pixel around the image with minimal value. If the image does not have this property, it can be easily be extended consequently.

If not stated differently, indices named $i,j,k,..$ will range over $\nat$, while variables $x,y,z,w,..$  will range over $\erre$. In particular, $i\in[0,n-1]\subset\nat$, $j\in[0,m-1]\subset\nat$, $x\in[0,n]\subset\erre$, $j\in[0,m]\subset\erre$. When clear from the context, also image coordinates will inherit corresponding ranges.

Let us introduce a continuous function $R$ that is built from the image function $I$ through bilinear interpolation.
The image $R: [0,n]\subset\erre\times[0,m]\subset\erre\longrightarrow \erre^+$, where $R(i,j)=I(j,i)$ and the values in $R$ inside a pixel are defined by the bilinear interpolation: given $(x,y)=(i+z,j+w)$, where $z,w\in[0,1]\subset\erre$ are the offsets within the pixel, $a=I(i,j), b=I(i+1,j), c=(i,j+1), d=(i+1,j+1)$ and 

\noindent $R(x,y)=\sum_{a=0}^1\sum_{b=0}^1 v_{a,b}z^aw^b$, where $v_{0,0}=a, v_{1,0}=b-a, v_{0,1}=c-a$ and $ v_{1,1}=d+a-b-c$. 
Note that $R$ is continuous.

The image can visualized as a surface defined by the set of three-dimensional points $(x,y,R(x,y))$.
In practice, pixel values define the altitude of a landscape, composed of ridges, valleys, saddle points, local maxima and minima that play a central role in edge detection.

Given an integer coordinate $(i,j)$, let us define the set of 8-neighbors $n8(i,j)=\{(i',j')$.$(i,j)\neq(i',j'), |i'-i|\leq 1, |j'-j|\leq 1\}$. The set of 4-neighbors for $(i,j)$ is $n4(i,j)=\{(i',j')$.$|i'-i|+|j'-j|= 1\}$. In some cases discussed below, we characterize a generic neighborhood $n(i,j)$ as a subset of the 8-neighborhood, $n(i,j)\subseteq n8(i,j)$. 

Constant regions of the image are avoided, since they are not relevant for edge characterization and they introduce many technicalities in formal proofs and algorithms.
Rather than altering the image values, comparisons between pixels are redefined to break ties. It is sufficient to compare $n8()$ neighbors by using lexicographical order of coordinates. Given a point $p=(i,j)$ and two points $p_1=(i+i_1,j+j_1), p_2=(i+i_2,j+j_2)$,$p_1,p_2\in n(p)$, we define
 $I(p_1)\prec I(p_2)$, iff $I(p_1)<I(p_2)$ or $I(p_1)=I(p_2) \wedge 2i_1+3j_1<2i_2+3j_2$. Note that $i_1,j_1,i_2,j_2\in \{-1,0,1\}$.
W.l.o.g, the upper left value is the most favourite in case of equal pixel values. From now on, we assume that the image $I$ has all distinct values and/or refer to the operator $\prec$.

An \emph{isoline} is the intersection of the image surface with a constant value $v$ plane parallel to x-y plane. Formally, it is the locus of points $(x,y)$ that are solution to the equation $R(x,y)=v$. Within the function's domain, each isoline is a closed curve, possibly self-intersecting, whose points share the same value in $R(x,y)$. The line can be associated to an orientation: while moving along the line we assume to have greater values on the left.

There is a degenerate case when an extended region of the function is constant: according to the definition, the isoline contains an area rather than a line. In the context of bilinear interpolation, it easy to notice that the smallest case happens when four pairwise adjacent pixels have equal values. Let us assume that such cases are excluded and rely on the $\prec$ operator that implicitly introduces an infinitesimal slope to the constant plane. 
%These cases will be handled in the design of the actual algorithm and ideally are handled by introducing an infinitesimal slope that guarantees the absence of constant regions and does not alter the image structure.

The image \emph{gradient} $\nabla R$ is the vector made of the two partial derivatives along the $x$ and $y$ axis respectively, i.e. $[ \frac{\partial R}{\partial x},\frac{\partial R}{\partial y}]^T$. The gradient magnitude is defined as $\sqrt{\frac{\partial R}{\partial x}^2+\frac{\partial R}{\partial y}^2}$ and the gradient direction is the normalized gradient.

Focussing on a pixel with lower left corner in $(0,0)$ and 
recalling the bilinear interpolation definition given above,
we have that, given $x,y\in [0,1]$, $R(x,y)=v_{0,0}+v_{1,0}x+v_{0,1}y+v_{1,1}xy$. If follows that $\nabla R(x,y)=[ v_{1,0} + v_{1,1}y, v_{0,1} + v_{1,1}x]^T$. %If $v_{1,1}\neq 0$ $R$ is not a plane and it may present a saddle.

%The Hessian of $R$ is the 2x2 matrix containing second-order partial derivatives $H(R)=[\frac{\partial^2 R}{\partial x^2 }, \frac{\partial^2 R}{\partial x\partial y}; \frac{\partial^2 R}{\partial y\partial x} \frac{\partial^2 R}{\partial y^2}] = [0, v_{1,1}; v_{1,1}, 0]$. 

Function $\nabla R$ is piecewise continuous. Along pixel sides and corners, discontinuities of direction and magnitude may be encountered. 

%formula isoline, steepest line

%For each point belonging to an isoline, it follows that the gradient direction is orthogonal to the line (proof: along the isoline tangent the function is constant, therefore decomposing the gradient direction along the tangent and its normal, gives only normal....). 
%\hl{prooof..} inside a pixel -> gradient direction continuous, gradient magnitude function is continuous (no sopra punto split)
%(futuro edge splitting solo su bordi!)

\section{Image structure analysis}
\label{sec:structure}

We present an informal outlook of the proposed edge detection method.
The image model is based on a continuous surface (as in the plot of $R$). From a computational point of view, however, it is desirable to work at discrete level, i.e. with the original image $I$. 
The paper presents the theoretical results on $R$ and maps them to an acceptable algorithmical approximation on $I$.
%A surface contains peaks and valleys (local extrema) and saddles: 
%the relations among such points are crucial for the correct retrieval of edge structure.

%%%%%%%% da dire sotto
%Every graph edge represents a segment of an edge which is involved in at least one cycle in the graph.
%Edges are supported by isolines (lines that cover equally image valued points),  . A group of isolines  
%edges have an interval of values (support)

Figure~\ref{fig:edge-isolines} depicts a discrete image $I$ and shows some isolines associated to the corresponding $R$ function. Note that isolines partition the image into a set of areas. Traditionally, an edge is positioned where local maxima in gradient magnitude arise along the local gradient direction. This is a local property that should hold for every point along an edge line. If we map the property to the context of isolines, it means that an edge should be drawn at any place where the isolines get closer, since they show a gradient magnitude increase (compare to non-maxima suppression in Canny's detection).  The relationship between edges and isolines is actually deeper: an edge point is locally \emph{supported} by a set of isolines. 

Let us focus, in the figure, on the image section $A$, that connects a local (blue) minimum point to a local (red) maximum point by means of a white line that traverses the domain. Such line is a \emph{monotonic path}, meaning that it crosses an isoline at most once. In the paper, we will also consider \emph{steepest path}s that include the requirement that the path is orthogonal to isolines. 
Such paths contribute to classify isolines and to define the presence of edges.

The steepest path A intersects a set of isolines depicted with orange, green and cyan colors. This set of isolines concur in the identification of the purple point, namely the local maxima along the steepest path. We assume that an edge (depicted as a green dashed line) must be traversing such point. In case of several maxima are present along the path A, multiple edges will be traversing the path. The support of an edge will be partly defined by the set of isolines values that are included by the two local minima neighbor to the purple maxima along the path.
In principle, this idea could be applied to any steepest path in the image, as in Canny's non-maxima suppression. In addition,  while retrieving edges positions, the fundamental information about which isolines are related to a specific edge is computed. 

In particular, the same isolines that support the traversal of the edge through the path A, extend across the image and
form contiguous bundles (or a continuous range of image values), that sometimes separate and split in different sub-bundles. For example, cyan isolines loop around the image local maxima (the red dot), while the orange ones loop around the local minima (in blue). However, the green isolines cover a longer distance and eventually return at path A (due to $R$ image continuity).
This behaviour is caused by the presence of image \emph{saddles} (visualized by the green dot) or points where the gradient vanishes but second order derivatives along two orthogonal directions have opposite signs. Practically, at a saddle point isolines travel towards different image areas, according to their value. In the picture cyan and purple isolines cover higher values than the saddle point, while green isolines have smaller values. The set of green + cyan isolines passing through path A get divided into two sub bundles (resp. green and cyan) by the saddle point.
The path of an isoline provide information about edges at different locations. Green isolines, for example, contribute to the support of an edge at path A, but at the same time they reach path B. This fact is at the basis of the edge connectivity detection. An isoline that interacts with different paths represents a possible connection between the supported edges.
Note that at path B, the relative edge is also supported by blue and red isolines that are totally unrelated to isolines of path A.
We can infer that a bundle of isolines contribute to the definition of edge properties only if it is completely covered by a path that joins two local extrema in the image. In this way, all local maxima across the path are completely captured, but, more subtly, all isolines that support an edge are identified, so that they can be related to other paths.

The next important question is if and how the hypothetical edge passing through both paths A and B are related. In presence of a shared support (i.e. the green isolines), this set of isolines that connects the two paths suggests the presence of  an edge that connects the two purple dots. This property has a significant impact on edge detection, since there is no need to identify and connect \emph{all} gradient magnitude local maxima across the area between two paths, in order to discover the presence of a (possibly long) edge. 
The structural property derived from steepest paths and isolines guarantees the presence of an edge, even in the absence of a complete analysis of all pixels. 
A careful selection of the appropriate set of steepest paths will be discussed as well as a fast method to accurately draw the edges between paths. 

Connectivity plays a central role in the model and Figure~\ref{fig:outlook} presents a first simple example. 
Ideal edges are marked by green lines, isolines by black lines, local maxima by red circles, local minima by blue and the saddle points by green circles. The purple points represent the local maxima in gradient magnitude along each depicted steepest path.
The bold green line highlights the edge that connects the paths A and B in the previous figure.

\begin{figure}[htb]
\begin{center}
\includegraphics[width=0.40\textwidth]{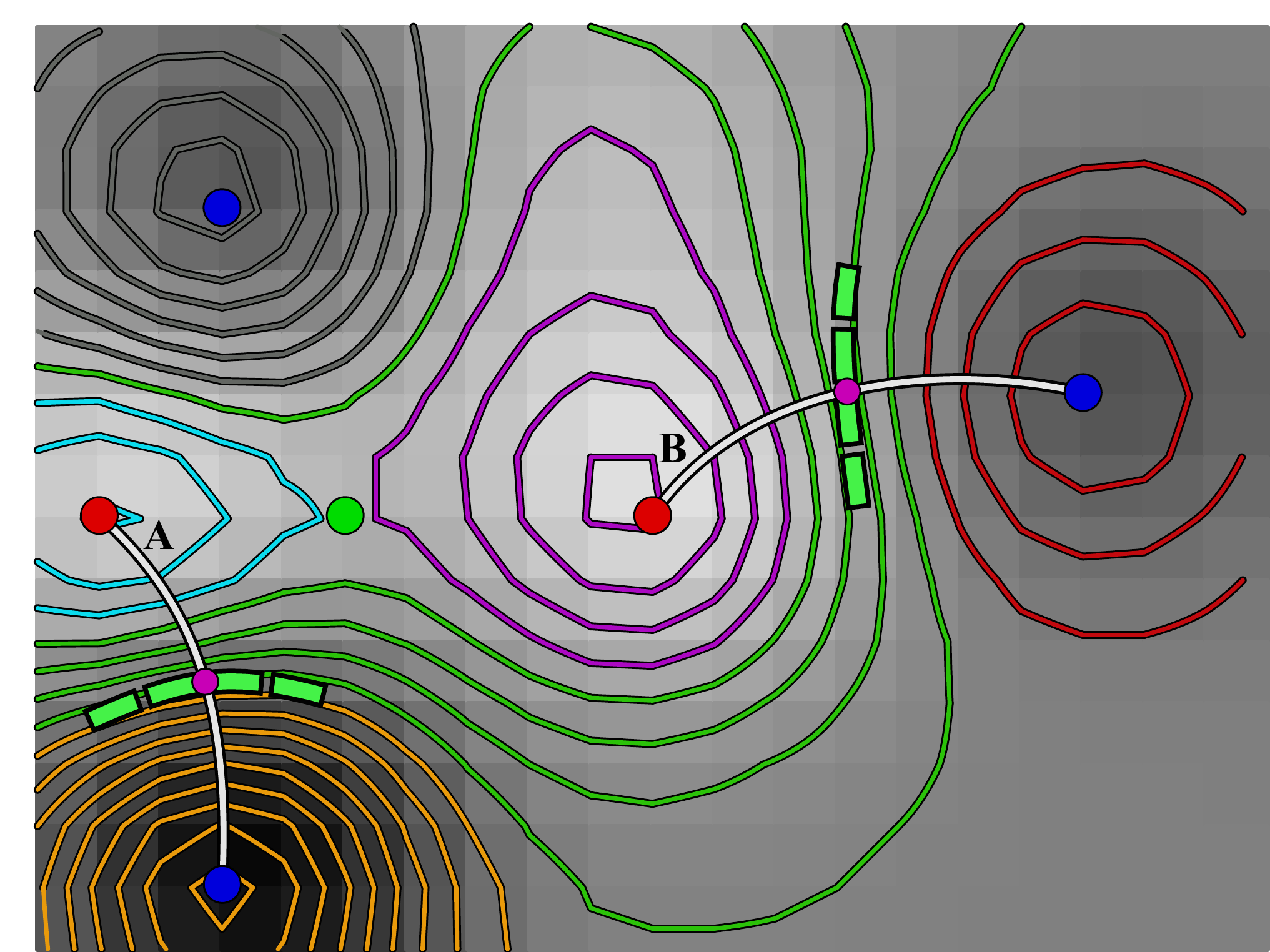}
\caption{Image isolines\label{fig:edge-isolines}}
\includegraphics[width=0.40\textwidth]{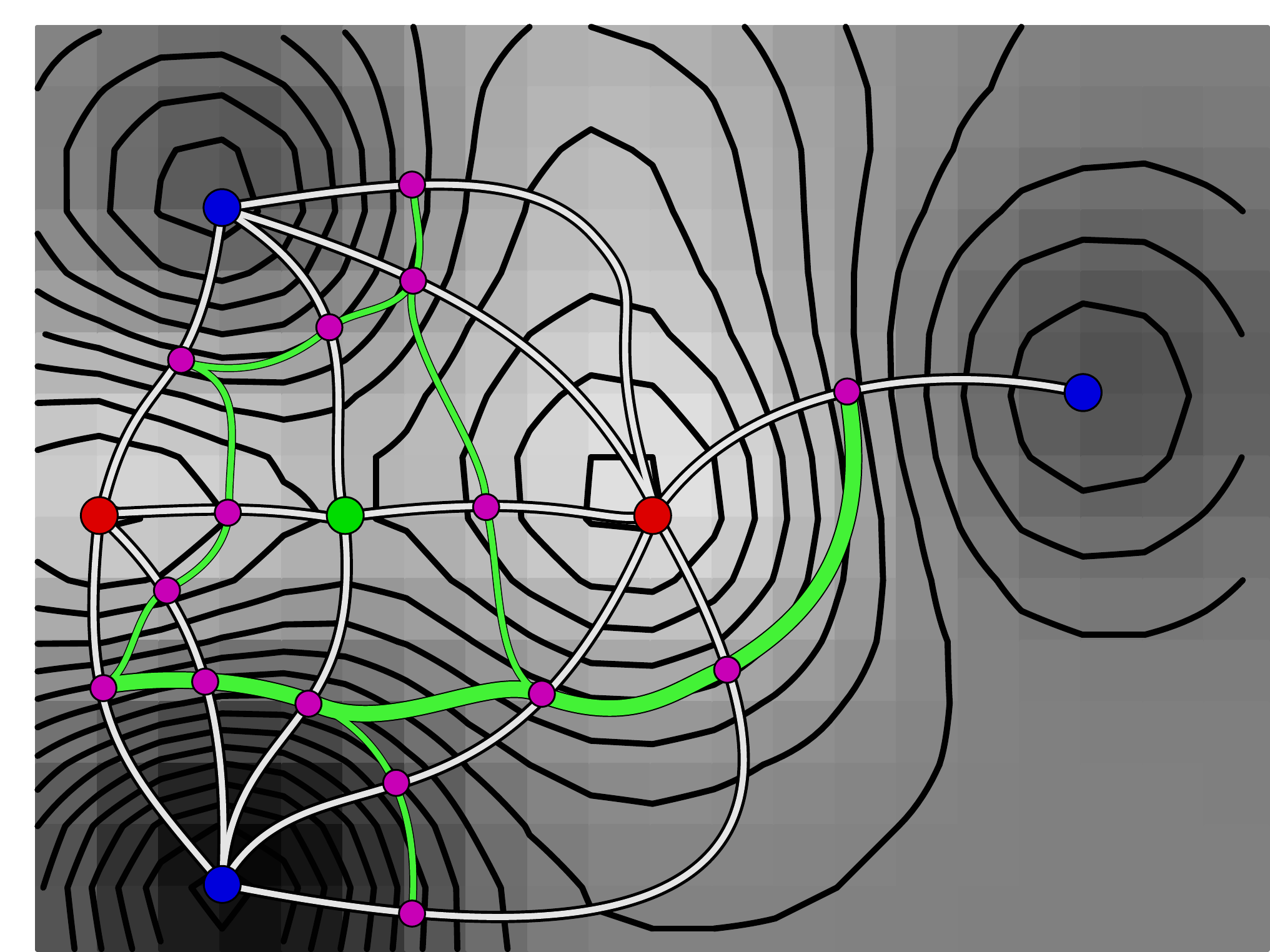}
\caption{Isolines, steepest paths and edges\label{fig:outlook}}
\end{center}
\end{figure}

Our connectivity model produces an edge graph where nodes are the purple points and graph edges are associated
to an image edge that links two nodes. Each maxima along a path is supported by a bundle of isolines which divides into other bundles crossing other paths, because of the presence of saddles and/or  maxima on different paths. The bundle evolution on paths allows to follow corresponding edge connectivity.

In this work we give more importance to the detection of edge branching, rather than its accurate localization. This is related to the fact that judging the actual edge branching point can be subjective and, if necessary, it can be post-processed with the preferred perceptual model, provided the correct connectivity graph is available with the information about such branching. Compare Figure~\ref{fig:y}, where correct connectivity is detected (left) and post processed into a more visually pleasing version, where the junction becomes translated to create a smoother line joining (right). 

The detection process: (1) identifies all steepest paths that section isoline bundles; (2) it defines the regions of interest delimited by steepest paths, which result in an image partition; (3) analyzes the regions for graph construction, where graph edges cross exactly one partition; (4) it draws edges without tracking. 
%We will show that the graph based on all steepest paths is planar and each path can be associated to at least an existing edge.
%Moreover the edges belonging to two adjacent regions smoothly join on their common path, since they pass through the same local maxima in gradient magnitude. 
Each region can be processed independently and in parallel.
%The procedure is feasible and robust, as opposed to direct detection methods or edge thinning and tracking (like in Canny's approach).

Isolines and steepest paths, however, need  a special treatment in the presence of critical points (saddles). In fact, 
isolines can intersect themselves at critical points as well as steepest paths can bifurcate and merge,
giving rise to a complex network of paths joining local extrema (depicted by white lines in Figure~\ref{fig:outlook}). 
%The graph, whose nodes are local extrema and saddles and edges are the steepest paths, is planar. This property guarantees that each region, delimited by two paths and two local extrema, intersect a set of non intersecting isolines and hosts at least one edge. 

\subsection{Caveats}

This plan sketched above hides a number of issues that need to be tackled.
The first consideration is that there is an infinite number of steepest paths, since they can be drawn starting from any point in the continuous 2d-domain. In this scenario, the proposed detection would resemble the Canny's procedure assuming an infinite resolution image: non-maxima suppression applies along every steepest path and infinitesimally close paths are connected by their maxima. Since this is not computationally feasible (and also unnecessary), the paths are clustered to a finite set of representatives that retain sufficient information for edge localization and correct connectivity. Redundant paths are characterized by the following property: two paths connect the same extrema and they contain the same (possibly empty) set of saddle point. This allows to limit the number of paths by a finite upper bound (the number of saddles). In Figure~\ref{fig:outlook}, for example, the two lowest leftmost paths are redundant as well as the two lowest rightmost.

It is preferable to design a core procedure that is based on $I$, since dealing with continuous coordinates is computationally expensive. At the same time, though, final edge information should be provided as vectors with continuous coordinates. For example, the bilinear interpolation possesses a closed formula for steepest paths. Nevertheless,  this would introduce additional floating point arithmetics that can be avoided and approximated by the introduction of \emph{discrete increasing steepest paths}. 
It is worth noting that even on $R$, as opposed to any $C^1$ function, due to the piecewise pixel interpolation, an \emph{increasing} steepest path can be different from a \emph{decreasing} steepest path. Since the issue must be tackled also for image $R$, we propose a solution that also works with $I$ image. 

Another drawback of bilinear interpolation is the presence of saddle-like points at corners (see Section~\ref{sec:saddles}), where $\nabla R$ is often discontinuous. Discrete paths are defined to correctly handle such cases.

%The approximated and discrete version of the connectivity graph is planar and it connects all extrema and saddles/saddles-like points of $R$.

%A region partitioned by the graph has two paths that delimits it. We impose that such paths intersect only at the extrema, in order to avoid that %a region gets fragmented and separated . 
%In case of a n4 neighborhood for building the paths, this may not happen. To avoid this we consider 

%Edge bifurcation (4) and related edge connectivity is taken care by analyzing the partitions and their surrounding paths. In particular, the actual bifurcation point is not directly calculated, in order to avoid expensive edge tracking. Instead, a correct model for connectivity is presented and the edge is assumed to bifurcate across the region, while the exact splitting location is not exactly determined.

Since each region defined by the steepest paths does not contain any saddles, the edge drawing is based on isoline interpolation within the region (see Section~\ref{sec:edgeDrawing}), in order to avoid tracking of the maximal gradient magnitude.

\subsection{Saddles}
\label{sec:saddles}
%The section characterizes saddle points and defines a discrete graph that captures their connectivity.
%We base our discussion on the $R$ version of the image, with the ultimate goal of dropping almost all continuous models and directly work on $I$.

Saddles determine isolines and steepest paths behaviours.
We first introduce the \emph{Continuous Monotonic Increasing Path} in order to provide a general definition of saddle point, due to the fact that $R$ is not a $C^1$ function.

\begin{definition}[Continuous Monotonic Increasing Path]
A continuous monotonic increasing path is a continuous line $L(t)=(x_t,y_t)$ such that if $t<t'$ then $R(x_t,y_t)<R(x_{t'},y_{t'})$.
\end{definition}

Starting from a local minimum $L(0)$, such path identifies a line on the surface, such that, while proceeding, it increases its altitude. Since  we assume a function without constant pixels and sides, there is always a next candidate except for local maxima.
%A key property is that not every peak in $R$ is reachable from a specific local minimum. 
%Saddle points impose some constraints on the monotonic paths and they play a central role in the reachability of extrema.

\begin{definition}[Saddle]\label{def:saddle}
Consider two distinct curves $\ell_1$ and $\ell_2$ such that they intersect at the point $p$. $\ell_1$ and $\ell_2$ are pairwise delimited by points $p_1,p_2$ and $p_3,p_4$ respectively. The $p_i$ points and $p$ are all different. If the sub-curves delimited by $(p,p_1)$,$(p,p_2)$,$(p_3,p)$, $(p_4,p)$ are all monotonic increasing paths and $p_1, p_4, p_2, p_3$ appear in clockwise order around $p$, the point $p$ is a saddle.
\end{definition}

Definition~\ref{def:saddle} is slightly more general than the usual one, since it accounts for possible discontinuities of  $\nabla R$ at pixels borders. In such cases, the classical definition ($\nabla R(p)=\vec 0$, but at the same time $p$ is not a local maximum nor a minimum) would not hold anymore.

\begin{figure}[ht]
\begin{center}
\includegraphics[width=0.35\textwidth]{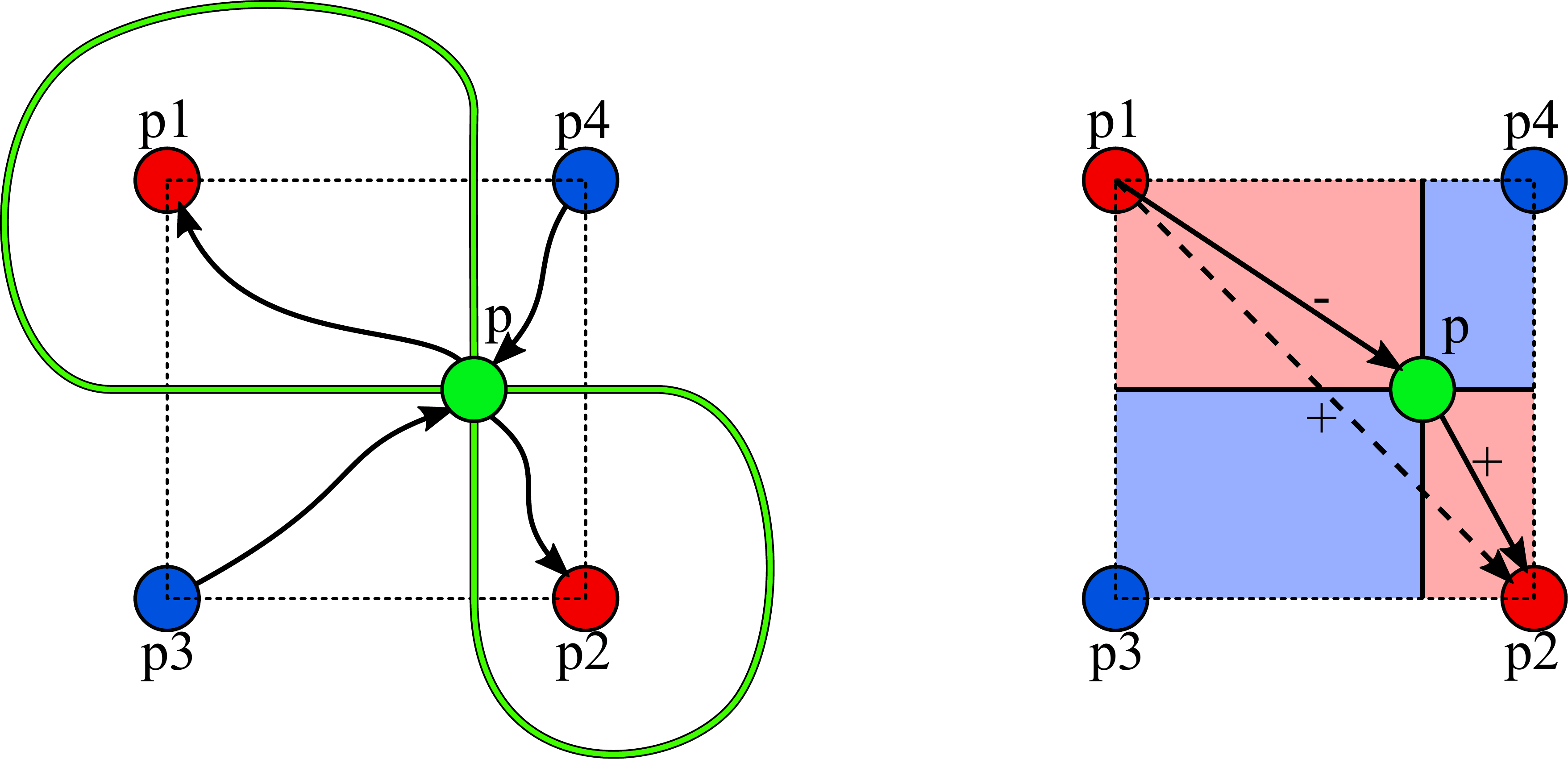}
\caption{Left: a possible isoline arrangement (green line) at saddle point (green dot). Right: discrete vs continuous neighborhood\label{fig:split-iso}}
\end{center}
\end{figure}

In Figure~\ref{fig:split-iso} we depict a minimal example. In blue we show pixel corners with values smaller than $R(p)$ and in red the corners with greater value. The saddle point $p$ is depicted with a  green dot and the corresponding isoline is traced with a green line. The saddle identifies two different areas delimited by the green isoline (in a different scenario, it can also enclose only the blue dots, depending on the neighbor pixels values). For greater isoline values, the two areas become disjoint, while lesser values they produce a larger and unique area. 

%The following lemma relates this saddle definition to special behaviour of isolines in $R$.

\begin{lemma}[Saddle point and isolines]
A saddle point is positioned at $p=(x,y)$ iff there are two distinct and equally valued isolines that intersect at $p$.
\end{lemma}

%Euler formula...

Let us exhaustively treat three different types of saddles, depending on their position: inside a pixel (square perimeter excluded), on a side of the pixel (corners excluded) and on a pixel corner.
We assume a generic pixel at position $(i,j)$ and a saddle point $p$ with coordinates $(x,y)$ with $i\leq x\leq i+1,j\leq y\leq j+1$, $a=I(i,j), b=I(i+1,j), c=I(i,j+1), d=I(i+1,j+1)$, $ v_{1,1}=d+a-b-c$ as defined above.

{\bf Split point:}
If $i<x<i+1,j< y< j+1$, the saddle lies inside a pixel. Recall that $R$ and $\nabla R$ are continuous inside the pixel.
% and for any line $\ell_1$, the lines $p_1,p$ and $p,p_2$ are both monotonic increasing (or decreasing). There is only one exception:
If $\nabla R(x,y)\neq \vec 0$, any line $\ell_1$ with $p=(x,y)$, $p_1=(x+\epsilon,y+\delta)$ and $p_2=(x-\epsilon,y-\delta)$, for sufficiently small $\epsilon$ and $\delta$, verifies that lines $p_1,p$ and $p,p_2$ are both monotonic increasing (or decreasing) and therefore the saddle property is falsified.
When $\nabla R(x,y)=\vec 0$,
% $R$ is either trivially a constant plane ($a=b=c=d \Rightarrow v_{1,1}= 0$) or 
and $v_{1,1}\neq 0$ (we exclude the planar case), given $p_1=(x-\epsilon,y-\epsilon)$, $p_2=(x+\epsilon,y+\epsilon)$, $p_3=(x+\epsilon,y-\epsilon)$, and $p_4=(x-\epsilon,y+\epsilon)$, $\epsilon\in R^+_0 $ such that $p_i$s are inside the pixel, the saddle property holds. 

In this case, we define a \emph{split pixel} as the pixel (associated to its low left corner $(i,j)$) that contains a saddle at $(x,y)$. We name $p$ as a \emph{split point}.
%, where $x=i+x', y=j+y', x',y'\in[0,1]$. 
Therefore, for a split pixel it holds that
%$v_{1,1}\neq 0$, 
$0<x=(a-c)/v_{1,1}< 1$ and $0< y=(a-b)/v_{1,1}< 1$, as solution to $R(x,y)=\vec 0$ with $(x,y)$ inside the pixel. 
Let us name $R(x,y)$ as the \emph{split value}.
It also follows that $R(i,y)=R(i+1,y)=R(x,j)=R(x,j+1)=R(x,y)$ because of bilinear interpolation. 
Note that there can be at most one split point per bilinear function, since $\nabla R$ is a vector of linear functions.

\begin{lemma}[Split pixel corners value]
Given a split point $p=(x,y)$, associated to the split pixel $(i,j)$, it can either be $I(i,j)>R(p)$,$I(i+1,j+1)>R(p)$,$I(i+1,j)<R(p)$ and $I(i,j+1)<R(p)$
or the same relationships with inverted order. Informally, two opposite corners are greater than $R(p)$.
%\hl{qui dimostro split pixel -> alternating corners. Dovrei dimostrare alternating corners -> split pixel?}
\end{lemma}

{\bf Pixel sides:}
We show that no saddles can occur along a pixel side, w.l.o.g. at $(i,y)$. 
Let us assume that there is a degenerate split pixel with a split point at $(i,y)$. This would imply $R(i,j) = R(i,j+1)$, which is excluded by assumption. Therefore, the only other case is when both pixels $(i-1,j)$ and $(i,j)$ contain no split pixels. It is easy to see that due to continuity of $R$, the isoline passing through $(i,y)$ is continuous, possibly non differentiable at $(i,y)$, and it
divides each pixel in two parts. It is not possible to cross the isoline with two curves $\ell_1$ and $\ell_2$ such that they verify the saddle property. In particular, it is impossible to verify the clockwise ordering of points for any pair lines $\ell_1$ and $\ell_2$ intersecting $p$.

{\bf Mix points:}
Let us assume the saddle is at the corner $p=(i,j)$ and let us name it \emph{mix point}.
With a similar argument as in the pixel side case, we can exclude the cases where any of the four neighbor pixels of $p$ are a degenerate split pixel. 
The corner value is shared by four different bilinear functions that join at the same point. Moreover, the piecewise gradient $\nabla R$ can be often discontinuous in $p$.
W.l.o.g., let us assume the point $p=(1,1)$ is the corner being investigated.

\begin{lemma}[There is at most one isoline that passes inside a pixel and on one of its corners]
Given a pixel with corner $p$, the isoline passing through $p$ solves the equation $R(x,y)=R(p)$. Given the bilinear interpolation, the line is unique and there is at most one line that intersects the pixel area. 
\end{lemma}

\begin{lemma}[Mix point characterization]
The following statements are equivalent:

(i) A mix point $p$ exists at integer coordinate

(ii) each of the four pixels adjacent to $p$ contains an isoline with value $I(p)$ that reaches the pixel border at $p'\neq p$

(iii) n4(p) neighborhood is such that $W=I(0,1)-R(p)$, $E=I(2,1)-R(p)$, $S=I(1,0)-R(p)$, $N=I(1,2)-R(p)$, $NS>0$, $EW>0$, $NE<0$
\end{lemma}

We exhaustively showed that saddles can occur inside a pixel (split points) and/or at pixel corners (mix points), given the assumption that $I$ contains no equal valued $n8()$ neighbors.
%There are no other points that can generate a scenario depicted in Figure~\ref{fig:split-iso}. 

\iffalse
From a reachability point of view, the saddle point represents a critical choice point that allows to access both green areas with a monotonic path and eventually reach the local maxima inside the area. Of course, there are other (infinitely many) paths that, starting from the minima,  can reach the two maxima without passing through the saddle. The importance of the saddle point is that it is the only point that, if crossed, it allows to reach both maxima. In fact,
any monotonic path can cross at most once the green isoline, except for the saddle point, where the isolines are collapsed.
Note that once entered one of the two areas, it is not possible to reach the other one through a monotonic path.

\hl{questo forse dopo quando pulisco. Moreover, the saddle point can be used as representative of all paths that connect a pair of extrema (passing through the green isoline).}
\fi

%\subsection{Saddles in discrete context}

The process of building discrete steepest paths must account for the discrete sampling of the image $I$ at corner values only. The implicit assumption in the discrete scenario is that a discrete segment $(p_1,p_2)$ with $p_1\prec p_2$ and $p_1\in n8(p_2)$ 
should correspond to the presence of a monotonic increasing continuous path from $p_1$ to $p_2$ in $R$.
This is trivially true along sides ($n4()$ neighborhood) because of bilinear interpolation. However, for diagonal segments crossing split pixels, it does not hold (see Figure~\ref{fig:split-iso} on the right --- dashed arrow).
In fact, for such diagonals the discrete interpretation violates the assumption. A finer analysis reveals that in $R$ 
there is no monotonic path between $p_1$ and $p_2$.
We already showed that both corners on a diagonal are greater (or less) than the split value. Considering an increasing path starting from $p_1$, 
the split pixel contains a (lower than $R(p_1)$)  split value that has to be crossed (at least once) by a corresponding continuous path, in order to reach the opposite corner. The split value isolines are depicted with solid black lines. Thus in $R$, any continuous path connecting $p_1$ to $p_2$ contains at least an inversion of monotonicity (depicted by '-' and '+' subpaths). 

In conclusion, the diagonal comparisons over a split pixel are inhibited, in order to preserve the monotonicity property of paths  between $I$ and $R$ images. Thanks to this choice, the image $I$ can be equivalently explored at discrete level.
We therefore introduce the general neighborhood operator $n()$, defined as a subset of $n8()$ where diagonals crossing split pixels are removed from the set.

%\hl{dimostrare?}

\subsection{Minima and maxima}

In $R$, local maxima (minima) can only appear at integral coordinates, since for each pixel and point $(x,y)$ there is always a corner with greater (lesser) value, due to bilinear function properties.
It follows that detecting local maxima (minima) boils down to checking the maximality (minimality) among the n() neighborhood. Recall that diagonals that cross a split pixel are excluded from the test, since they would incorrectly account for the computation a greater (lesser) value than the correct extremum, while overpassing the saddle and including a different peak (valley).
The $\prec$ operator allows to identify local minimum and maximum.
A point $(i,j)$ is a local maximum (minimum) in $I(i,j)$, with respect to $n(i,j)$, iff $I(i',j') \prec I(i,j)$ ($I(i,j) \prec I(i',j')$) for every $(i',j')\in n(i,j)$. 

\subsection{Steepest Paths}
\label{sec:sp}

%Steepest paths allow to analyze isoline bundles and build the scaffolding for edge detection. 

In $R$ \emph{continuous steepest paths} are continuous monotonic paths that proceed along the gradient direction. 
%and they do not intersect each other (unless min/max/saddles are encountered). 
Due to the piecewise continuous gradient function, as depicted in Figure~\ref{fig:steepest-cont-discr} with black arrows,
a continuous steepest path can traverse points with undefined directions (i.e., saddles and pixels sides).

We prefer to work at discrete level, since the processing is more efficient and the cases to be treated are reduced.

%Informally, a discrete increasing steepest path is a list of discrete points that iteratively identify the highest image value on a local neighborhood n() w.r.t. the previous point.
%This approximation of the continuous version, does not guarantee that the path is orthogonal to isolines and gradient direction. However we show that the most important property, namely maintaining the correct connectivity of extrema, is guaranteed. 

Let us introduce the formal definition of discrete steepest path.

\begin{definition}[Discrete (Increasing) Monotonic Path]
A discrete monotonic path is a sequence $P=\{p_0,\dots, p_{n-1}\}$ of integer 2d points such that $p_i\prec p_{i+1}$.
\end{definition}

\begin{definition}[Discrete (Increasing) Steepest Path]
A discrete steepest path is a discrete monotonic increasing path P such that $p_i\prec p_{i+1}$ and there is no $p'\in n(p_i)$ such that $p_{i+1}\prec p'$. 
\end{definition}

As discussed below (see Lemma~\ref{lem:connect}), the discrete version may differ from the continuous one
but both share the monotonicity property. Compared to continuous steepest path, the discrete version produces the best orthogonal path to isolines.
Our choice of using a $n8()$ neighborhood, as opposed to $n4()$, is dictated by the need of
better approximating the continuous steepest paths. Moreover the regions identified by $n8()$ paths have a non self-intersecting perimeter.

%producing paths that intersect only at extrema. In fact, paths built on n4() and embedded in a graph as in the step 4 of next Section,
%could result in self intersecting paths, where non extremal node. Moreover, the n() neighborhood provides more directions and it follows more faithfully the correct continuous steepest path.

\subsection{Steepest graph}
The steepest graph collects the discrete steepest paths and it represents a finite and compact description of the image surface landscape. %Moreover, any path in the graph monotonically intercepts the set on intersected isolines.

\begin{definition}[Steepest Graph]
The steepest graph G=(V,E) is defined as follows: the set of nodes V is the set of points $(i,j)$ and directed edges $e=(n_1,n_2)\in E$ are such that $n_2\in n(n_1)$.
The graph $G$ is built incrementally according to the following steps:
\begin{enumerate}
\item Given a split pixel $(i,j)$ the four edges $((i,j),(i+1,j))$, $((i,j),(i,j+1))$, $((i+1,j+1),(i+1,j))$, $((i+1,j+1),(i,j+1))$ are added if $(i,j)\prec(i+1,j)$. Otherwise the edges are reversed.
\item Given a mix point $(i,j)$, the four edges $((i,j),(i+1,j))$, $((i,j),(i-1,j))$, $((i,j-1),(i,j))$, $((i,j+1),(i,j))$ are added if $(i,j)\prec(i+1,j)$. Otherwise the edges are reversed.
\item Given any steepest discrete path $P$,  edges $(p_i,p_{i+1})$ are added if $(p_i,p_{i+1})\in P$.
\item Given a node $p\in V$ such that there are no incoming edges and $p$ is not a local minimum, the edge $e=(p_L,p)$ is added, where $p_L\in n'(p), p_L\prec p$ and there is no $p_L'\in n'(p)$ such that $p_L'\prec p_L$, and $n'(p)\subseteq n(p)\setminus D$, with $D$ containing the diagonals that intersect any diagonal edge contained in E.
\end{enumerate}

\end{definition}

Steps 1 and 2 introduce predetermined edge patterns that characterize the discrete behaviour of steepest paths in the presence of a saddle. Even if a point contained in such edges has another best optimal neighbor in $n()$, such patterns are enforced. 
This allows to guarantee the bifurcation of a steepest path caused by a saddle (compare Figure~\ref{fig:split-iso} on the left, the two incoming edges at $p$).
Edges added at steps 1 and 2 are marked as non removable, since any graph simplification must maintain these fundamental structures.

%Since they pass by a saddle point they the key to determine min-max paths and carry information \hl{tanto da dire... piu' giu???}

After step 3 has introduced all steepest paths of the image, it is possible that some paths may start from non local minima $p$ (see Figure~\ref{fig:connection} left). The reason is that any lower valued node w.r.t. to $p$ (e.g. $p_L\prec p$, $p_L\in n(p)$) has a steepest neighbor node ($p_H$) which is different from $p$. 

\begin{figure}[ht]
\begin{center}
\includegraphics[width=0.4\textwidth]{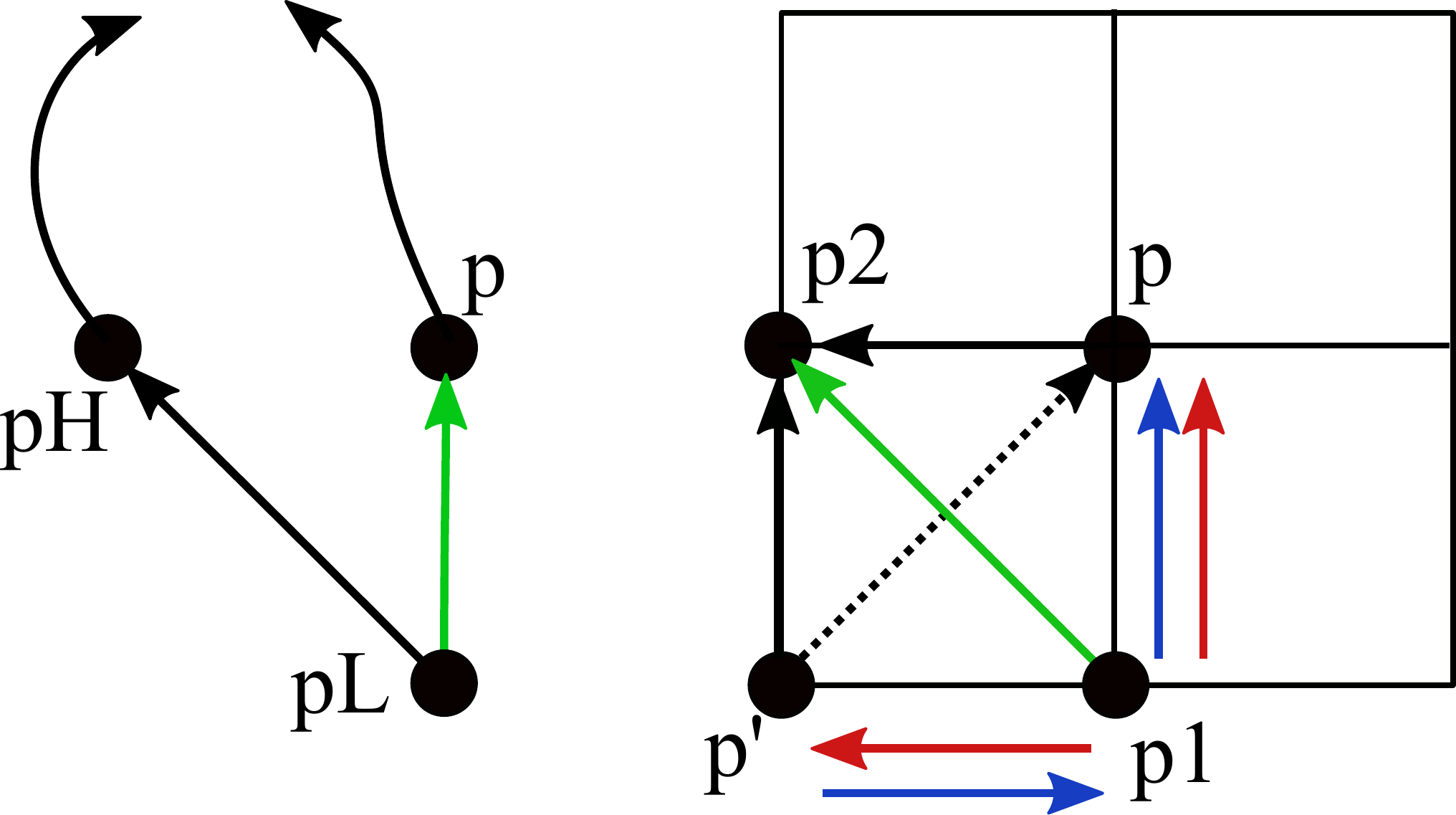}
\caption{Edge model for a region\label{fig:connection}}
\end{center}
\end{figure}

Step 4 guarantees that only local minima have no incoming edges. 
The procedure introduces additional edges with the requirement that the new addition is the steepest w.r.t. $p$ (the green edge in the Figure). Moreover, special care is devoted to avoid to select a new edge that crosses another edge already in the graph. This is due to the need of maintaining a planar graph. Note that this step introduces a potential outdegree greater than one for nodes.
%Note that this issue derives from the fact that the function $R$ is not $C^2$. In that case, each steepest path would converge to the local minimum without any intersection. The discretization of the path can not account for such precision.
We show now that it is always possible to find an edge that connects a non minimal node without incoming edges to a lower valued node.

\begin{lemma}[Non minima joining]
Given a non minimal node $p$ with no incoming edges and the graph built by steps 1,2 and 3, there is a new edge that does not intersect a diagonal edge and that connects the node to a lower valued node.
\end{lemma}

Let us cover some fundamental properties of the steepest graph.

\begin{lemma}[DAG]
The graph is a directed acyclic graph. 
\end{lemma}

\begin{lemma}[Planar graph]
The graph is planar
\end{lemma}

\begin{lemma}[All saddles are connected]
The graph contains, for each saddle, at least a path that connects it to a local extrema.
\end{lemma}

\begin{lemma}[Correct connectivity]\label{lem:connect}
  A path starting from a local minimum and ending to a local maximum in the graph is a monotonic increasing path embedded in R.
\end{lemma}

This lemma shows that every path contained in the graph correctly explores the surface of $R$, i.e. the paths are monotonic visits of $R$ and therefore they cross each isoline at most once.
However, it can be easily seen in Figure~\ref{fig:steepest-cont-discr} a counter example where the discrete steepest path (in red) is not a \emph{steepest} path in $R$ (in black). Some local value relationships are shown in green and we assume there are no saddles. 
Depending on pixel values, it may exists a path A, if bottom right pixel contains higher magnitude gradients within the pixel rather than along the side between bottom pixels, while path B could be selected in the opposite case.
Steepest paths in $R$ are often undefined along sides.
A discrete steepest path may end into a different local maxima compared to a path built on continuous lines (i.e. the path A and next pixels are never accessed by the red path).

%\hl{siamo sicuri? non e' che il percorso non si puo' comunque separare. vale la pena fare la dim?}

In conclusion, the discrete steepest graph is a correct monotonic visit of the $R$ surface. Since the visit is planar,
the isolines are crossed in order and a characterization of sectioned edges can be performed. 
%The n8 neighborhood offers an approximation of continuous steepest curves in $R$. The n4 neighborhood would cause a stronger approximation as well as the problem that different path could intersect at non saddle points. \hl{da dire qualcosa in piu'}
An example of steepest graph is depicted in Figure~\ref{fig:test1} where edges are drawn as black to white gradient lines, the local minima (maxima) as blue (red) circles, saddles as green circles.
 
\begin{figure}[ht]
\begin{center}
\includegraphics[width=0.25\textwidth]{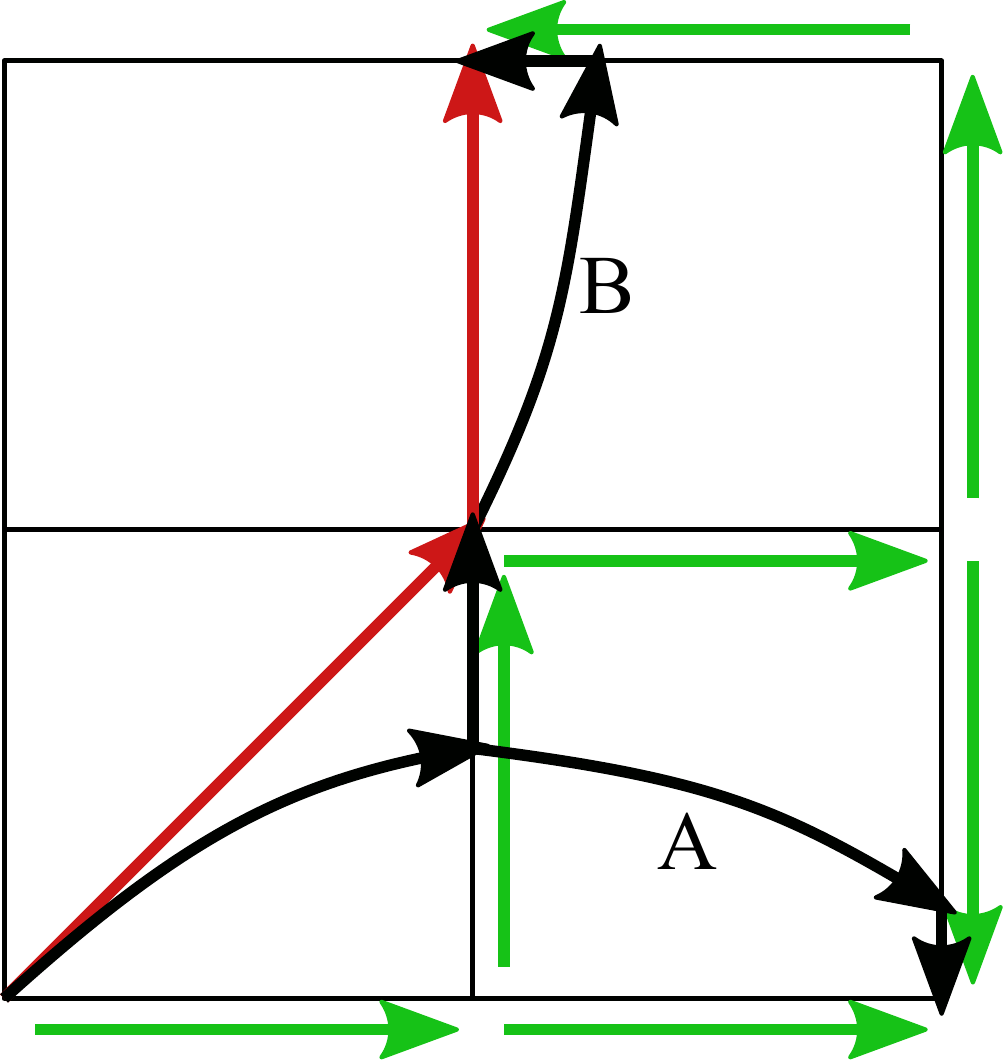}
\caption{Steepest paths: discrete vs continuous in $R$\label{fig:steepest-cont-discr}}
\end{center}
\end{figure}

Even if in the different context of detecting the inclusions of areas and the description of the image topology through a tree of areas, \cite{caselles2009geometric} supports the idea that saddles are an important element for guiding the image structure analysis.

\iffalse
\hl{qui c'era il grafo... rimettere? dopo?}
\begin{figure}[ht]
\begin{center}
\includegraphics[width=0.6\textwidth]{grafo-connect}
\caption{split iso\label{fig:grafo-connect}}
\end{center}
\end{figure}

A monotonic path can encounter several saddles and each traversal causes a new binary choice, which can be represented as a connectivity graph (Figure~\ref{fig:grafo-connect}\hl{da rifare giusta, da esempio vero!!!}). The graph nodes are the local minima/maxima and saddles, while directed edges connect a pair of nodes only if there is a monotonic path that connects them. 

Monotonic paths and the graph show an important property: they capture the reachability of local extrema in the image. Moreover
they are useful to determine relevant properties for edges that will be discussed \hl{: if two extrema are not reached by a monotonic path, there can not be an edge support touching both of them. Recall that multiple supports can be concatenated, as well as contained edges.
It can happen that a concatenation of supports can touch both extrema. dopo!}
\fi

\subsection{Image regions}

\begin{definition}[Min-max path]
Given a steepest graph $(V,E)$, a min-max path is a maximal list of nodes $P=(p_1,\dots,p_k)$ such that $P\subseteq V$ and there are no $(p',p_1)\in E$ nor $(p_k,p'')\in E$. $p_1$ and $p_k$ are image local minima and maxima respectively.
\end{definition}

The steepest graph contains a finite set of \emph{min-max path}s and every graph edge belongs to at least a min-max path.
%Note that two min-max paths can partially overlap, when they share some edges.
%Note that a min-max path is a self-avoiding walk. dimostro
Such paths partition the image domain into a set of \emph{regions}. Each region, given the planarity of the graph, is 
delimited by a polyline made of a set of graph nodes and edges that define the region's perimeter. Each region contains no saddles by construction. In principle the edges that delimit the split pixel form a single pixel region: we are not interested in such area since it is the discrete representation of the split point. Therefore we omit such regions and we assume that the split point is copied and projected to all four edges along the pixel sides, at the corresponding split value position.

\begin{definition}[Monotonic region]
A region is \emph{monotonic} if, excluding its perimeter, it contains no saddles.
\end{definition}

We show that a monotonic region's perimeter has exactly one local minimum node and one local maximum node.

\begin{lemma}
A monotonic region contains no self intersecting isolines.
 \end{lemma}

\begin{lemma}
The non collapsed part of a monotonic region has exactly one minimum and one maximum in image values on its perimeter.
 \end{lemma}

The region can collapse towards the extrema, since the two min-max paths delimiting the region may become overlapped (see Figure~\ref{fig:region}). It can also happen that along the collapsed min-max paths, several saddles (black points) are traversed, with the result that multiple extrema can be reached. This is not relevant, since any of the completely collapsed paths between a saddle and an extrema contain only collapsed edges, since actual edges will traverse other adjacent regions.
Let us name the left, L in the figure, (right, R in the figure) path as the one that delimits the region at its right (left), while moving towards the maximum. 

Another degenerate case is when the two sides are completely overlapped. This means that there is  only one region that covers the whole domain. 

\begin{figure}[ht]
\begin{center}
\includegraphics[width=0.07\textwidth]{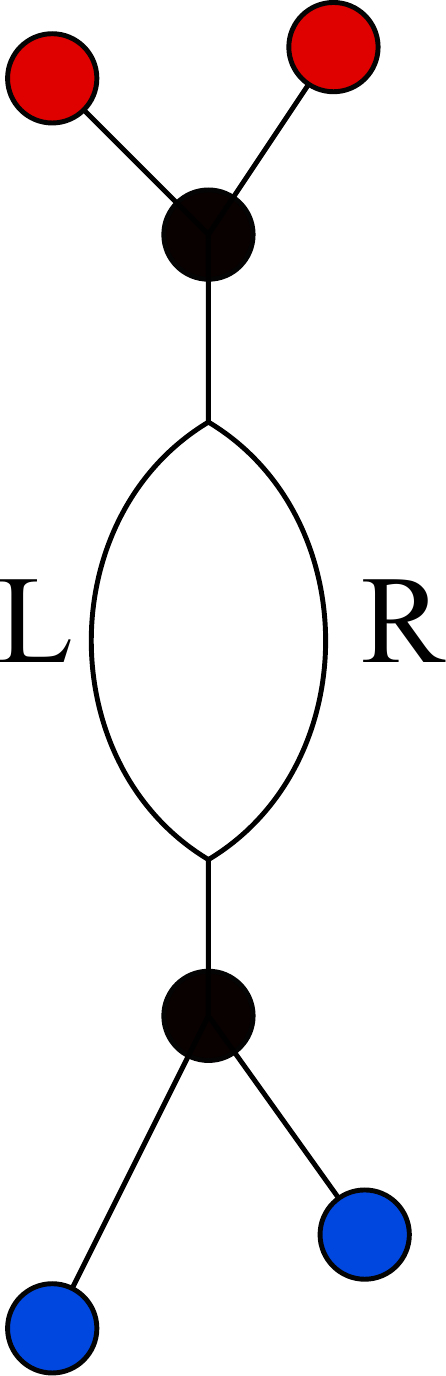}
\caption{Region organization\label{fig:region}}
\end{center}
\end{figure}

%%%%%%%%%%%%%%%%%%%%%%%%%%%%%%%%%%%%%%%%%%%%%%%%%%%%
%%%%%%%%%%%%%%%%%%%%%%%%%%%%%%%%%%%%%%%%%%%%%%%%%%%%
%%%%%%%%%%%%%%%%%%%%%%%%%%%%%%%%%%%%%%%%%%%%%%%%%%%%

\section{Edge model}
\label{sec:em}

We pursue two main goals in the detection of edges: the retrieval of accurate edge line position and the exact connections between edges.
The edge model is based on monotonic regions as constructed in the previous section. Operationally, each region can be independently processed and the union of the results can be merged into the overall set of edges and connectivity graph.
%The edge connectivity is the leading property for edge lines characterization. 

It is fundamental to model edges with continuous coordinates. A model based on discrete coordinates (compare with raster detectors, where edge polylines are derived from pixel chaining of local pixel neighborhood) is not able to reach the desired precision and to disambiguate the correct connectivity.

%We start with min-max paths that are indeed discrete. However, the edge construction relies on (continuous) isolines in $R$.

%Edge tracking, namely the attempt to determine a local maxima of gradient magnitude and to extend a line on top of the gradient ridge,
%is very errors prone even at continuous level (zero second derivative tracking approaches). (compare canny section)
%At the same time, modeling an edge with continuous line, would in principle allow a , but it becomes a numerical problem of accumulating errors, expensive. not feasible. \hl{un po' di letteratura qui...}
%Working at discrete level is not suitable for directly computing edges, since it is essential to preserve the continuous nature of the edge (edge line discretized too coarse). Another fundamental property is connectivity between edges is very difficult to observe and track in discrete images (especially if pre-filtering is applied).

%Regions are traversed by edges, and each edge traverses the min-max path over a local maximum in gradient magnitude.
%Edge computation can be partitioned over each single region and later joined.

Let us introduce three assumptions that approximate the edge model, mainly for computational reasons.
(1) {\bf Discrete regions}: the min-max paths (and consequently regions) are discrete over $I$; 
(2) {\bf Connectivity}: edge can split and merge only at regions boundaries; (3)
{\bf Edge drawing}: edges are interpolated inside a region, rather than exactly traced (i.e., exact maxima /zero second derivative tracing).

The three assumptions allow to recover exact edge connectivity and accurate edge positions. At the same time, they allow a practical algorithm design. Discrete regions and their properties have been discussed in previous section. The second assumption is discussed in Section~\ref{sec:ed} and the last one in Section~\ref{sec:edgeDrawing}.

A monotonic region has the peculiar property of containing a set of isolines that do not intersect with each other (due to absence of saddles). As informally presented in the overview, the region is a maximal partition that satisfies this property. There are a number of possible monotonic partitions of the image, but the one presented above is computationally efficient and the min-max paths mimic at best the continuous steepest paths while preserving correct monotonic reachability of local extrema. 
%In Figure~\ref{fig:outlook}, the steepest graph allows to correctly cluster isolines bundles and define suitable sections, in order to track related edges.

Figure~\ref{fig:edge} depicts the main elements in a monotonic region, delimited by left and right paths. Structural features are highlighted on the left image, while the corresponding edge properties are shown by the central image. Shared local minimum and maximum for the two steepest paths are identified by Min and Max points. Other elements are detailed below.
%The left path traverses points named Min,LM2, Lm1, Lm1 and Max, while right path Min, RM3, Rm2, RM2, Rm1, RM1, Max.

\begin{figure*}[ht]
\begin{center}
\begin{minipage}{0.33\textwidth}
\includegraphics[width=0.99\textwidth]{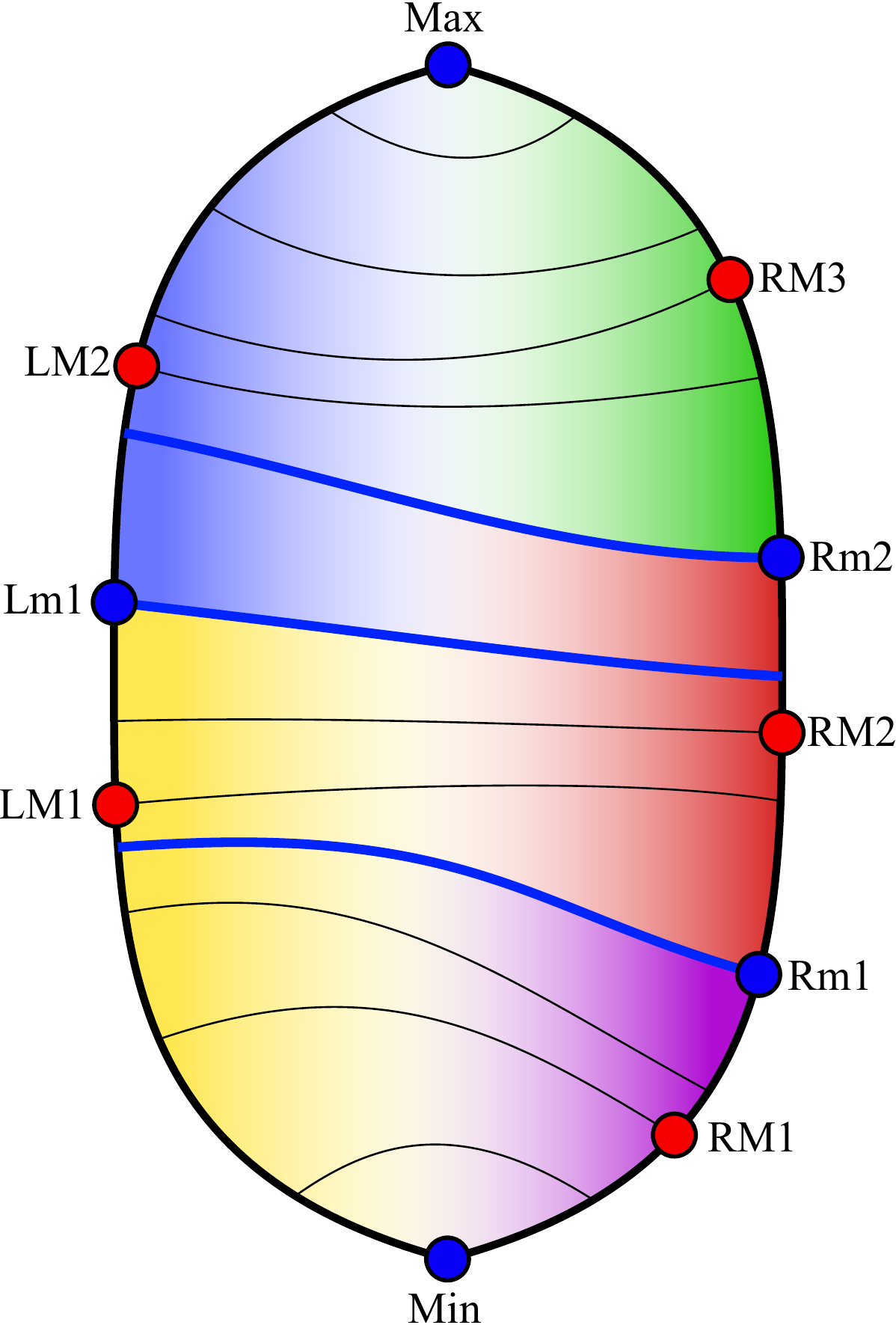}
\end{minipage}
\begin{minipage}{0.33\textwidth}
\includegraphics[width=0.99\textwidth]{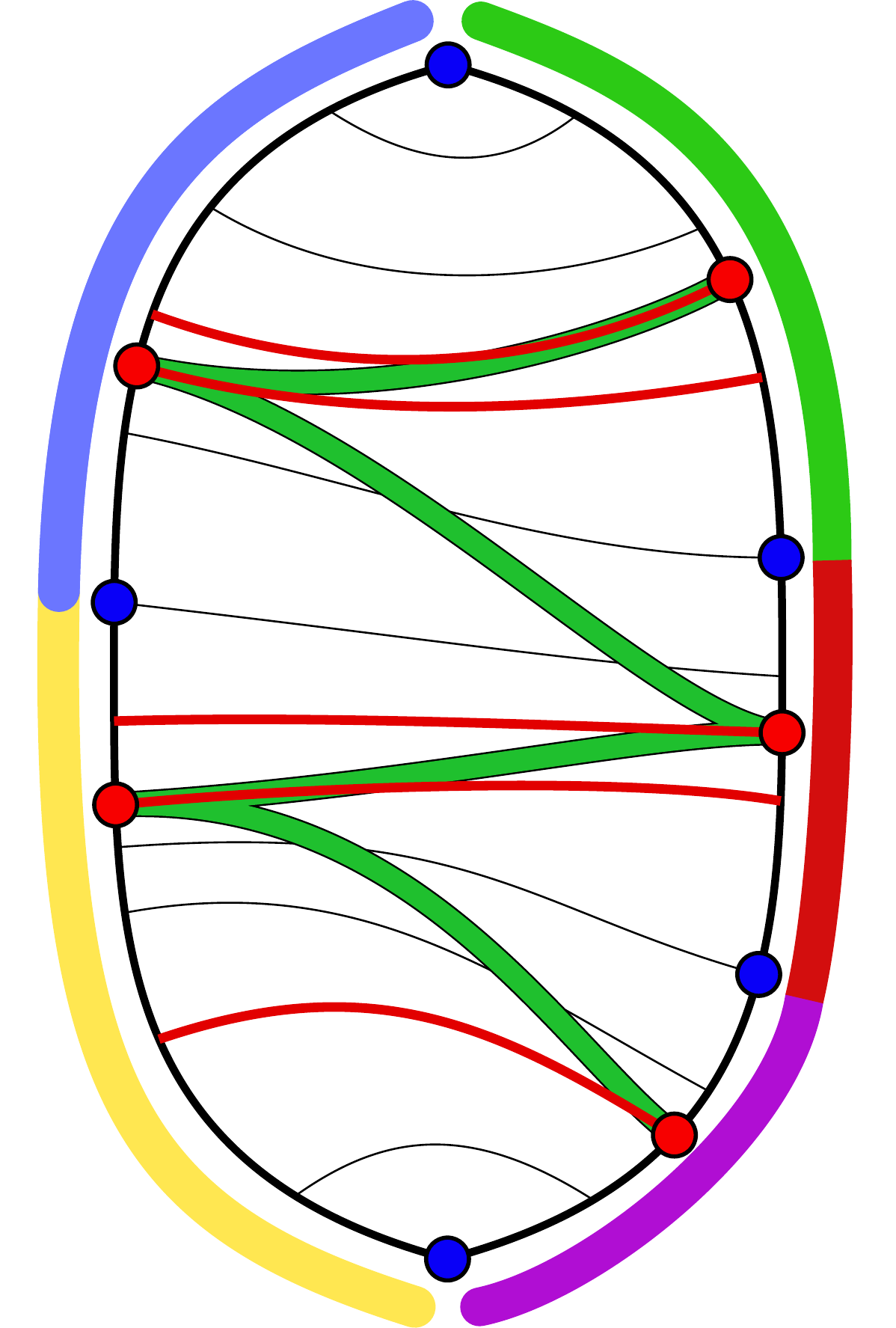}
\end{minipage}
\begin{minipage}{0.32\textwidth}
\begin{minipage}{0.50\textwidth}
\begin{center}
\includegraphics[width=0.99\textwidth]{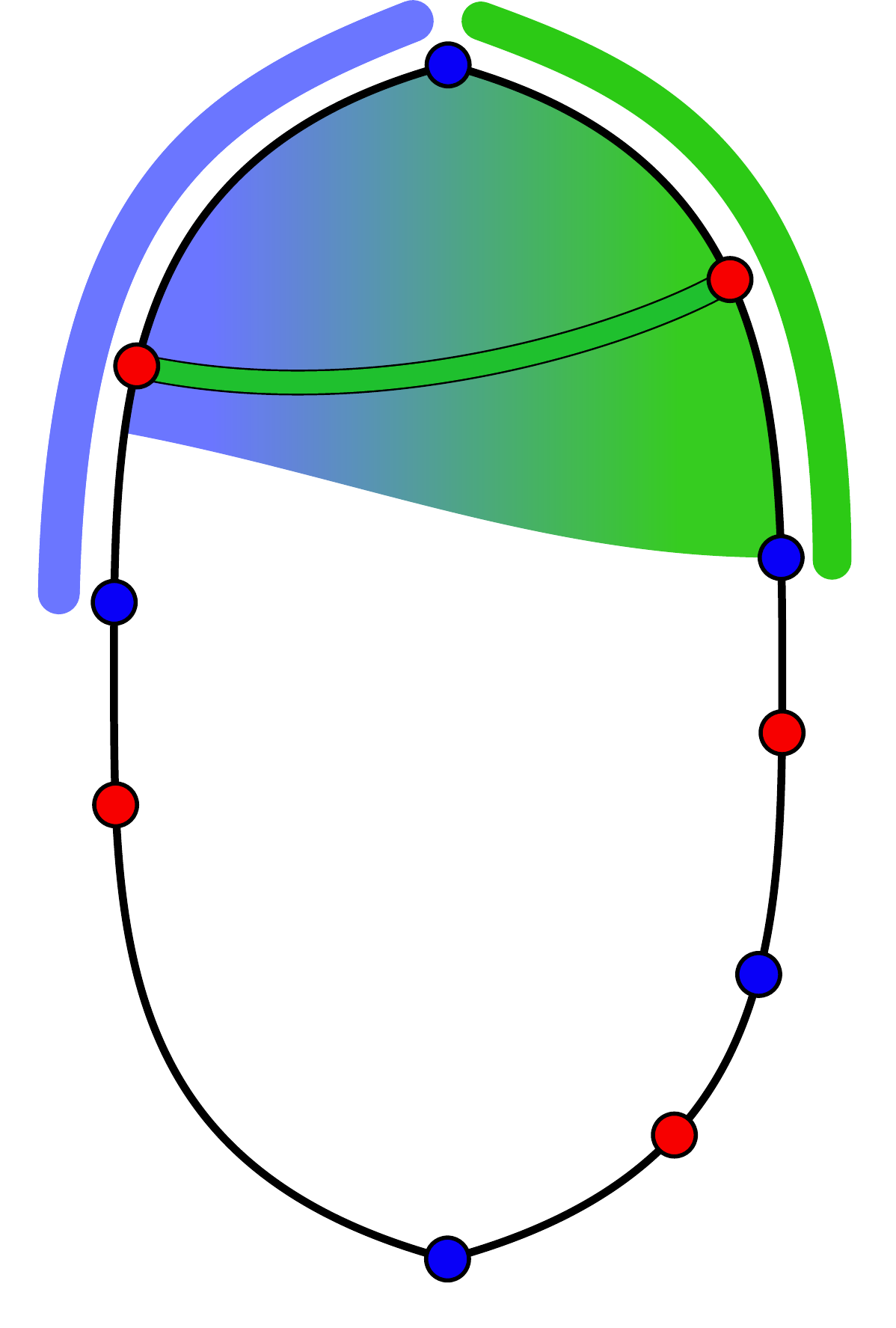}
\end{center}
\end{minipage}

\begin{minipage}{0.50\textwidth}
\begin{center}
\includegraphics[width=0.99\textwidth]{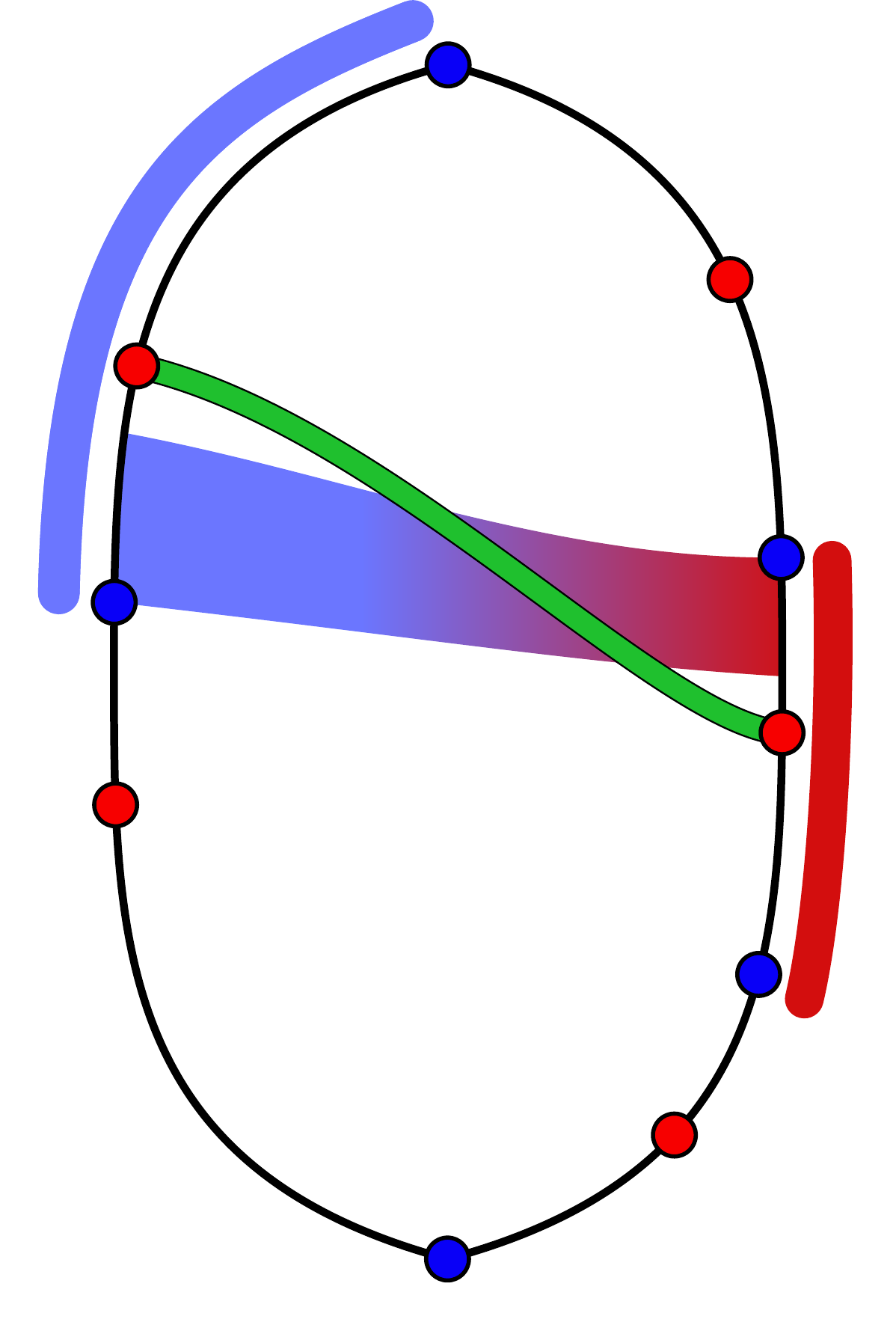}
\end{center}
\end{minipage}
\end{minipage}

\end{center}

\caption{Edge model for a region. On the left, isolines and supports are shown. At the center, maximal isolines and retrieved edges (in green). On the right, two examples of carry for selected edges \label{fig:edge}}
\end{figure*}

Isolines inside the region ideally form  a single bundle, since there are no saddles. The analysis of their spatial relationships suggests the presence of edges. In principle, closer isolines may indicate the presence of an edge. 
Rather than analyzing the whole region area, we focus on minimal and sufficient properties along both region paths, in order to infer the expected behaviour inside the region. This is not only computationally more efficient, but it guarantees to recover the correct connectivity of edges.

The key idea is to associate an edge to a bundle of isolines. The bundle may get divided by (i) saddles, which impose a clear and perceptual separation of edges and by (ii) local minima in gradient magnitude, which suggest a more refined partition into parallel edges due to the presence of multiple local maxima in gradient magnitude.
For (i) it is sufficient to verify the property at paths, while for (ii) the property is optional, depending on the type of perceptual model in use and it should be also verified also inside a region, in case of a simplified steepest graph (see Section~\ref{sec:opt}).
Let us characterize the bundles at path intersection.

At the paths, the isoline bundle can actually separate, due to the saddles and the behaviour of isolines when covering adjacent regions. Blue points in Figure~\ref{fig:edge} represent such locations. 
In the case (1) in which the point is a saddle, the isolines exiting the region will separate and, consequently, two independent edges will leave the region. 
Practically, in case of a split pixel, the path intersects one of its sides and we select the point along the path with value equal to the split value. Note that for split pixels, the point can lie at continuous coordinates, even if the path is discrete.
Note that the presence of a saddle can not be deduced only  by the analysis of the region pixels: it depends also on neighbor region pixels. Therefore the identification of saddles informs adjacent regions about the isolines bundles separation.

If the point is (2) a path extremum, this clearly limits the bundle. Finally, a  
more perceptually faithful version can also include (3) local minima in gradient magnitude as blue points. In this case, the isoline bundle defined by (1) and (2) can be further divided because of  the presence of multiple local maxima in gradient magnitude. The divisions induced by (3) do not affect the overall connectivity, since it only duplicates the affected edges. It can be activated depending on the kind of perceptual model in use. In some cases, where corrupting noise is present, it is more likely to observe several minima of type (3) that actually introduce parallel undesired minor edges (see Figure~\ref{fig:edges} at top left). In the reminder, we add all (3) points and we delegate the optional local minima removal to a post-processing phase.

\begin{definition}[Support and Span]
Let us define the set $B$ made of the following points on a min-max path: Min, Max, all split values and mix point values and local minima in gradient direction. The points in $B$ are ordered according to point image values.
Given two consecutive points $p_1$ and $p_2$,
the continuous interval of image values defined by $[R(p_1),R(p_2)]$ is named the \emph{support} of an edge for such path.
The restriction of the path between points $p_1$ and $p_2$ is named the \emph{span} of an edge for the path.
\end{definition}

In the Figure, the set $B$ is made of points \{Min, Lm1, Max\} for the left path and \{Min, Rm1, Rm2, Max\} along the right path. 
Spans are colored in blue and yellow on the left path and green, red and purple on the right path. 

Isolines contained in the region (depicted in black solid thin lines) are expected to separate when exiting the region, because of the presence of the saddles (compare Figure~\ref{fig:outlook}). This suggests a possible edge arrangement at the paths. 

\begin{definition}[Edge at path]
For each span of a path, there is exactly one edge crossing the path.
\end{definition}

The span supports the presence of exactly one edge. This means that the span describes 
the spatial range where an edge is locally and perceptually confined. 
This reminds of Canny's single response for an edge and an informed non-maxima suppression, where the support is explicitly modeled. 
The properties establishes what are the isolines associated to each edge (the support) and the portion of the min-max path (the span) that is considered to determine the location of the edge.

The localization of the exact point where the edge intersects the span (the red dot) can be defined, as in classical detectors, at the point that contains the local maxima in gradient magnitude. 
Since the presence of the span information allows a global analysis of image values along the path, we prefer to introduce a more perceptually pleasing location, given by the ideal barycenter of points of the span, where the weights are the gradient magnitudes associated to the points. Such computation outputs a point that is often very close to the local maxima, but in certain cases, where e.g. there is some noise in an almost constant positive magnitude path (e.g. a ramp edge on a wide span), it could help in placing the edge in a position more robust  to noise. Note that without the presence of the span information, this refinement would be impossible. In the reminder, when referring to local maxima in gradient magnitude, the concept can be substituted by the barycenter defined above. All the examples depicted use the barycenter computation.

On the left path (the one that delimits the region on its right while increasing the values), there are two local maxima of gradient magnitude (LM1 and LM2 red dots), while on the right path we depict three maxima (RM1, RM2 and RM3 red dots).

\begin{definition}[Edge strength at path]
Given an edge crossing a path, its \emph{strength} is defined as the ratio between its support size and its span length.
\end{definition}

The edge \emph{strength} approximates the image derivative along the span. However, a more refined characterization of edges is possible: considering the support size and span length as independent variables, a classification of edges as strong/sharp, strong/blurred, weak/sharp and weak/blurred is possible.

\subsection{Edge connectivity}
\label{sec:ed}

\begin{definition}[Edge graph]
The edge graph EG=(M,A) is defined by a set of nodes $M$ that contains all local maxima in gradient magnitude for each span of min-max paths in the steepest graph. The graph edges are directed edges that connect two maxima in $M$, such that the first one is at the left path and the second one is at the right path of the same region.
\end{definition}

The edge graph summarizes all edges disposition in the image and it is obtained by the union of the graph edges retrieved from each region. This graph describes the pure connectivity of image edges at region boundaries: the proper drawing of each graph edge within each region is described in Section~\ref{sec:edgeDrawing}.

Let us describe the edge connectivity model within a region. 
Isolines are bundled into supports along min-max paths, because of interactions between adjacent regions (i.e. saddles) and local minima in gradient magnitude. Within a region, the isolines can be separated by local minima.
We assume that within the space between left and right paths the bundles identified at paths are able to reorganize and shift from one set of bundles to the other one. This assumption is critical and it approximates the actual behaviour inside the region. This choice allows to completely skip the tracking of edges inside the region and to model both connectivity and edge drawing based on paths.
For this reason, supports are depicted by faded colors in Figure~\ref{fig:edge} on the left. The transition of spans is ideal, since spans do not have a specific characterization inside the region. 

The \emph{edge connectivity} is defined by a bipartite matching of the subset of $M$ that intersects the region, i.e. the red points in the Figure (center). 
Recalling the assumption made above, edges can split and merge only at region sides (edge junctions). If a split occurs inside the region, it can be detected and represented by two edges starting from a node and connecting two nodes on the opposite path.
The actual correct position of the junction in space is in fact approximated and this issue is discussed in Section~\ref{sec:edgeDrawing}.

In order to establish whether a matching is possible, let us define the measure that compares left and right supports.

\begin{definition}[Edge carry]
Given the points $LM_i$ and $RM_j$ in $M$, such that $LM_i$ lies on the left side and $RM_j$ lies on the right side of the same region $T$, the \emph{carry} of an edge in $T$ between $LM_i$ and $RM_j$ is the intersection of the supports associated to $LMi$ and $RMj$
\end{definition}

The possibility for an isoline bundle at left path to evolve into a bundle at right path, depends on the carry of the pair. If the interval is empty, it means that are no isolines connecting the two spans and no edge can be supported.
If there is a non empty carry, there is a set of isolines that connects the two sides, which suggests the presence of an edge.

%Carry information along min-max paths are the key to detect correct connectivity of edges, without the need to investigate region area. 
%For efficiency purposes, the edge support is not computed for each steepest path, but it can be approximated by interpolation between the two sides. \hl{da espandere un pochino}

In Figure~\ref{fig:edge} (center) the complete bipartite matching is shown.  The matching (no orientation depicted) is drawn with thick green lines between nodes.
For example, the blue support allows an edge towards the green support, since they have a non empty intersection. 
The yellow support is disjoint from green support and therefore no edge can connect the two spans.
Figure~\ref{fig:edge} (top right) depicts the isolines defined by the carry between blue and green supports and 
Figure~\ref{fig:edge} (bottom right) depicts the isolines defined by the carry between blue and red supports.
This last example shows a delicate case, perceptually speaking: no local maxima are included by the carry on the sides, but the edge should be drawn starting from the two red nodes. Moreover, the edge could also have no local maxima in gradient magnitude inside the region. This could lead to a false positive detection and a 
visually awkward  drawn line (compare with Section~\ref{sec:edgeDrawing}). We discuss the handling of such case in the optimization Section~\ref{sec:opt}.

Note that graph nodes are associated to positions in the image domain and they contain support and span information. However graph edges are not yet drawn, and therefore they do not represent a correct positioning of the actual edge lines. Nevertheless, edges contain carry information and the edge graph can be used for fast vector and graph based processing pipeline.
Note also that some graph edges can be directed self loops: even if the edge is not informative \emph{per se}, the actual edge is correctly drawn as a (possibly long) closed line across the image.

%We are finally ready to define an edge according to the model presented above.

\begin{definition}[Image Edge]
An image edge is defined by any path in the edge graph.
\end{definition}

An image edge can be retrieved by a graph processing, since edge graph contains no high-level interpretation.
The extraction of main edge lines, apart from trivial carry thresholding for graph edges,
 can be performed by greedy exploration of best carry edges or adaptations of shortest path algorithm for a more robust retrieval.

In literature, many types of edge have been formalized (step, ramp, line, roof etc.), depending on properties and spatial relationships among them~\cite{canny1987computational}. 
The edge graph contains only one basic type of edge, while the graph analysis can recover any complex relationship. Our model separates the edge detection task from the 
high-level interpretation of edges where, e.g., anti parallel edges are recognised as the support of a solid black line (roof edges). 

A perceptual characterization of contours is not uniquely defined, since texture, changes of luminance, global relationship of features~\cite{papari2011edge} and even optical illusions~\cite{opticalillusion} carry additional information that influences actual detection. We focus on the isolines information (original signal) and we delegate any further higher-level (and graph-based) interpretation to a next processing step.

\subsection{Edge drawing}
\label{sec:edgeDrawing}

A graph edge is a match between two sides of a region with non empty carry. 
The drawing of the graph edge produces a vector line which approximates at best the perceptual traversal on the region.
The goal is to produce a line that minimizes the distance from the local maxima in gradient magnitude, without computing them.
The proposed drawing is an approximation that simplifies the issues of tracking the edge local maxima inside the region. 

Since, by construction, each edge is correctly placed at the intersection of paths ($LM_i$ and $RM_j$ points in Figure~\ref{fig:edge} left), we compute the two isolines passing through $LM_i$ and $RM_j$.
In Figure~\ref{fig:edge} at the center, all red isolines for corresponding maxima are highlighted. 
The edge is confined in the area between the two isolines. The edge is drawn as a linear interpolation between the two associated isolines (see the green thick lines). The edge starts at left side with all the weight on the corresponding isoline, and while approaching the right side the weight is shifted according to the travelled distance. At the right side, the weight is given only to the corresponding isoline.
The edge drawing inside the region is rather predictable, given the region monotonicity and the absence of saddles. 
An edge line is a rather smooth and robust feature and usually it is rather parallel to the bundle isolines even for long distances. 
We argue that such retrieving can be sufficiently adequate in most cases.
%This allows a vectorial modeling of the edge, reduced computational costs and high precision. 
This relaxation is visually acceptable, since regions are usually small, especially for natural images.

The only drawback of such technique is that narrow regions and/or high difference in values introduce lines that significanlty cross isolines, since the region extension is not able to accommodate the transition between the two maximal values. Often such edges are weak (see Figure~\ref{fig:edge} bottom right) and they are perceptually irrelevant. Such edges can be rearranged through a graph processing that allows edges to cross multiple regions before the junction. In the next section we present a synthetic y junction example.

\subsection{Optimizations}
\label{sec:opt}
\subsubsection{Steepest graph simplification}

The steepest graph induces a partition of the image into regions. Often, see Figure~\ref{fig:edges} top left, many regions are only one pixel thin. In practice, such regions are rather redundant, since they do not introduce any additional connectivity to the edge graph. Computationally, such divisions require extra computation and they could be simplified by a single and larger region that captures the same behaviour of the included regions. 

The graph can therefore be simplified, based on a merging procedure that combines suitable adjacent regions.
Two regions can be merged if the common path does not contain any saddle (preservation of monotonicity). This simplification has little impact on connectivity.
In fact, no saddles are included in merged regions and the overall connectivity is unchanged. The
only difference is that some edges that were separating and merging because of a local minima in gradient magnitude on the removed path, become joined in the merged region. 

Another advantage in simplifying the steepest graph is that, since the localization of local maxima in gradient magnitude nodes on min-max paths is subject to small approximation errors of the min-max paths (since they are not perfectly orthogonal to isolines), the presence of narrow regions enhance such errors in edge positioning. Wider regions compensate such errors and the corresponding graph edge results more smooth. Nevertheless, drawn edges can undergo to any vectorial simplification, smoothing and minimization (as for active snakes), which is not considered in the paper.

\subsubsection{Local minima in gradient magnitude}
As described in Section~\ref{sec:em}, sometimes the presence of noise fragments one edge into a multiple set of parallel edges.
Sometimes a main edge experiences multiple divisions and merging with an unpleasing result as in Figure~\ref{fig:edges} top left. 
The information associated to nodes and edges describes the span and carry. A graph processing aiming at recognizing such fragmentations can merge the nodes along the same path, so that parallel edges can be joined as a single and stronger edge.
An example of such post-processing analysis is presented in Figure~\ref{fig:edges} bottom left, were edges are redrawn according to the simplified graph.

\begin{figure*}[ht]
\begin{center}
\begin{minipage}{0.40\textwidth}
\includegraphics[width=0.49\textwidth]{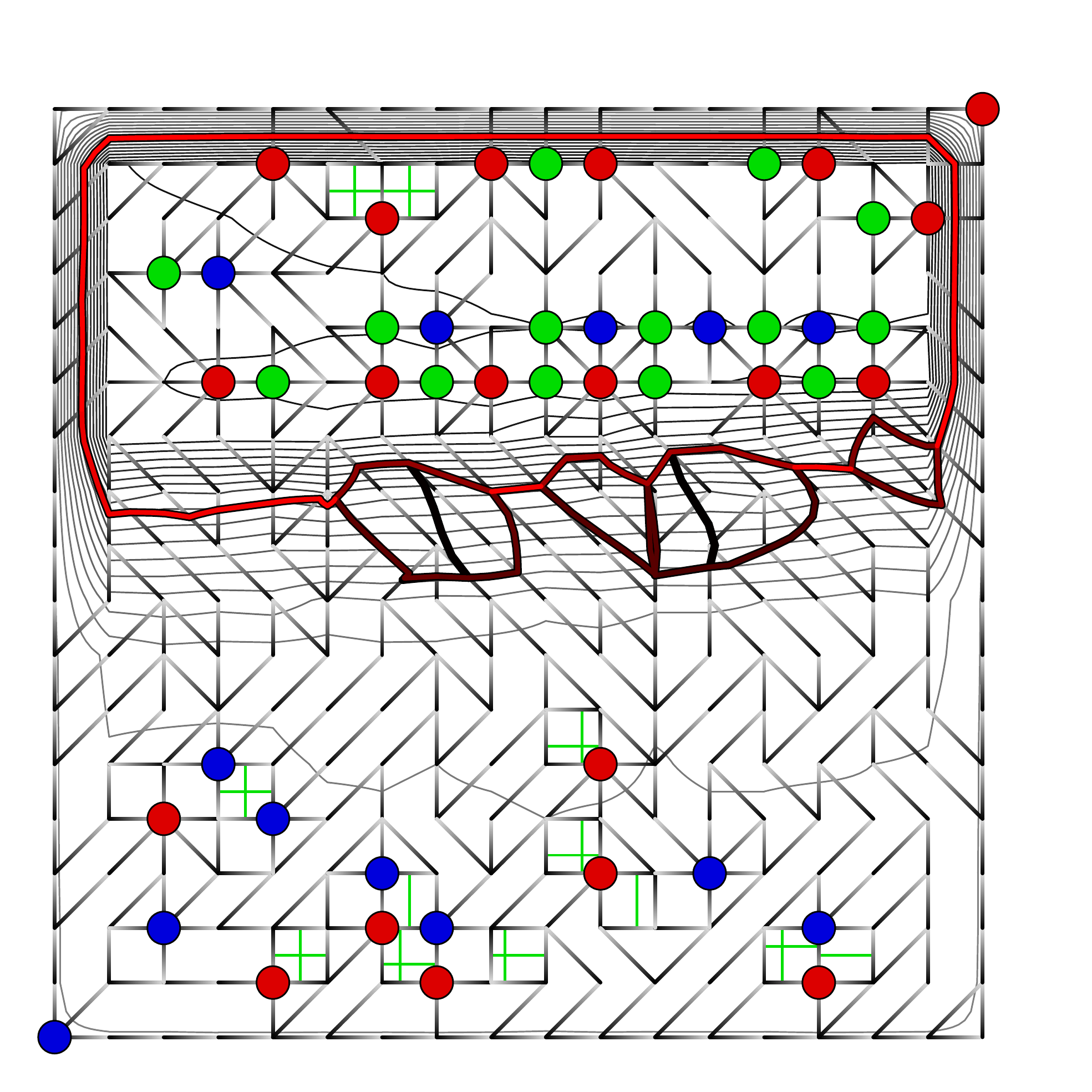}
\includegraphics[width=0.49\textwidth]{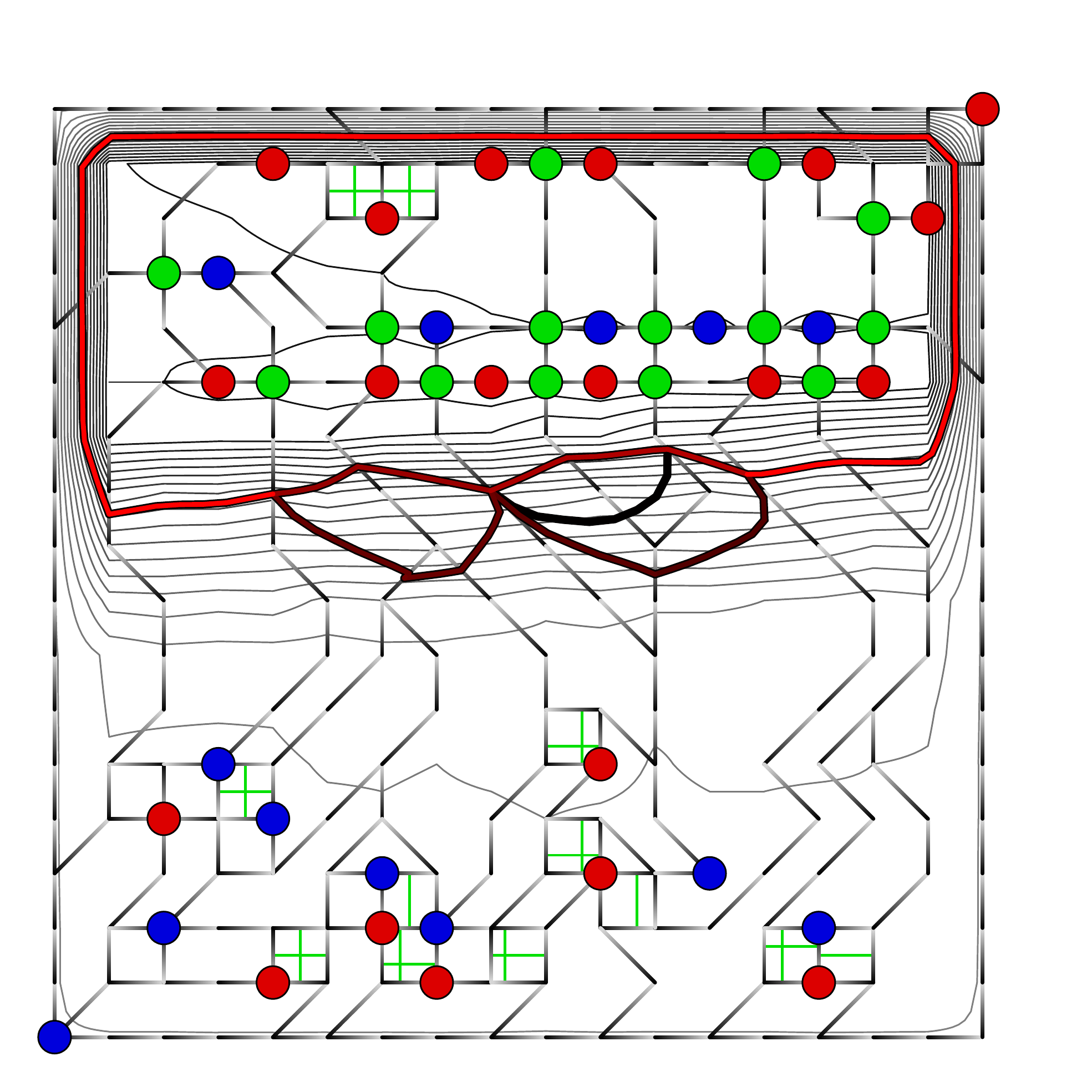}

\includegraphics[width=0.49\textwidth]{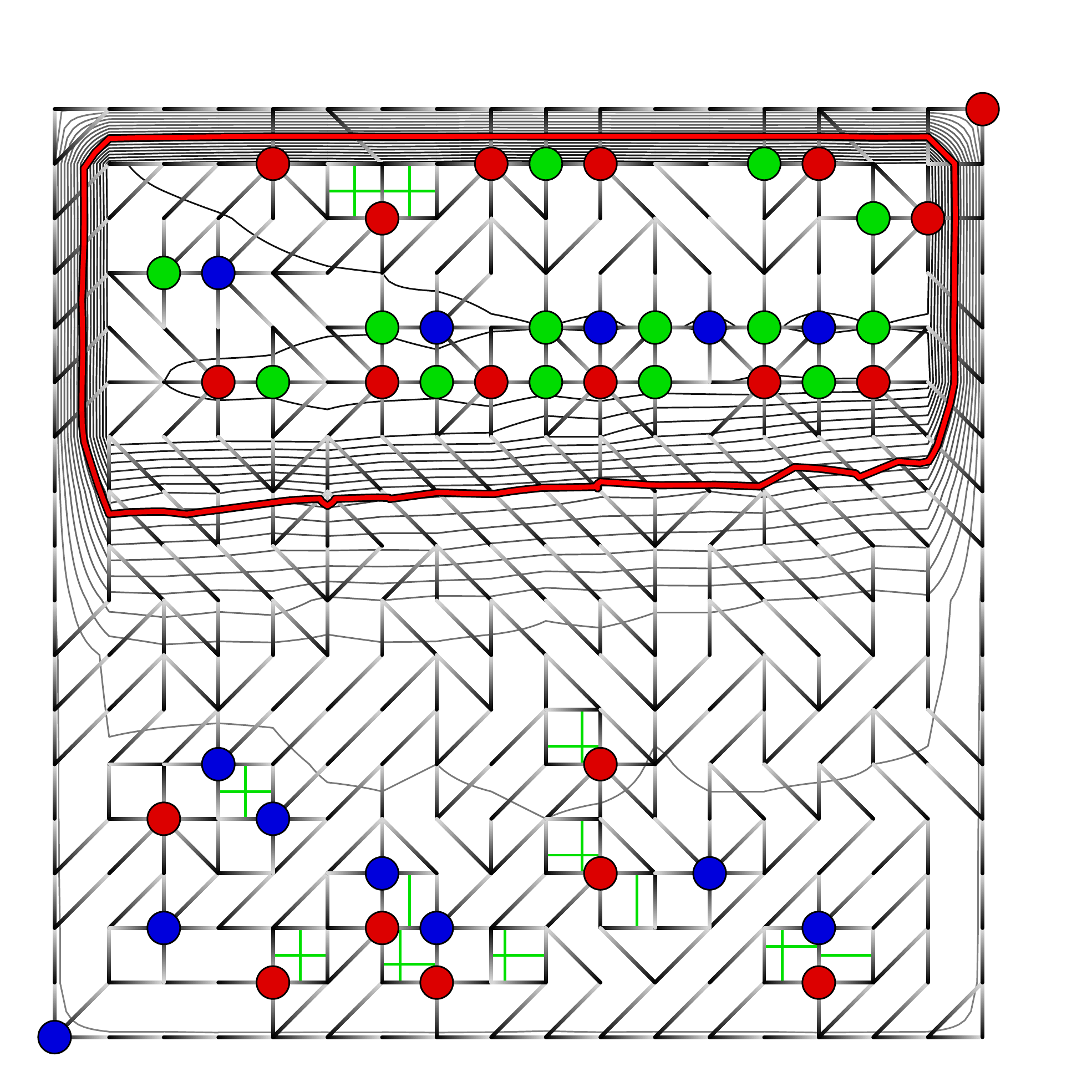}
\includegraphics[width=0.49\textwidth]{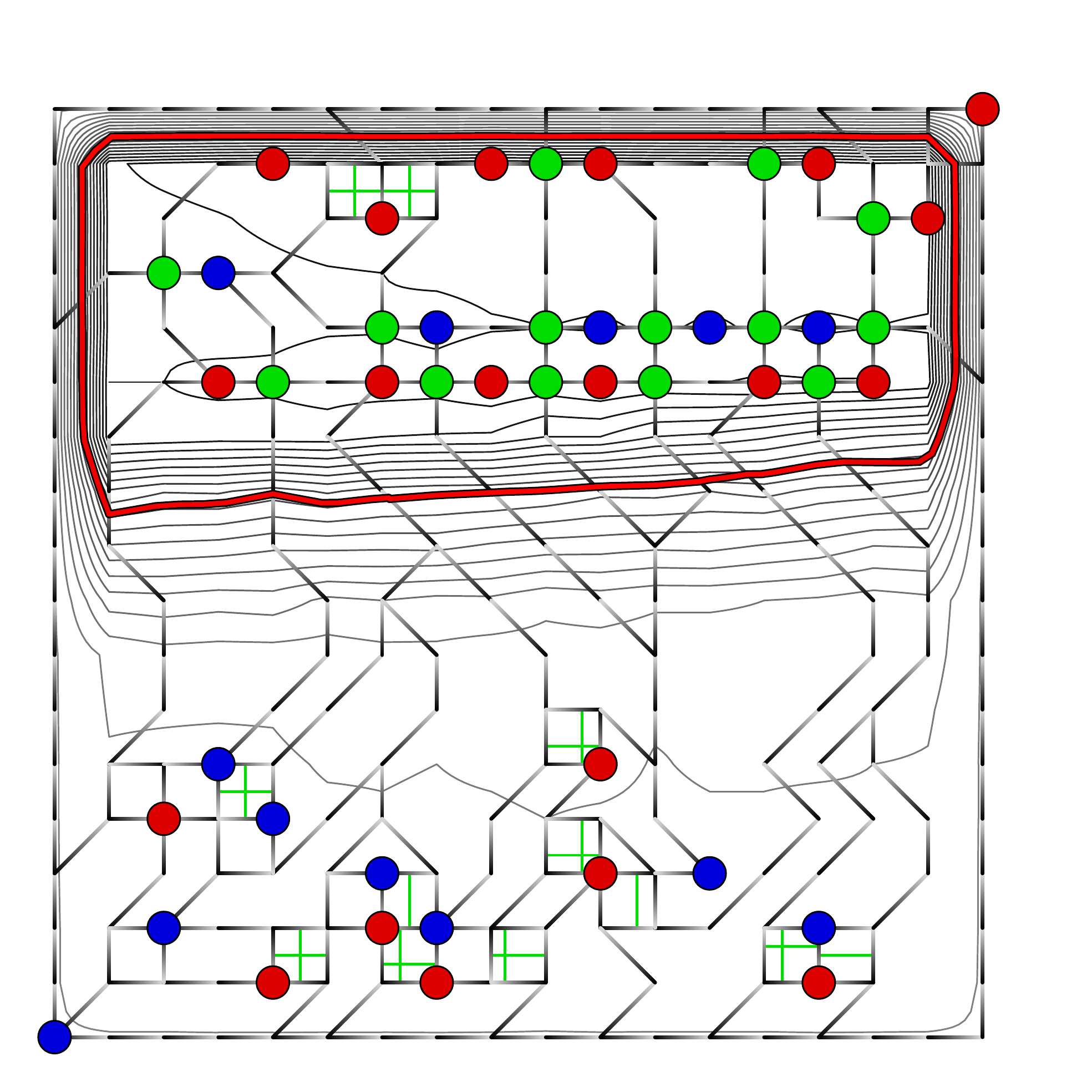}
\end{minipage}
\begin{minipage}{0.40\textwidth}
\includegraphics[width=0.99\textwidth]{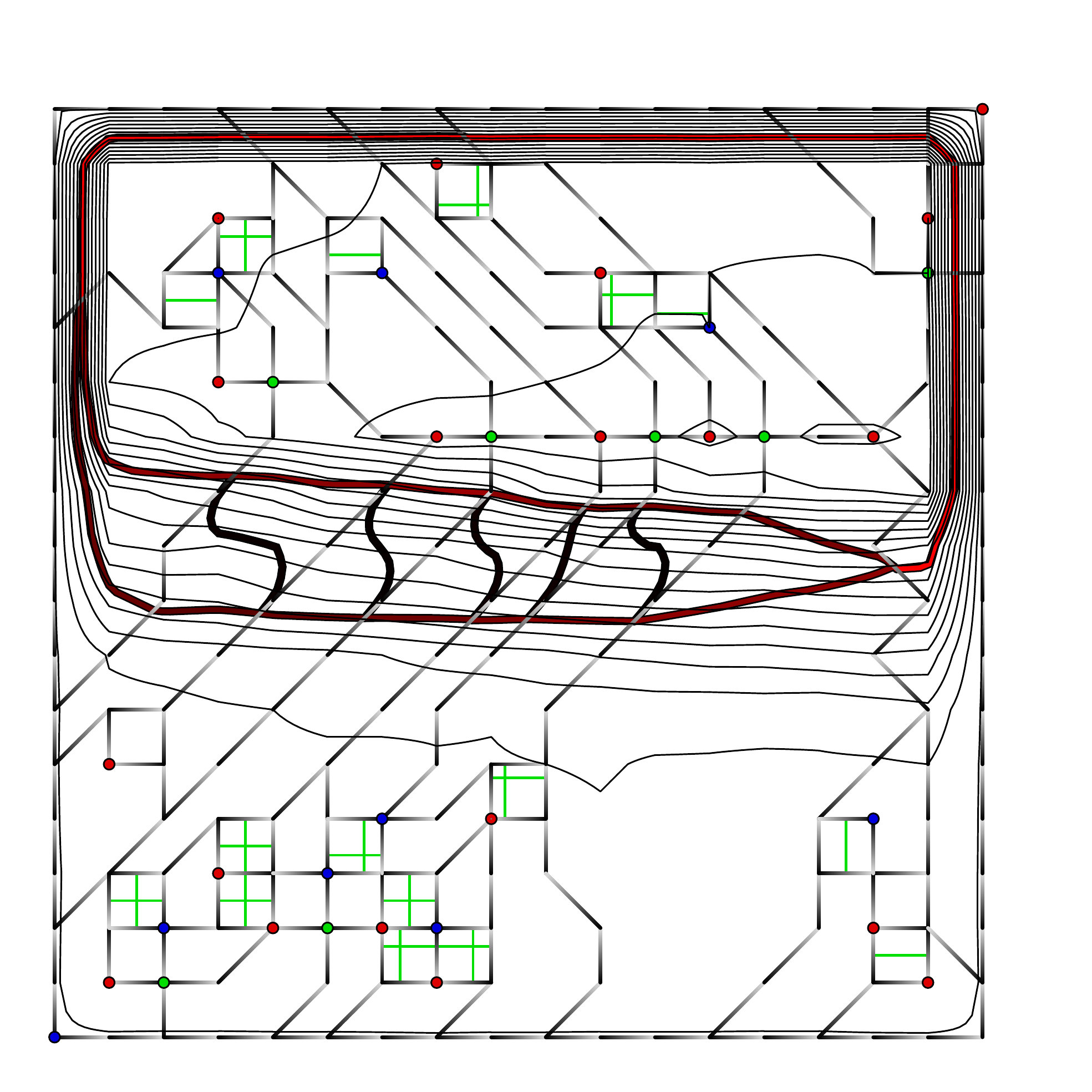}
\end{minipage}
\caption{Edge separation due to local minima and graph simplification (2x2 images on the left) and small carry edges (right)}
\label{fig:edges}
\end{center}
\end{figure*}

In Figure~\ref{fig:edges} on the left, we show a set of 2x2 detections for the same test image (16x16 pixels). For clarity, all weak edges have been removed and only the behaviour of the relevant edge is shown.
There are 4 scenarios being depicted: at the top, the detection of all points on the paths is shown, while at the bottom 
the post-processing procedure merges nodes that are separated by a local minima. On the left, the complete steepest graph is used, while on the right the simplified graph is used to determine edges. It can be noted that the presence of (almost unnoticeable) local minima in gradient magnitude causes the division of the main edge. For this test image, noise is enough to create local minima and corresponding parallel edges. 
Note that in this example, no saddles cause edge separation. Stronger noise can also introduce saddles. In this case, the graph post-processing should analyze independent paths and decide to merge them according to the local behaviour.

\subsubsection{Perceptually pleasing edges}

\begin{figure*}[ht]
\begin{center}
\includegraphics[width=0.30\textwidth]{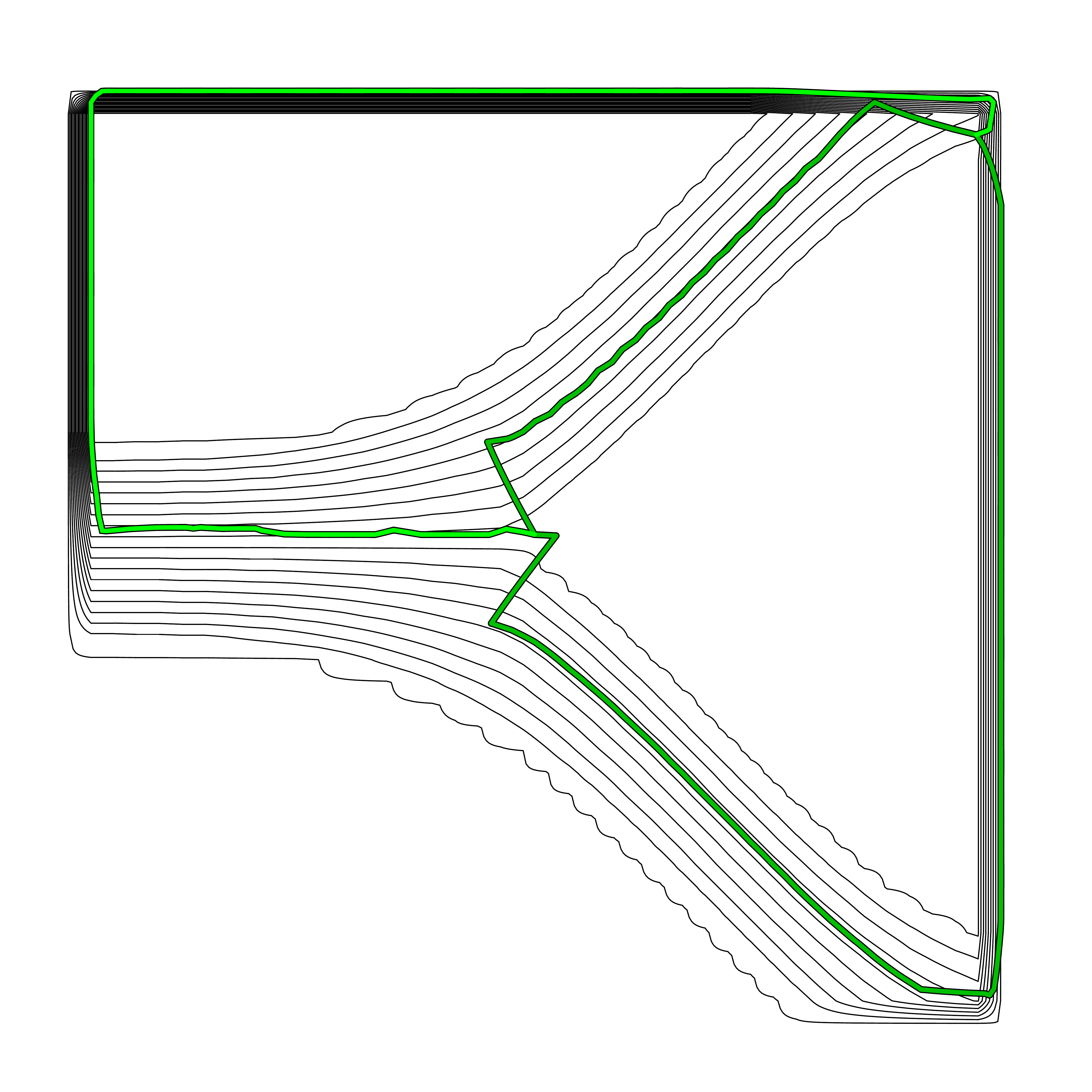}
\includegraphics[width=0.30\textwidth]{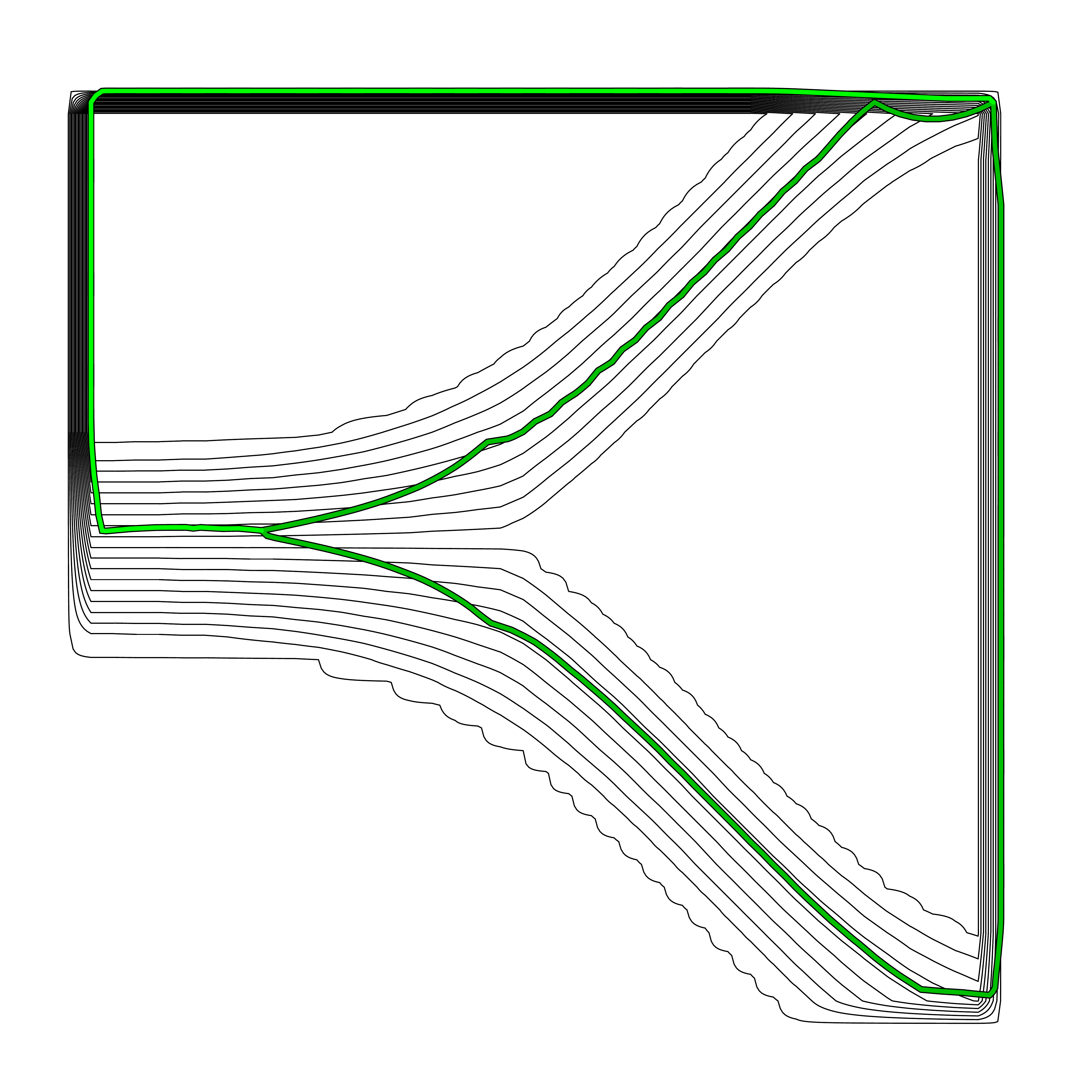}
\caption{Handling of edges at a y junction \label{fig:y}}
\end{center}
\end{figure*}

The edge graph delivers correct connectivity, however, when edges are drawn, the resulting lines may contain abrupt changes.
An example of a 24x24 image with a synthetic 'y' junction is depicted in Figure~\ref{fig:y}. Isolines for bilinear interpolated pixels are drawn with black lines and corresponding edges are depicted in green.
The junction point is recovered in the geometrically correct place, since on the left there is a single undivided bundle of isolines, while on the right the bundle get separated into two parts.
Depending on the perceptual interpretation of junctions, it is possible to post-process the edge graph and to adapt the location of the  junction further on the left. Basically, a node's support is partitioned into two new nodes along the same path, if there are two edges starting from the original node and they have a significant difference in image values (the drawn line crosses many isolines on a short distance). In Figure~\ref{fig:y} at right, we show an example of such post-processing, with a more natural interpretation of the junction. Since the edge graph contains all the needed information, it is possible to define a model for junctions and apply the post processing  to produce the desired rendering of edges.

Figure~\ref{fig:edges} on the right presents an example of edges with small carries that do not contain local maxima. In such cases, the rendered edges appear as an 's' shaped line, which connects the main nodes with an almost orthogonal junction. 
Given the minimal influence, compared to principal carries of the region, the edge could be removed. Another option is again to post-process the edge graph and divide the bundle support and forward the junction to the adjacent regions. The only drawback is to increase the number of parallel edges and in general this could lead to a visually cluttered result.

\subsection{Comparison with Canny's detector}
In this section we propose a critical review of the classical Canny's detector, under the point of view of the edge model presented above.

The detector starts with an image filtering by a Gaussian smoothing controlled by a parametric and user selectable $\sigma$ value. It removes weak minima, strengthens relevant edges and merges close-by local maxima ridges into a single and stronger edge. 
The automatic tuning of $\sigma$ for optimal results is difficult~\cite{lindeberg1998edge}, and 
different regions of the image would benefit from different parameter values. The practical effect of smoothing is that weak but clear edges are removed and edge connectivity gets disrupted. The filtering worsen the treatment at edge junctions since it decreases the  gradient magnitude at the junction point.
Our edge extraction, instead, exploits the full amount of information of the original image and it separates the noise removal and restoration from the geometric detection of edges. 

Canny's detector works at pure pixel level. Gradient magnitude is computed and stored as new raster information, as result of a convolutional filter. Our approach computes gradient magnitudes along discrete steepest paths and they are treated as monodimensional information along the best local direction. Sobel convolution computes only gradient magnitude while gradient direction is lost. Implicitly, each pixel stores information about a single principal direction and it ignores the presence of saddles.
This limits the precision of the edge positioning to integer coordinates and it introduces
ambiguities in chaining the correct neighbors and branching detection, especially in low resolution context.

The non maxima suppression corresponds to the approximation of 
our identification of the maximal gradient along every steepest path. Canny's algorithm performs a local non-maximality test, that  removes the candidate pixel upon failure. 
The procedure is performed for each pixel and the information about edge support, saddles and span is simply ignored. 
After processing, the set of maximal pixels has also lost information about pixel neighbors, gradient directions and their connectivity.
Basically, the structure of min-max paths is not considered at all. 

The consequence is that the next phase executes an uninformed and local pixel chaining: 
the greedy iteration starts with a gradient magnitude pixel greater than a high threshold, conquers the best next local maximum (in the $n8()$ neighborhood), until a low threshold is reached  (compare to the match of maxima on region right and left paths and edge drawing). During chaining, there is no knowledge about carry-based edge connectivity, nor edge correct direction related to gradient orientation.
The low threshold avoids to loose track of the edge direction in favour of noise. This is undesirable because the side effect is to skip weak yet relevant edges with low carry and large spans.
Moreover, edge connectivity can not be reconstructed in weak signal conditions, since only the predominant edge is followed, while the secondary ramification does not reach a sufficient threshold to get connected to the main edge.
In our model, noise can produce a complex connectivity graph, but connectivity can always be reconstructed by graph analysis.

An edge should be natively described by a line (in terms of vector line with continuous coordinates) rather than a reconstructed sequence of chained pixels, which are modelled as a vector polyline with integer coordinates. \emph{A posteriori} edge subpixel refinements are biased by the lack of information lost during convolutional smoothing and derivative filters.

\iffalse

\begin{figure}[ht]
\begin{center}
\includegraphics[width=0.6\textwidth]{canny.pdf}
\caption{Canny vs edge model\label{fig:canny-vs-edge}}
\end{center}
\end{figure}

\hl{potrebbe essere carino, ma immagine da rifare meglio e anche descrizione...}
Figure~\ref{fig:canny-vs-edge} depicts the comparison between Canny's detector and our method.
Canny's detector is pixel based: each pixel is tested whether is part of an edge, it is connected to neighbors and a list of pixels form the edge.
In our method, the min-max structure is built (it can be done pixel-based with no drawbacks), model the edges properties along each region sides (exact, assuming $R$ function), edge are properly connected and drawn (vectors). Properties are associated to 
edges (carry, support).

\fi

\section{Results}
\label{sec:res}

We implemented a single threaded benchmark program, with no relevant optimizations, run on a 2.5GHz Intel i7, 16 GB RAM 1600 MHz DDR3 laptop. The construction of complete steepest graph is rather stable and it costs 60 mS per Mpixel. Next phases depend on the complexity of the image. The graph simplification costs 5-6 times more. The edge graph construction takes around 3 seconds per milion of graph edges. The drawing of edges takes around 4.5 seconds per milion of graph edges to be rendered. In the context of parallel optimization, the algorithm is expected to run in real time for images up to 1 Mpixel size.
At \url{www.unipr.it/~dalpalu/edges} a more comprehensive set of tests is reported. For size requirements, our results are rendered as high resolution rasters, however at specified url the vector pdf files are downloadable for full comparison.

In Figure~\ref{fig:test1}, we detail the complete detection procedure for the same test image of Figure~\ref{fig:edge-isolines}. Steepest graph (at top center) is depicted with directed graph edges (black to white lines). The edge graph is with red to blue edges (top right).
Drawn edges (bottom left and center) are colored by carry value. Green edges are shifted by (+0.5,+0.5), in order to adapt to visual perceptual interpolation of pixels.
The steepest graph used for this test is simplified, in order to have a smaller edge graph. 
As last figure (bottom right), we show the comparison with the output of Canny's detector, computed with the built-in function in OpenCV~\cite{opencv_library} (Sobel mask 3x3 and $\sigma=2$) with gradient magnitude over-imposed to detected edgels. Note that our processing does not compute any Gaussian filtering. 

\begin{figure*}[ht]
\begin{center}
\includegraphics[width=0.32\textwidth]{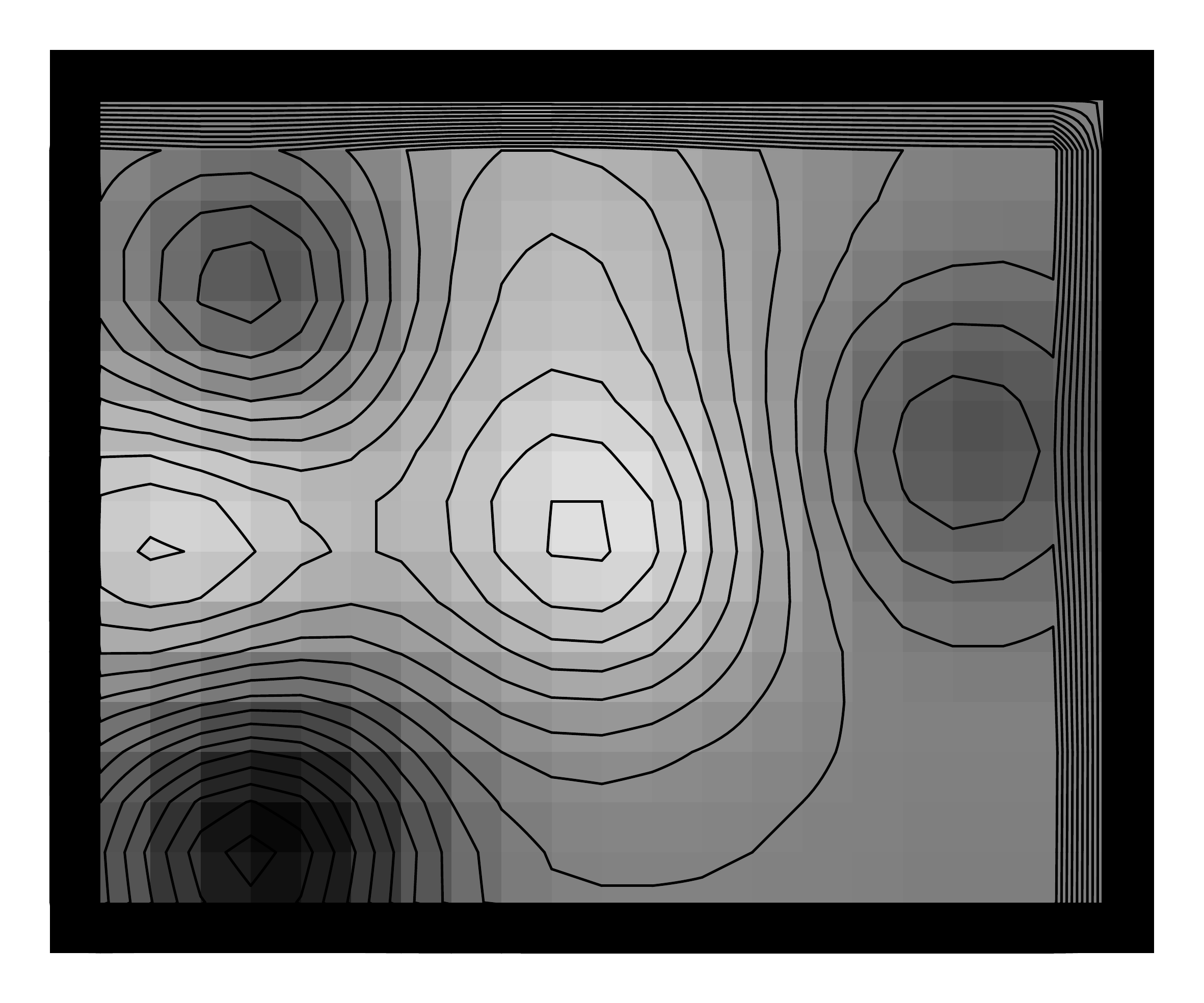}
\includegraphics[width=0.32\textwidth]{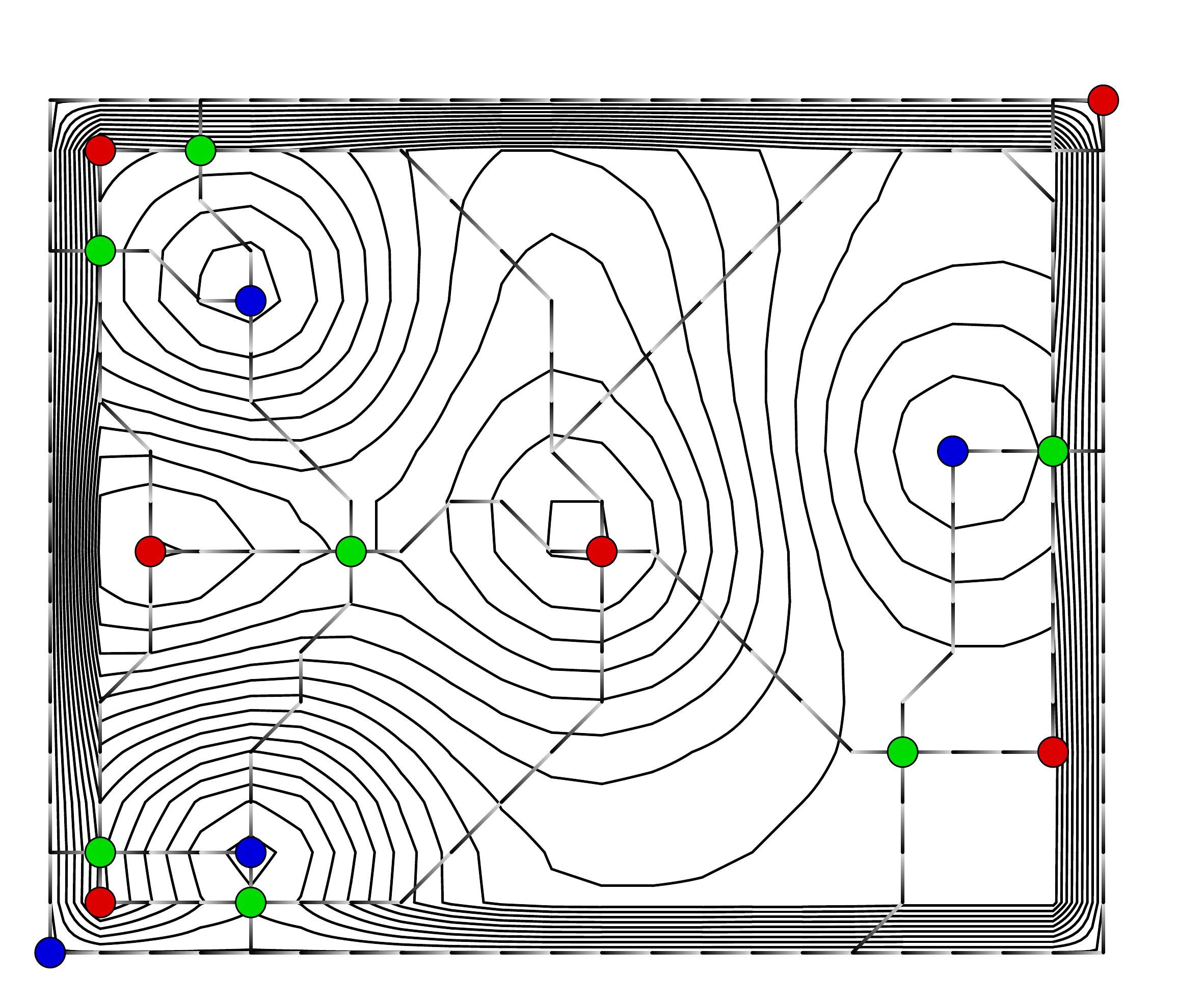}
\includegraphics[width=0.32\textwidth]{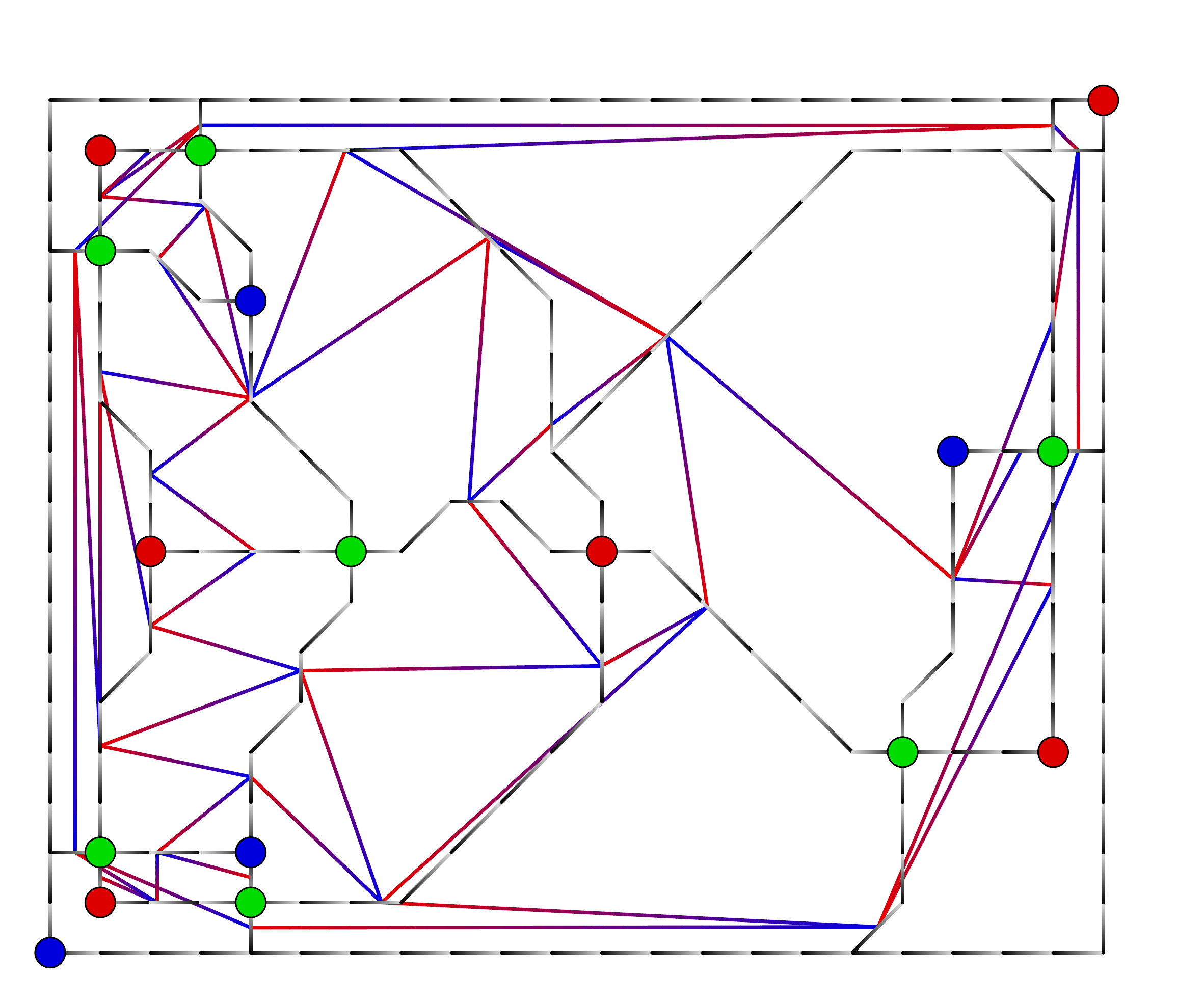}

\includegraphics[width=0.32\textwidth]{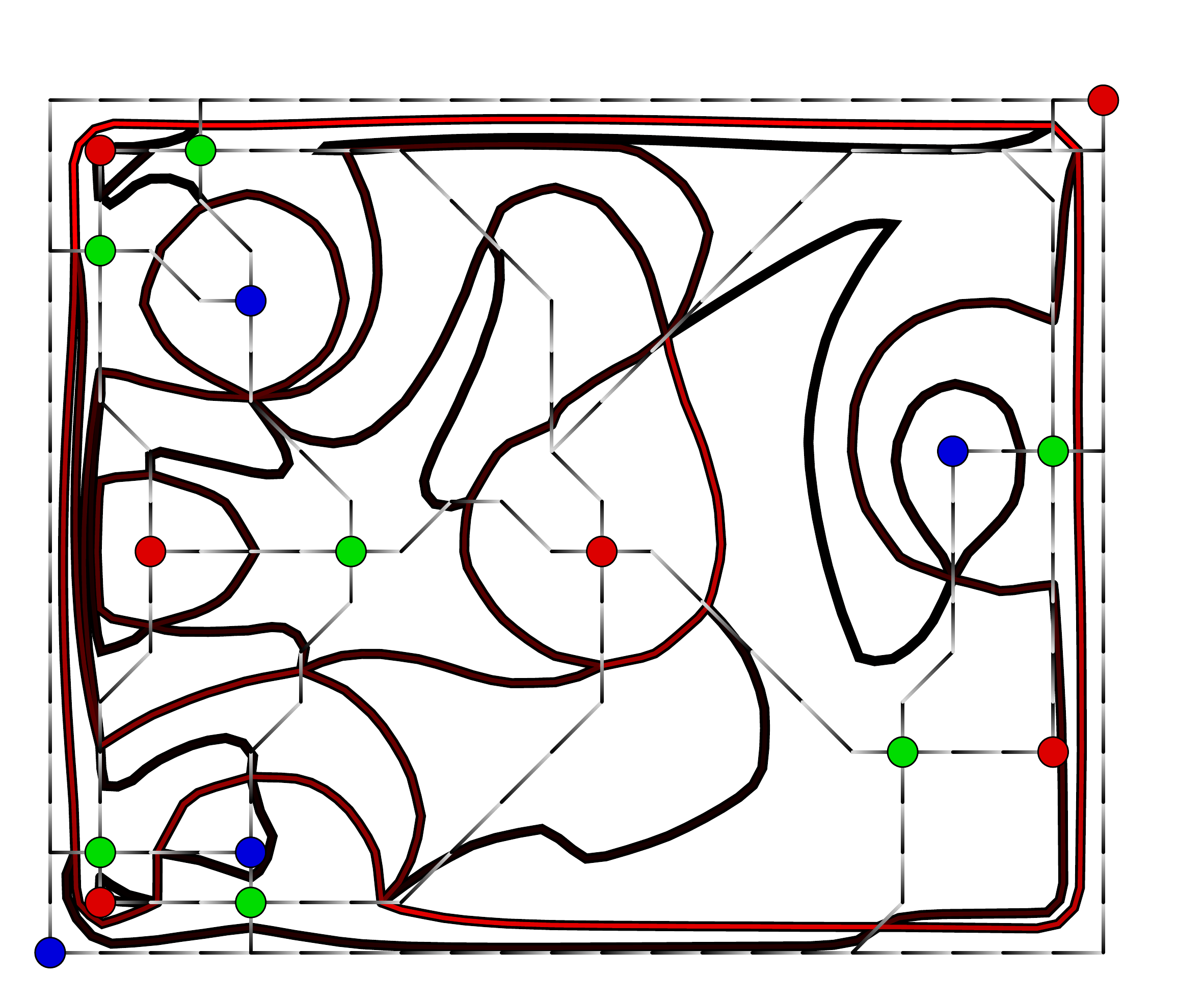}
\includegraphics[width=0.32\textwidth]{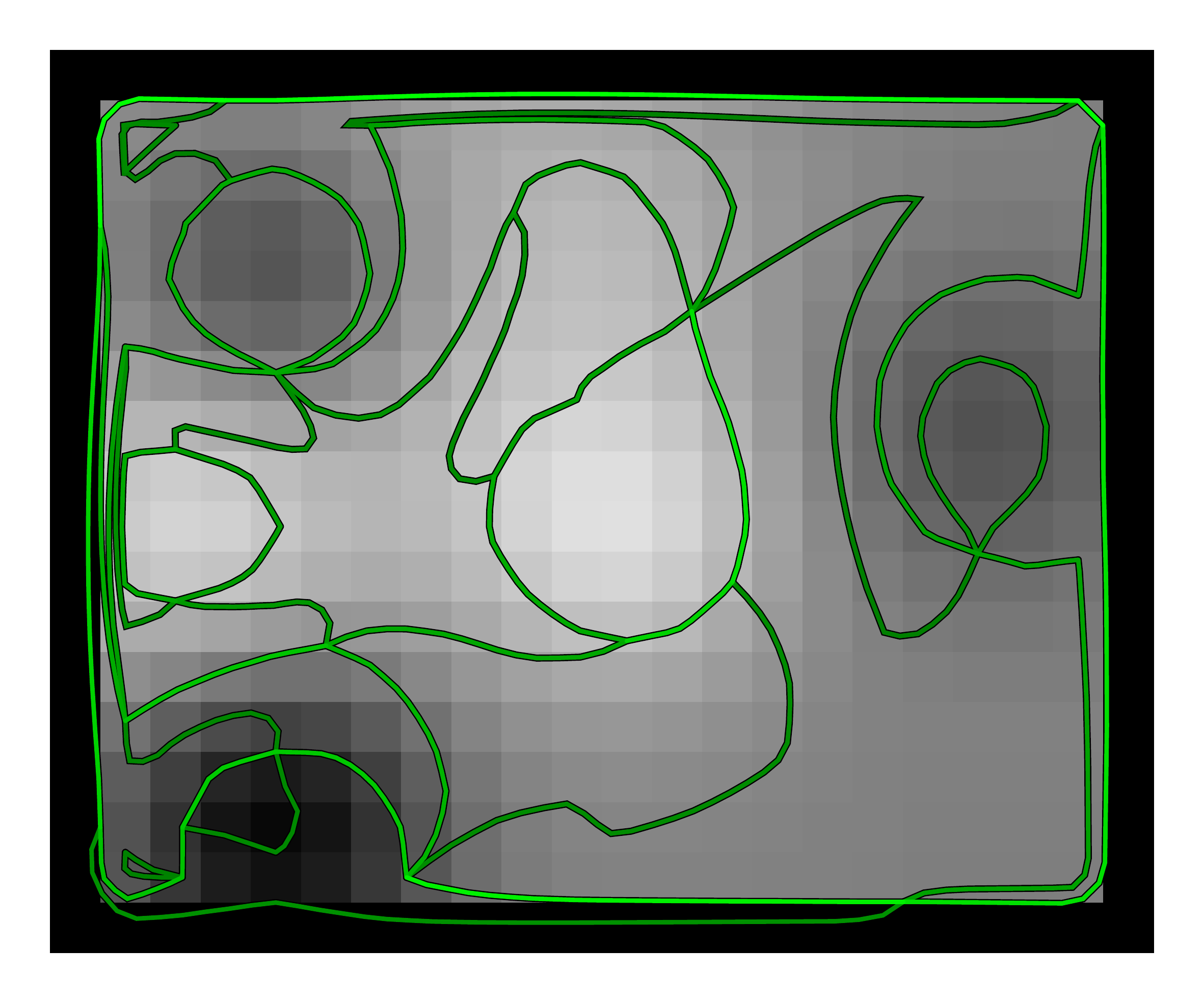}
\includegraphics[width=0.32\textwidth]{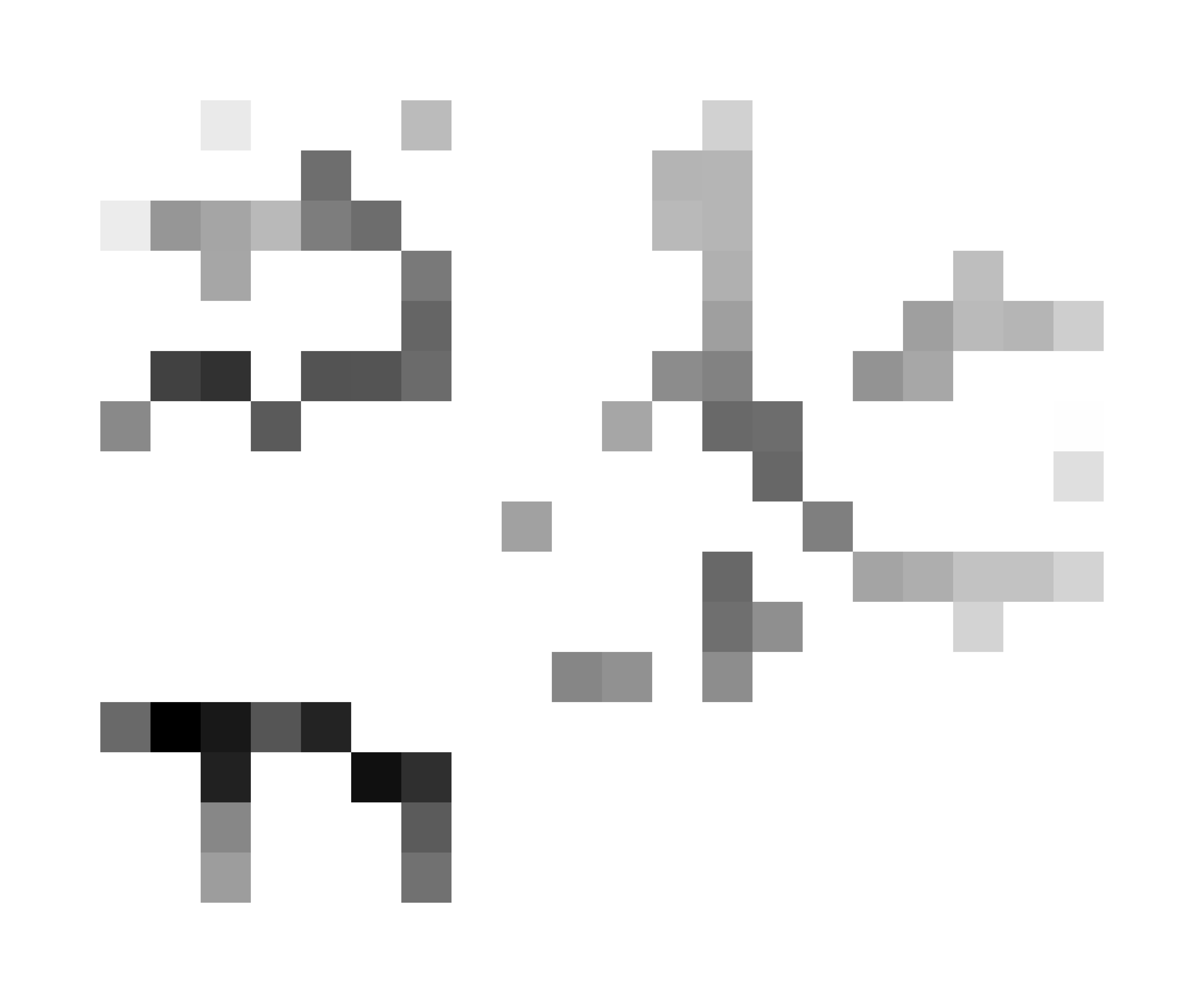}

\caption{Top row: Original image and isolines (left), isolines and steepest graph (center), steepest graph and edge graph (right). Bottom row: Steepest graph and drawn edges (left), original image and drawn edges (center), Canny's detector (right)    \label{fig:test1}}
\end{center}
\end{figure*}

Three synthetic small 24x24 pixels images are tested in Figure~\ref{fig:test2}. They capture some challenging behaviours that can occur at small scale. The first one is a bi-dimensional sinusoidal wave with a period of 4 pixels. The detector correctly captures the circular peaks and valleys. Moreover, small carry edges (depicted in darker green) capture the small transfer of isolines between the peaks. Note how the non maxima suppression in Canny's detector struggles, given the high variability of gradient direction that is not properly detected by Sobel convolution.

The second example investigates slightly rotated lines of 1 pixel width. Here the antialiasing effect is of help in determining the correct position of the main edge, which actually takes advantage of subpixel precision. Note that the detector also identifies the white and black oval shape islands that represent a dual image interpretation. Canny's detector is not capable of detecting such detailed configuration.

The third example presents a diagonal ramp edge that is embedded in a small vertical gradient. The isolines contribute to the identification of the edge, but the bundle that supports an edge varies along the diagonal. It is interesting to note how the main edge is recovered and how the network of low carry edges brings small bundles of isolines to the fading edge at the border. The right border of the image shows how the fading edge is handled. Part of the isolines that are lost from the vertical edge are connected to the diagonal edge. The almost orthogonal connection at the diagonal edge has been discussed in Section~\ref{sec:opt}.
Canny's detection correctly identifies the strong diagonal edge, but it fails at image corners and fading edge.

\begin{figure*}[ht]
\begin{center}
\includegraphics[width=0.30\textwidth]{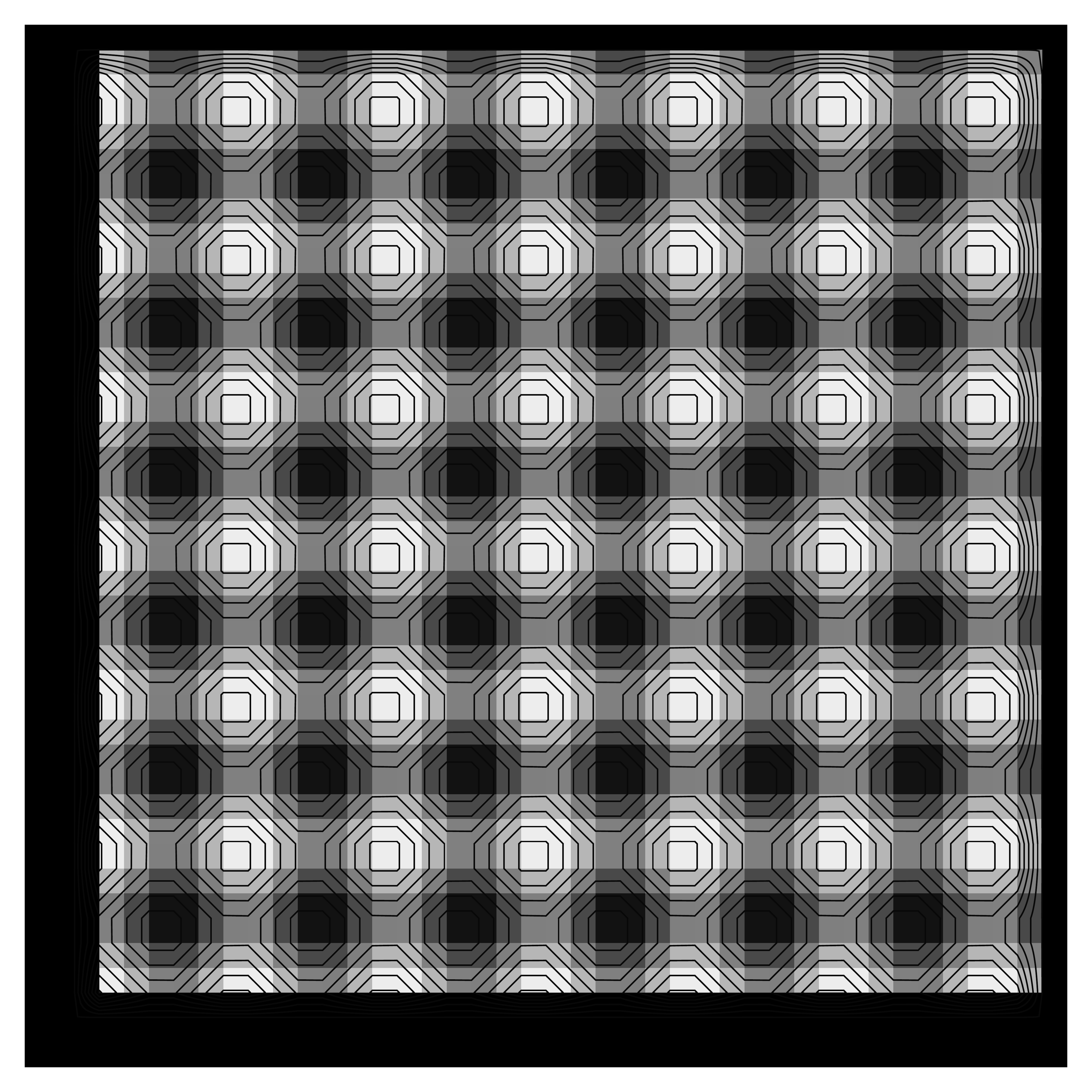}
\includegraphics[width=0.30\textwidth]{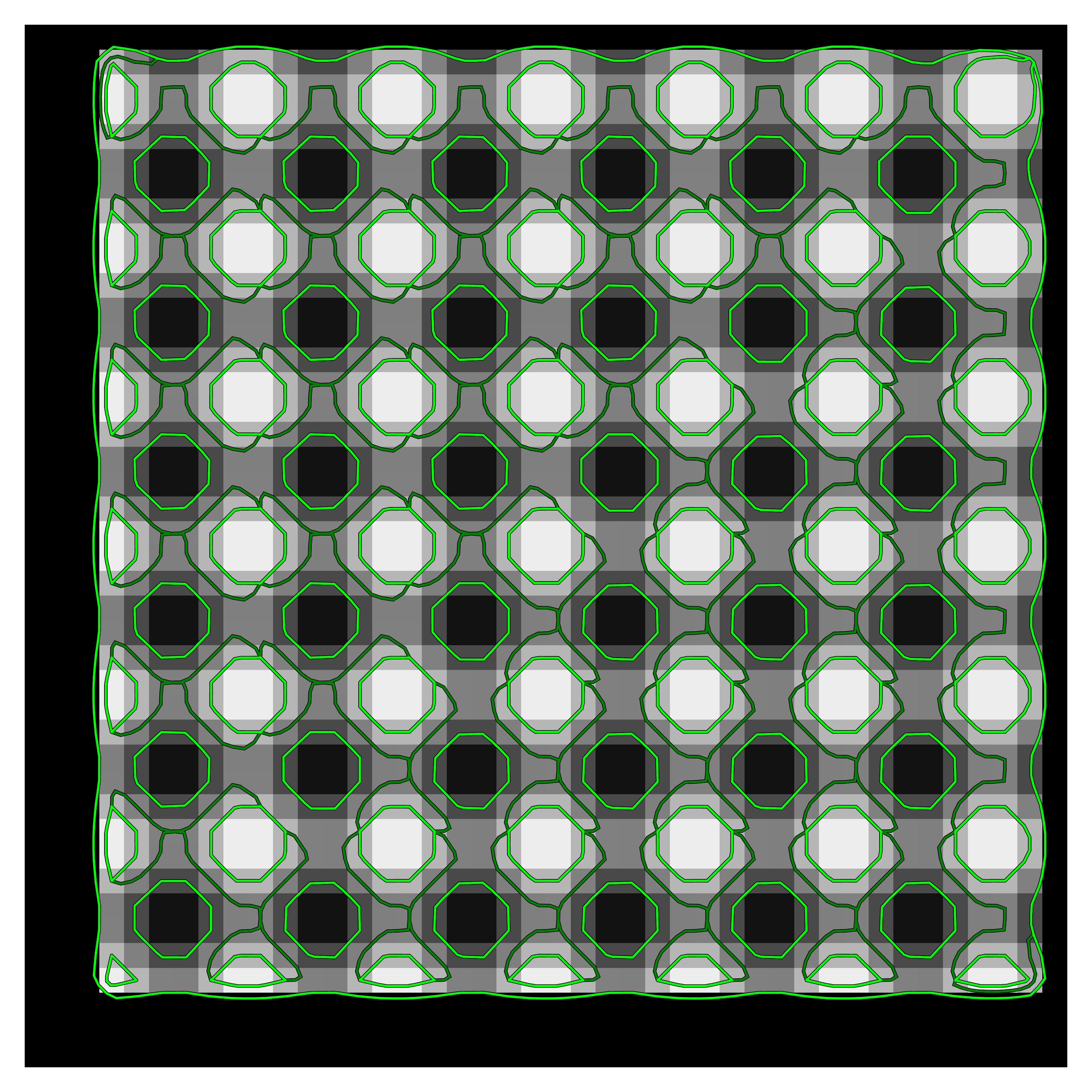}
\includegraphics[width=0.30\textwidth]{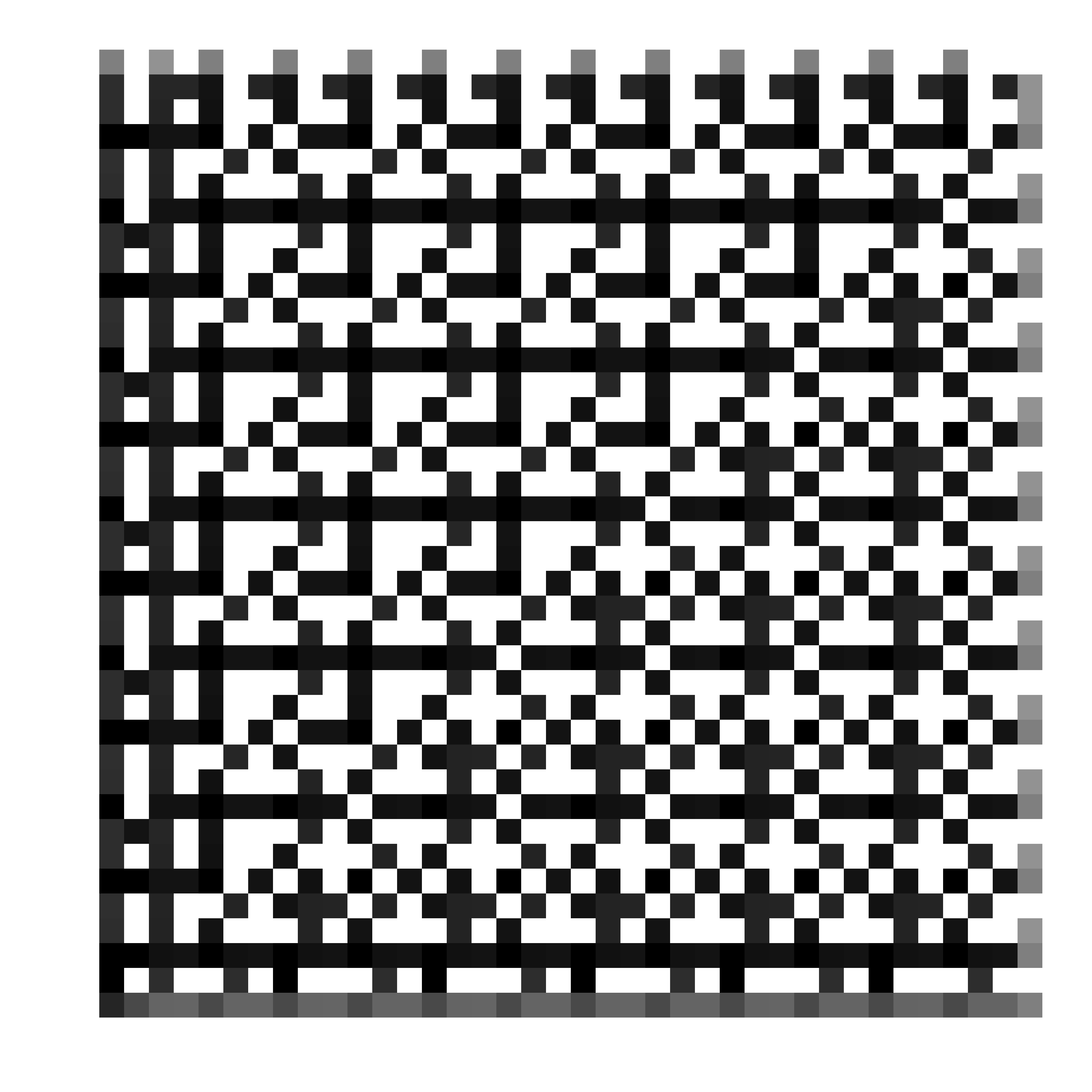}

\includegraphics[width=0.30\textwidth]{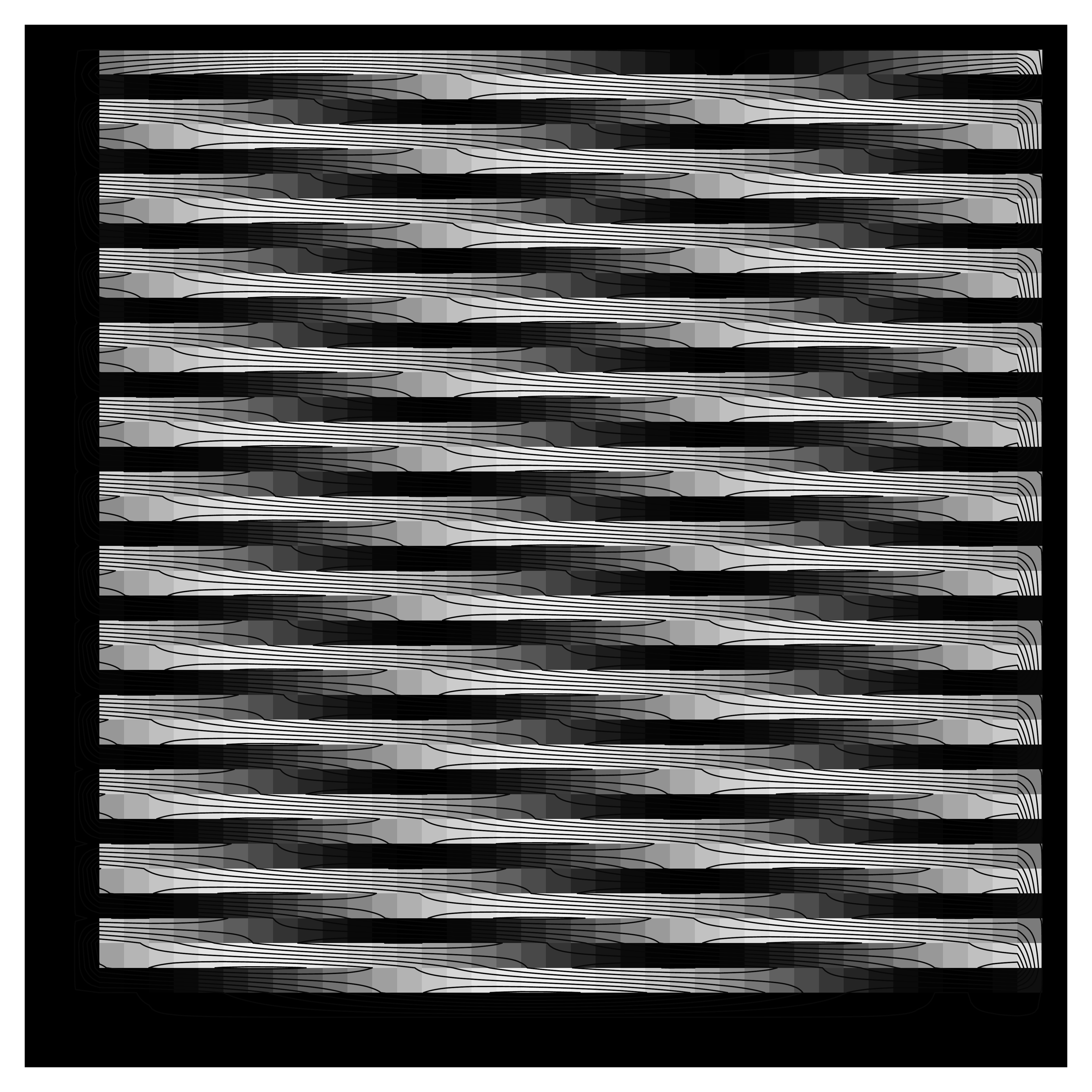}
\includegraphics[width=0.30\textwidth]{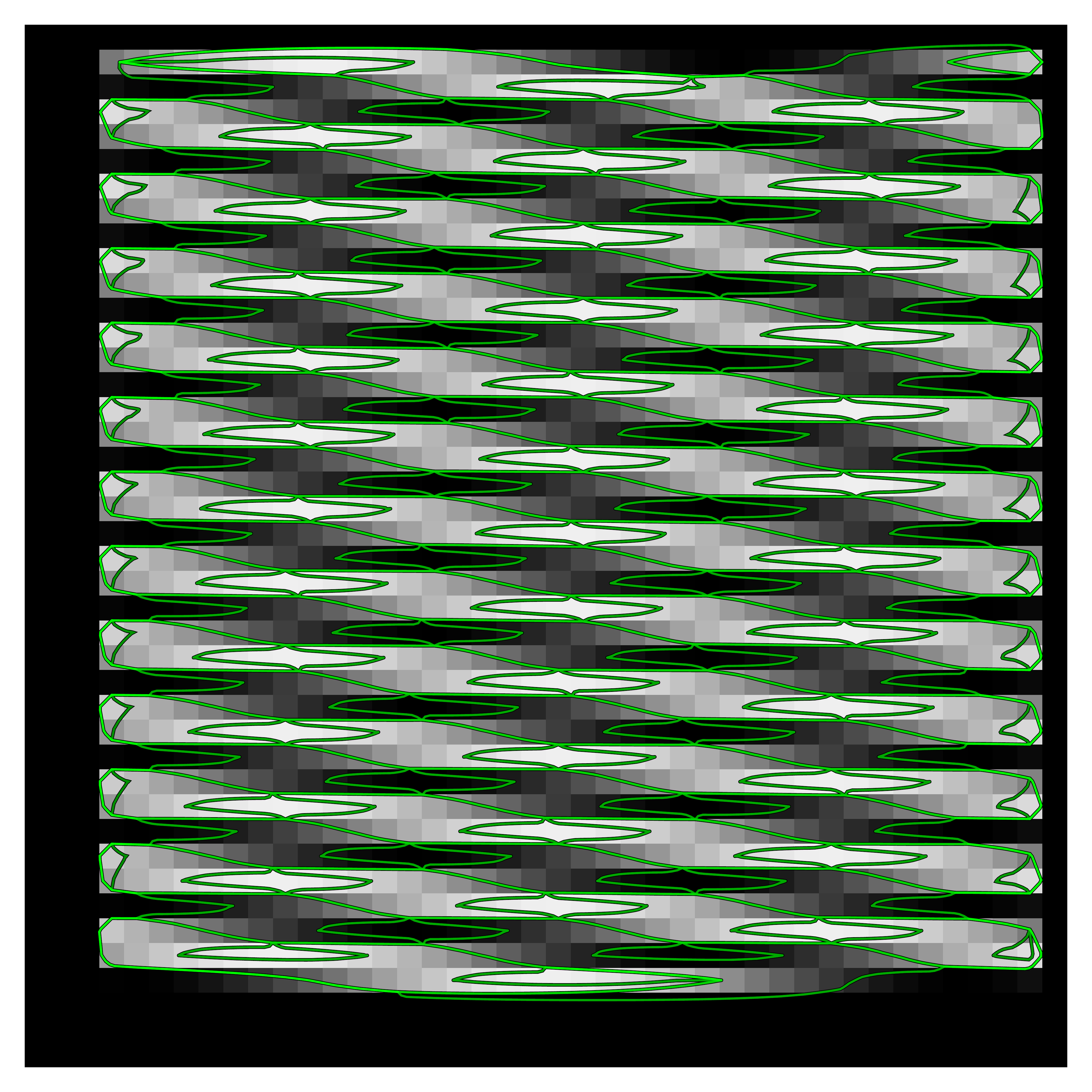}
\includegraphics[width=0.30\textwidth]{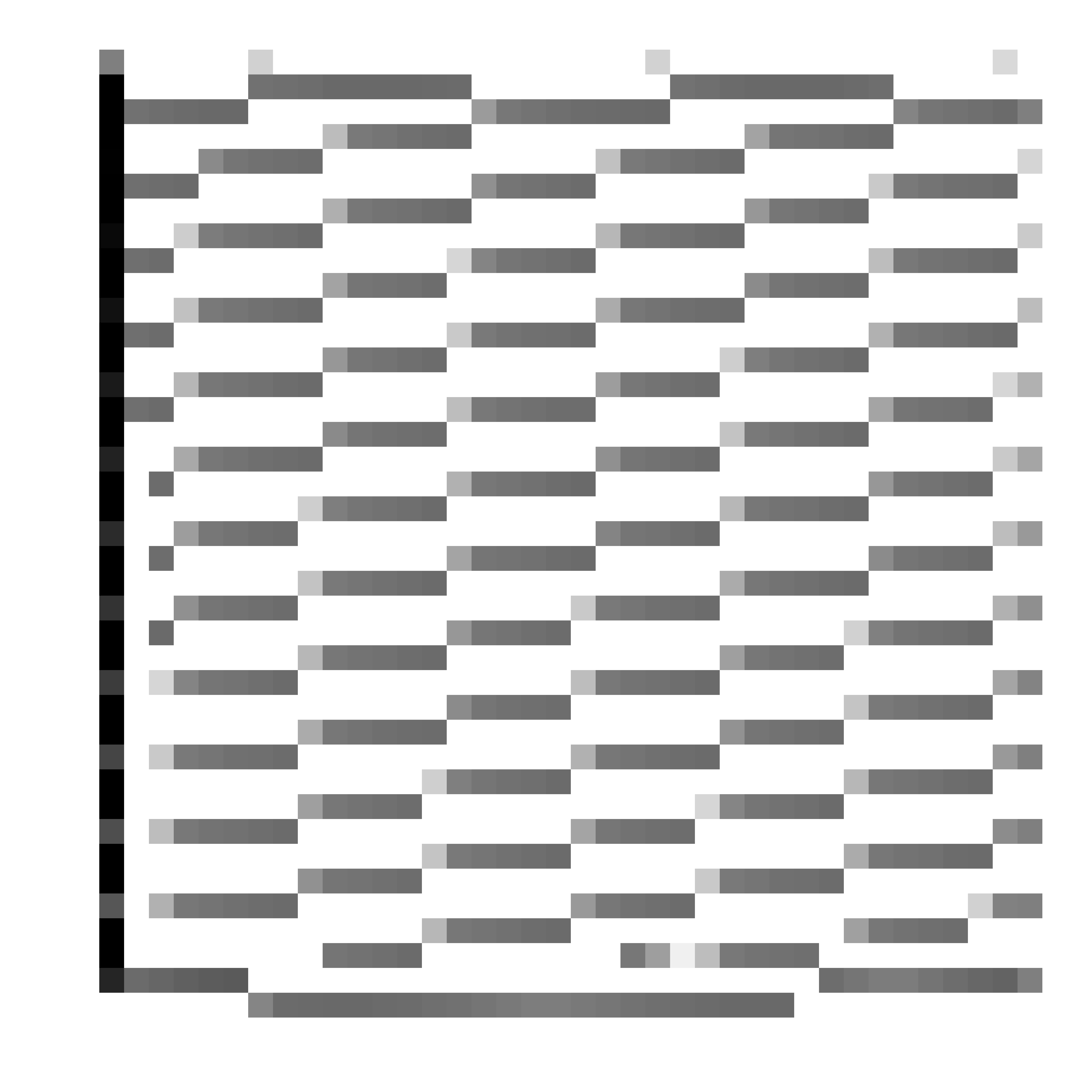}

\includegraphics[width=0.32\textwidth]{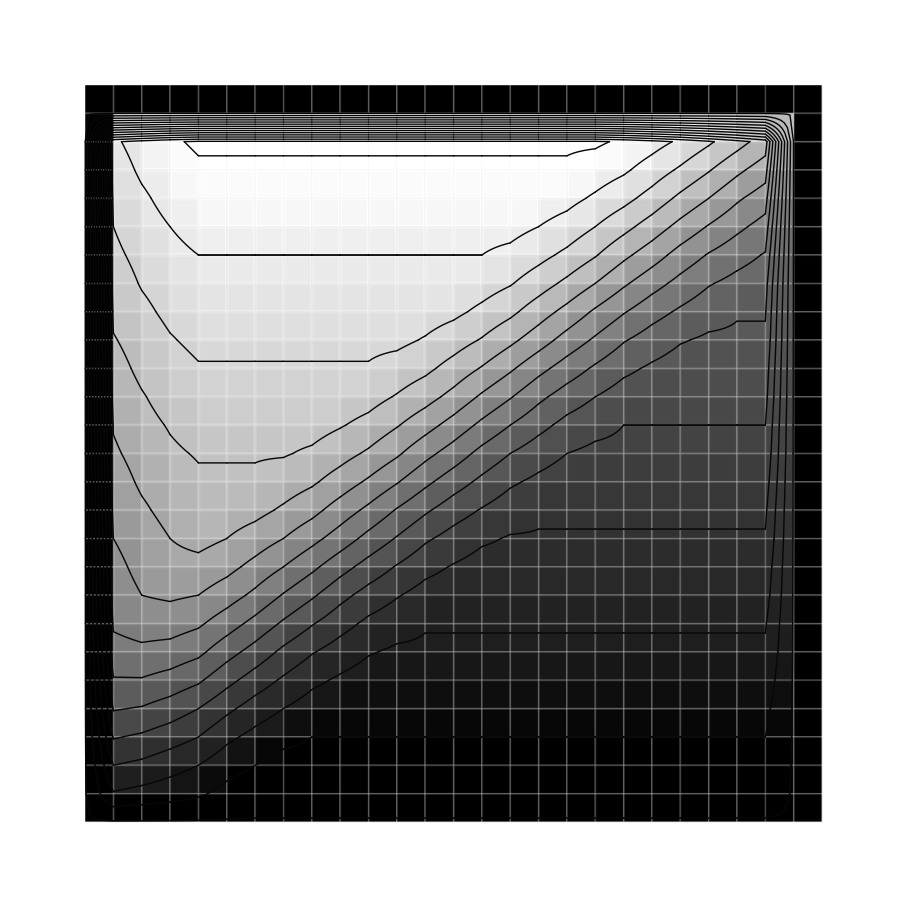}
\includegraphics[width=0.32\textwidth]{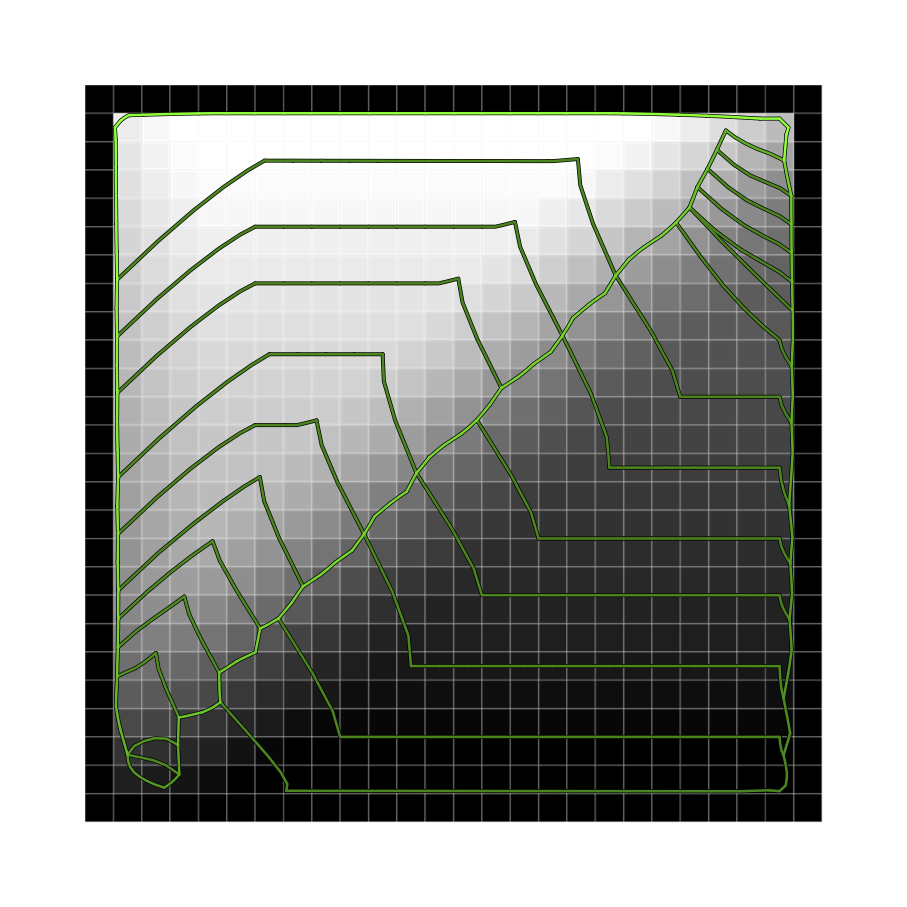}
\includegraphics[width=0.32\textwidth]{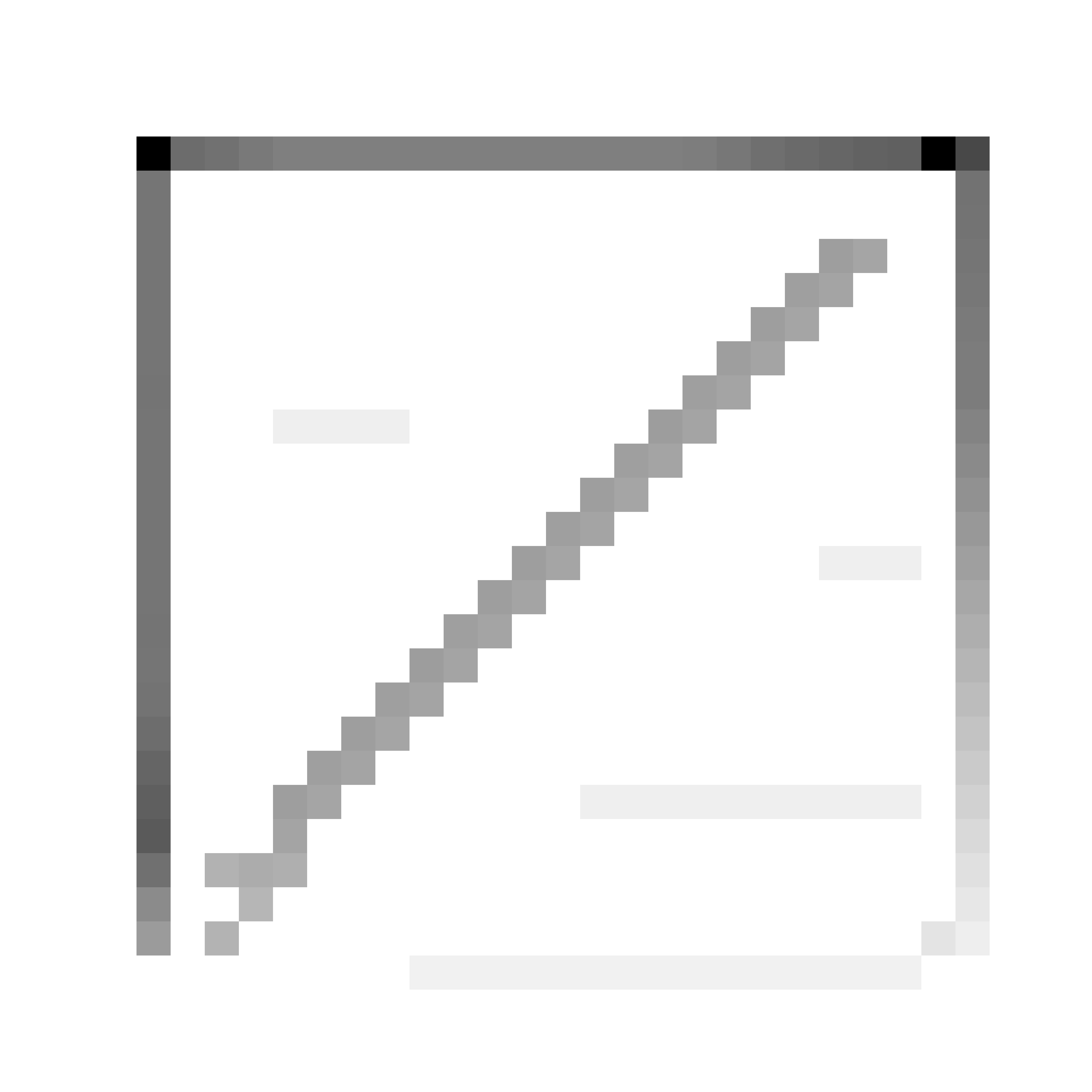}

\caption{Tests with 24x24 pixels images. Original image and isolines (left), original image and drawn edges (center), Canny's detector (right). 
Top row: Sin wave. Central row: rotated 1px thick lines. Bottom row: diagonal edge   \label{fig:test2}}
\end{center}
\end{figure*}

For Lena's image (Figure~\ref{fig:test3}) we show a comparison between our detection and Canny's. 
We use the simplified version of the steepest graph and we also 
post-process the edge graph in order to reduce the impact of local minima in gradient magnitude along the paths. In particular, we retain only local minima that are at least 2/3 less of neighbor local maxima's magnitude along the path. This is suitable for natural images, where the original and unfiltered noise can easily break an edge into two parallel ones (that can be anyway seen along the vertical edge on the left in the background). Consider that any post-processing based on local features of the connectivity graph, rather than on a local threshold, can produce a better filtering of such cases (this would be impossible with a single $\sigma$ parameter in Canny's detection).
Connections are always recovered, even in fuzzy positions. The edge drawing produce high quality sub-pixel lines. A further processing can be applied for vectorial smoothing and simplification for higher level processing. 

An example of a structured image with buildings is presented in Figure~\ref{fig:test5}. Note how the finest details with one pixels width are recovered and correctly connected, even for small gradient magnitudes. 

\begin{figure*}[ht]
\begin{center}
\includegraphics[width=0.32\textwidth]{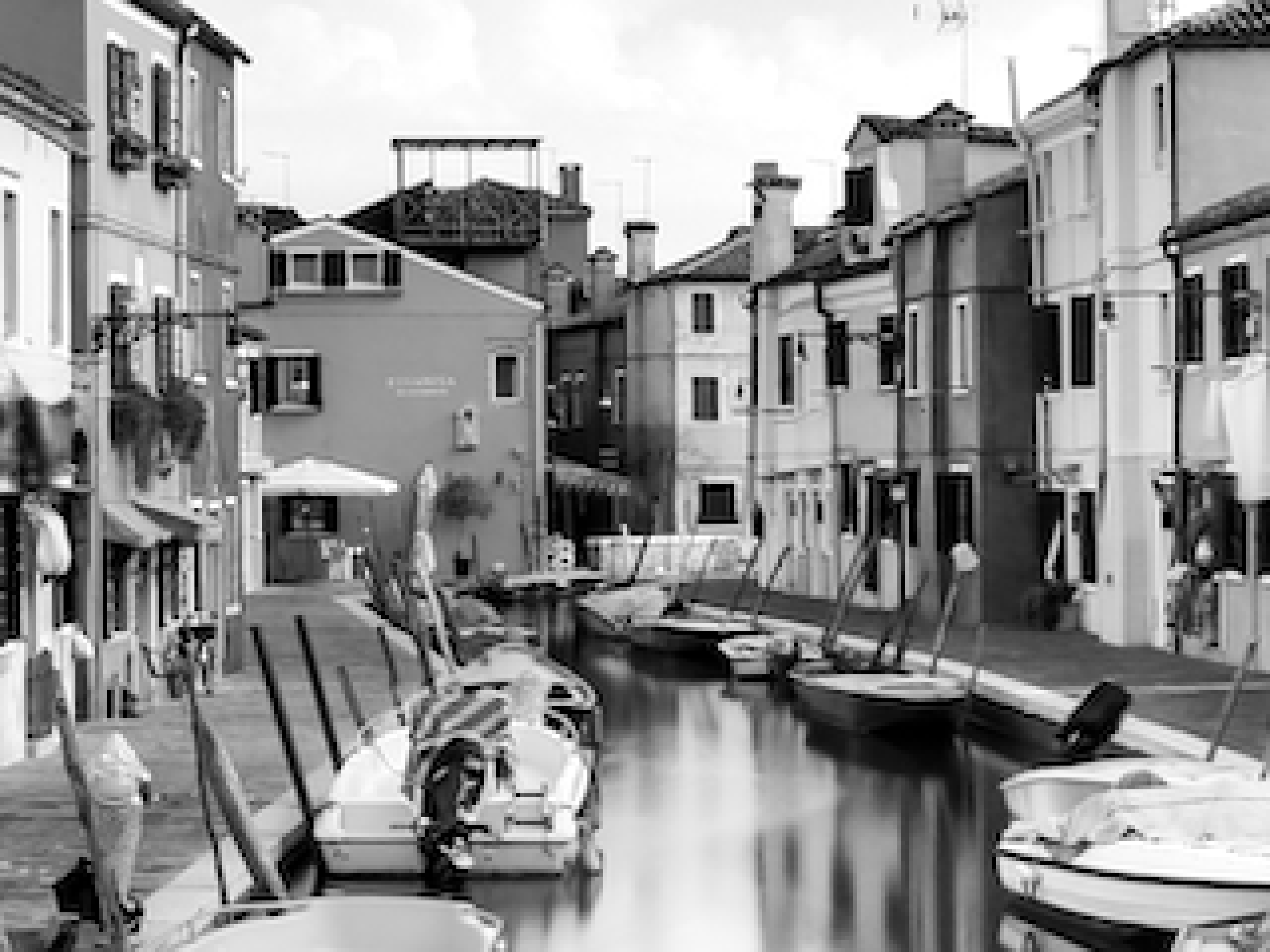}
\includegraphics[width=0.32\textwidth]{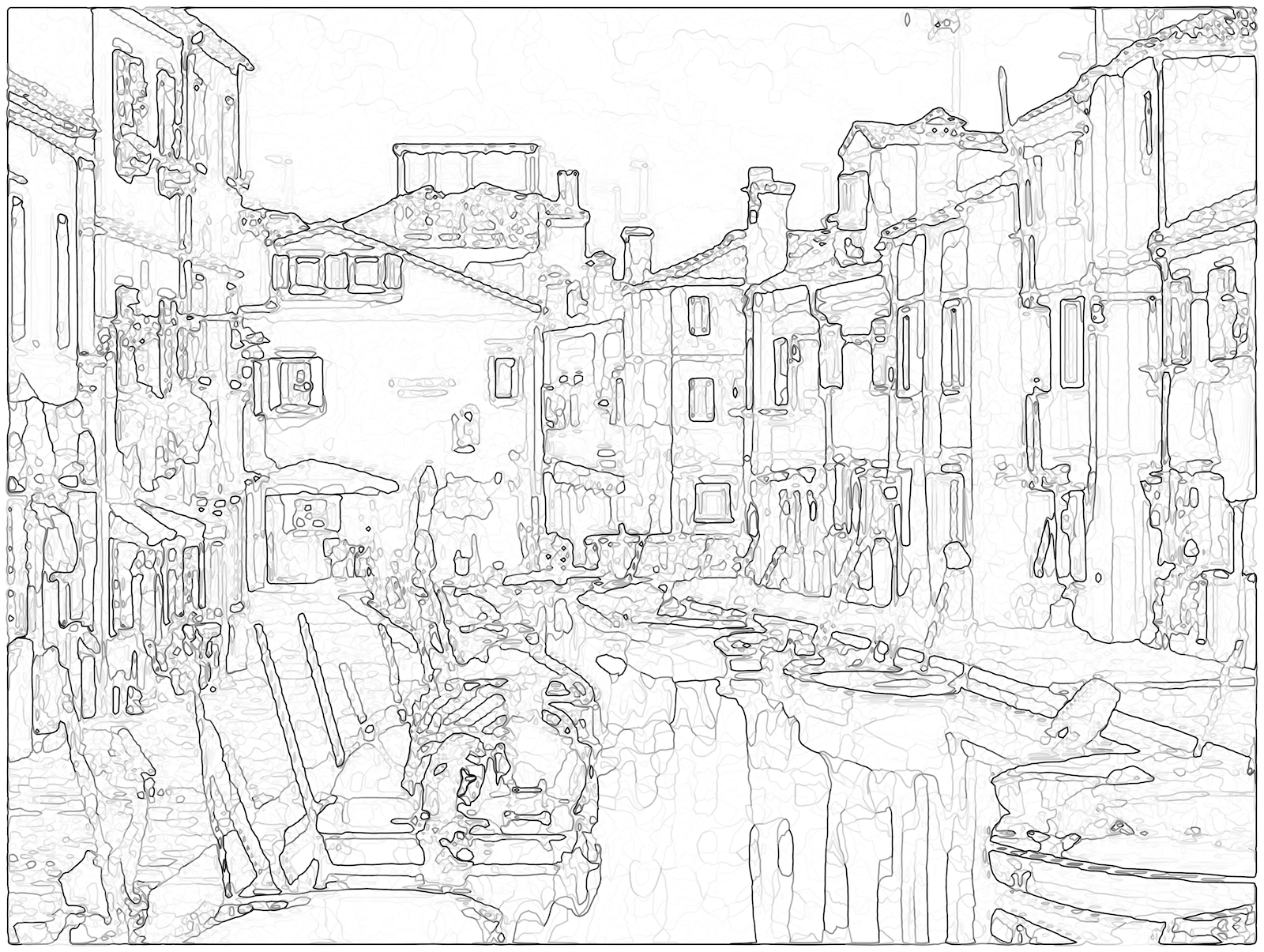}
\includegraphics[width=0.32\textwidth]{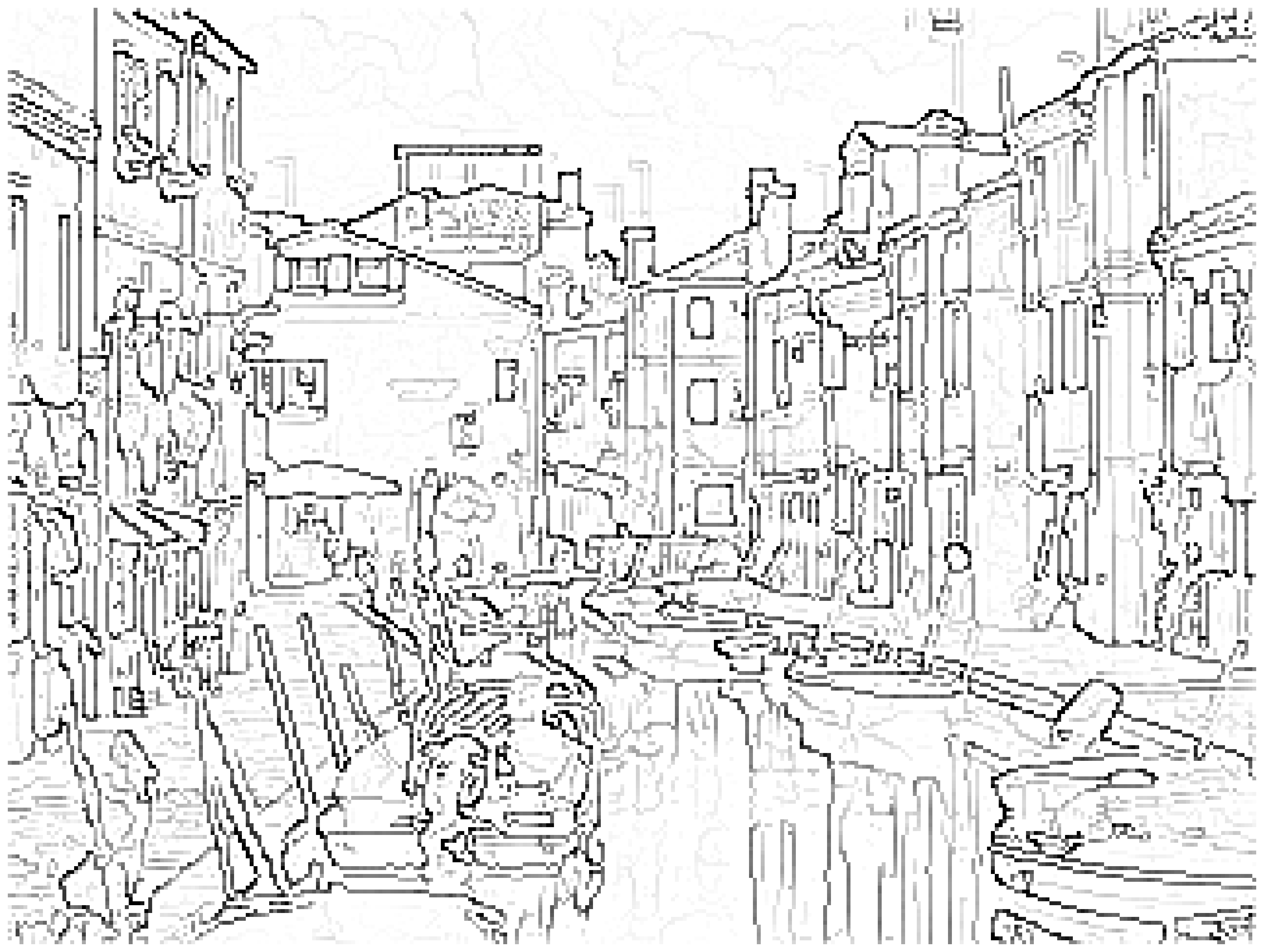}

\caption{Burano, Italy (320x240). Original image by Lopez Robin (left), drawn edges (center), Canny's detection (right)\label{fig:test5}}
%Photo by Lopez Robin on Unsplash.com
\end{center}
\end{figure*}

The BSDS500~\cite{arbelaez2011contour} contains an annotated dataset for training and validating learning based detection. The ground truth images contain a manually annotated rasterization of contours. The numerical comparison against such dataset would be unfair, since the ground truth is a perceptual consensus of the main edges, drawn on a raster image and designed for neural network training. Our vector and subpixel output would result in poor scores because of the greater accuracy and a pure signal-based detection. For completeness, in Figure~\ref{fig:test4} we show a visual comparison of a 240x160 crop of the image \#8068 from the dataset, where the precision of our detector can be appreciated. This suggests that high-quality low level information can actually feed the training phase for a perceptual selection of relevant features.

\begin{figure*}[ht]
\begin{center}
\includegraphics[width=0.32\textwidth]{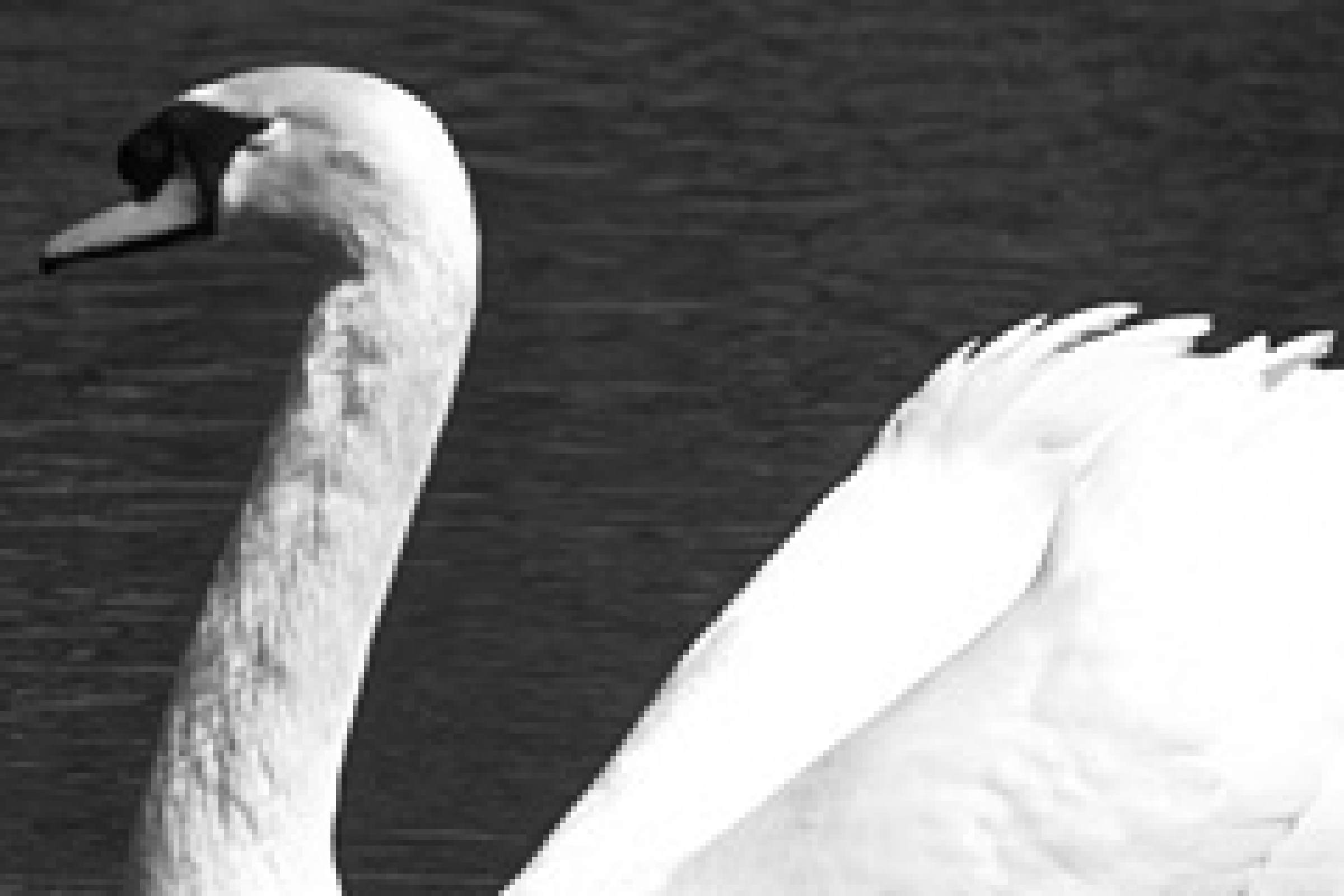}
\includegraphics[width=0.32\textwidth]{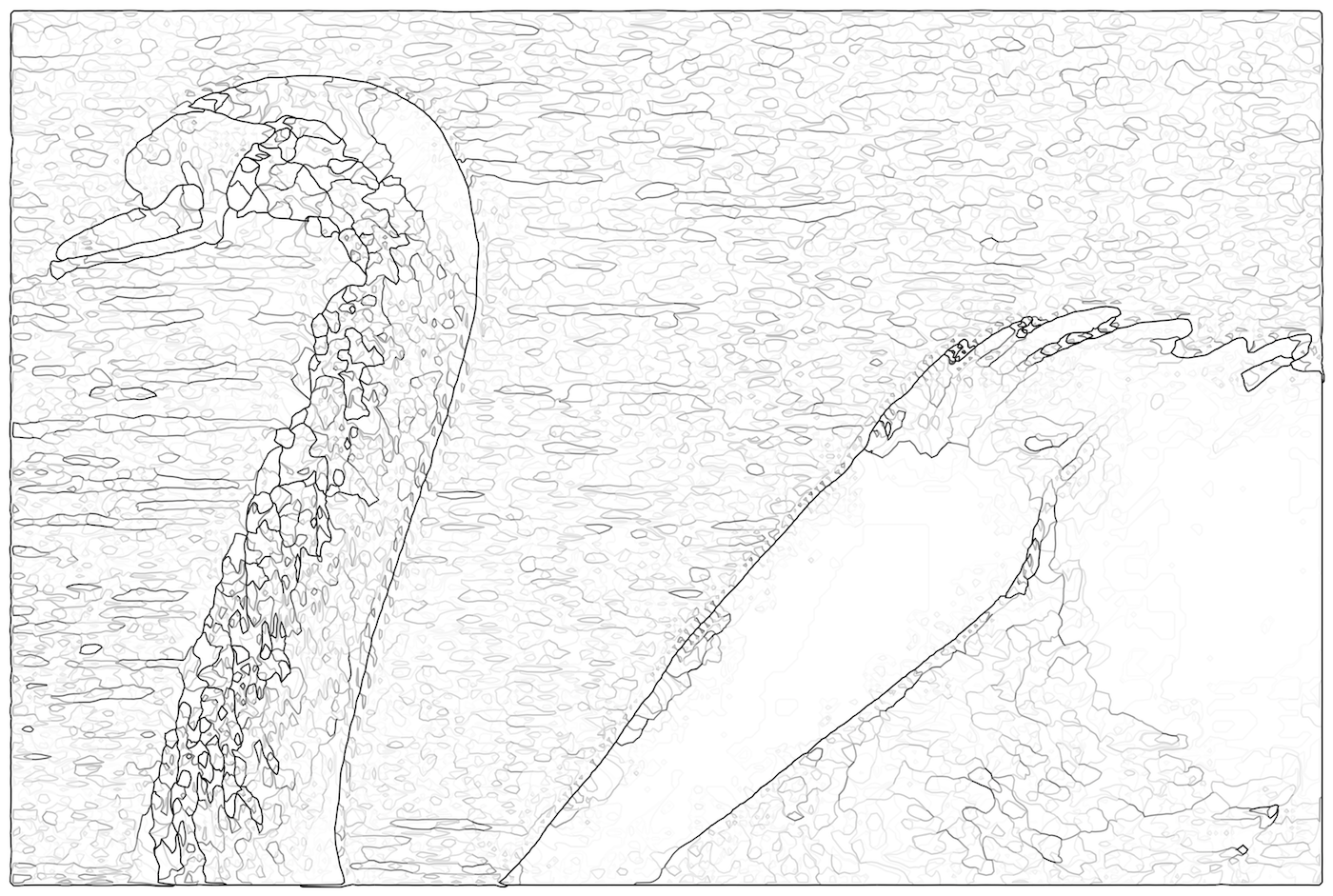}
\includegraphics[width=0.32\textwidth]{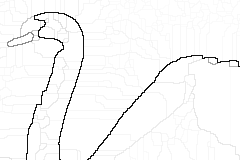}

\caption{Crop of picture 8068 from BSDS500 (240x160). Original image (left), drawn edges (center), ground truth image (from BSDS500) (right) \label{fig:test4}}
\end{center}
\end{figure*}

%From that page it is also possible to run tests on user selected images.

\section{Applications}
\label{sec:appl}

%The steepest graph allows to characterize edges as described in section~\ref{sec:em}.
The steepest graph structure allows to benefit from classical graph algorithms and to boost image processing capabilities. In particular, shortest path algorithm can help to selectively identify relevant edges in the image. Weighted cuts can identify edge based segmentation. Graph (sub-)isomorphism can lead to graph based vector pattern matching, 3d reconstruction from stereo images, image stitching etc..

A classification of saddles, in terms of areas contained by the isolines passing through them, allows to select which parts of the image can be perceptually merged (e.g. small peaks separated by a near split point). The graph represents the guidance for a controlled compression schema. This is one of the cases where fractal compression could perform better than other methods, since monotonic regions appears as perfect candidates for the identification of similar areas.
Saddles analysis can also be used for image denoising. The edge types contained in a region (e.g. small span ones) can suggest the presence of high frequency signal corruption. 

Edge features, span and support, can control a bi-dimensional decisional space for identifying blurred edge (large span and large support) vs sharp edges (small span and large support). This can control a selective image filtering, such as selective edge preserving smoothing, upsampling and area cutting while preserving the alpha blending of contours.

Edges and monotonic regions induce a rich segmentation of the image, that can guide contour detection, segmentation and feed for neural networks for high quality feature and object detection.

\section{Conclusions}
\label{sec:conc}

The paper presents an image model that allows a vectorial description of the underlying structure. The discrete construction is based on steepest graphs, which define a partition of the image into monotonic regions. The edge model stems from regions properties and it captures spatial arrangement, connectivity, span and support of each edge.
%The approximation errors due to integer approximation of steepest paths has an irrelevant impact on edge positions and no influence on edge connectivity.
The model produces accurate results starting from 2x2 images with single channel with any depth. The spatial resolution of drawn edges is natively sub-pixel, as opposed to detectors where the output is a raster image and/or discrete polyline. The correct connectivity is modeled and represented as an edge graph, as opposed to the majority of detectors.
The method does not require any image pre-processing, parameters nor prior knowledge. 
It is global, since the steepest graph construction connects local information and defines monotonic regions that are possibly influenced by the whole image. 

The graph edge is a novel description of edge features and it can be used as input of many graph processing algorithms for 
identification of rich higher-level features (segmentation, shape detection, feature matching, etc.).
The experimental section shows the results on a selection of tough synthetic cases and natural images.
As future work we plan to extend the model to the three-dimensional case and to combine multi-channel information.

\begin{acknowledgements}
This work is the result of the last 15 years of passionate research. I would like to thank Michele Petterlini for useful discussions and support along the journey. I also would like to thank my wife Fiammetta Maria Rossi for her patience and help. 
\end{acknowledgements}

\bibliographystyle{plain}       % APS-like style for physics
%\bibliography{bib}   % name your BibTeX data base

\clearpage
\section*{Appendix: Proofs}

\setcounter{lemma}{0}

\begin{lemma}[Saddle point and isolines]
A saddle point is positioned at $p=(x,y)$ iff there are two distinct and equally valued isolines that intersect at $p$.
\end{lemma}

\begin{proof}

($\Rightarrow$) By definition $p_1, p_4, p_2, p_3$ are in clockwise order around $p$. Since $R$ is continuous and $R(p_1)<R(p)<R(p_4)$, there is a point $p'$ on the segment $p_1, p_4$ such that $R(p)=R(p')$. The same can be said for any pair $p_1'$ and $p_4'$ respectively along the curves between $p_1,p$ and $p,p_4$. Therefore, an isoline is detected between the two curves.
The same can be applied to the other three cases.

($\Leftarrow$) Let us consider the three cases:
(1) $p$ is inside a pixel. Since the pixel is bilinear interpolated, there are two distinct lines, only if $\nabla R(p)=\vec 0$. Therefore the pixel is a split pixel and the four corners can be assigned to $p_i$s. The segments from corners to $p$ fulfil the saddle property.
(2) $p$ lies on a pixel side. In this case there is at most one isoline per pixel and therefore it is not possible to identify two distinct lines crossing $p$.
(3) $p$ lies on a corner. Each adjacent pixel to $p$ has one isoline reaching $p$. The $n4(p)$ neighbors of $p$ can be associated to $p_i$s and the corresponding segments towards $p$ fulfil the saddle property.
\qed
\end{proof}

\begin{lemma}[Split pixel corners value]
Given a split point $p=(x,y)$, associated to the split pixel $(i,j)$, it can either be $I(i,j)>R(p)$,$I(i+1,j+1)>R(p)$,$I(i+1,j)<R(p)$ and $I(i,j+1)<R(p)$
or the same relationships with inverted order. Informally, two opposite corners are greater than $R(p)$.
%\hl{qui dimostro split pixel -> alternating corners. Dovrei dimostrare alternating corners -> split pixel?}
\end{lemma}
\begin{proof}
Let us assume that $(i,j)=(0,0)$
Recall that
$v_{0,0}=a, v_{1,0}=b-a, v_{0,1}=c-a$ and $ v_{1,1}=d+a-b-c$,
$R(x,y)=v_{0,0}+v_{1,0}x+v_{0,1}y+v_{1,1}xy$

A split pixel implies that $\nabla R(x,y)=\vec 0$ and $v_{1,1}\neq 0$. Let us assume that $a>R(p)$.
We show that $d>R(p)$, $b<R(p)$,  $c<R(p)$. The case  $a<R(p)$ is symmetrical.

$\nabla R(x,y)=\vec 0$ implies that
$v_{1,0} + v_{1,1}y =0$ and $ v_{0,1} + v_{1,1}x=0$.
This means that $R(x,y)=R(x,y')$ for any $y'\in [0,1]$, since the gradient along the y axis is constantly equal to 0. 
Equally, $R(x,y)=R(x',y)$ for any $x'\in [0,1]$.

In particular $R(x,y)=R(0,y)=R(1,y)=R(x,0)=R(x,1)$,
which implies that $R(0,0)=a>R(x,y)=R(x,0)$. Since $R(x,0)$ is a linear interpolated value
between $R(0,0)$ and $R(1,0)$, it follows that $R(0,0)>R(1,0)$.
With equivalent arguments, $R(0,0)>R(0,1)$ and $R(0,1)<R(1,1)$.
\qed
\end{proof}

\begin{lemma}[There is at most one isoline that passes inside a pixel and on one of its corners]
Given a pixel with corner $p$, the isoline passing through $p$ solves the equation $R(x,y)=R(p)$. Given the bilinear interpolation, the line is unique and there is at most one line that intersects the pixel area. 
\end{lemma}

\begin{proof}
W.l.o.g. let us assume that the pixel is placed in $p=(0,0)$, so that the corner is the bottom left.
The isoline solves the equation  $R(x,y)=v_{0,0}+v_{1,0}x+v_{0,1}y+v_{1,1}xy = v_{0,0}$, where $v_{0,0}=a, v_{1,0}=b-a, v_{0,1}=c-a$ and $ v_{1,1}=d+a-b-c$. It follows that the function describing the isoline is $f(x)=-v_{1,0}/(v_{0,1}+v_{1,1}) x$. The function, except for its asymptotic point in $x=-v_{0,1}/v_{1,1}$ is continuous. Moreover if $f'(x)>0$, it follows that $v_{1,0}v_{0,1}<0$ and the line crossing (0,0) enters the pixel. If $f'(x)<0$, $v_{1,0}v_{0,1}>0$ and the function intersects the pixel only at the corner.
\qed
\end{proof}

\begin{lemma}[Mix point characterization]
The following statements are equivalent:

(i) A mix point $p$ exists at integer coordinate

(ii) each of the four pixels adjacent to $p$ contains an isoline with value $I(p)$ that reaches the pixel border at $p'\neq p$

(iii) $n4(p)$ neighborhood is such that $W=I(0,1)-R(p)$, $E=I(2,1)-R(p)$, $S=I(1,0)-R(p)$, $N=I(1,2)-R(p)$, $NS>0$, $EW>0$, $NE<0$
\end{lemma}

\begin{proof}
(i) $\rightarrow$ (ii)

By definition there are two lines $\ell_1$ and $\ell_2$ delimited respectively by $(p_1, p_2)$ and $(p_3,p_4)$. 
If the points $p_i$ are outside the area obtained by the union of the four pixels containing $p$, they can be trimmed and placed on the intersection of the lines and the area's perimeter. 
Since the function is continuous, it follows that there must be four distinct isolines starting from $p$ that interleave the increasing paths. 
In fact, for any consecutive pair (clockwise selected),e.g. $p_1,p_3$, there must be a point $\overline{p}$ along the segment $p_1,p_3$ such that $R(\overline{p})=R(p)$. Since all points are distinct, it follows that also  $\overline{p}$ is inside the pixel and it causes an isoline to be inside the pixel. Since each pixel can host at most an isoline with value $R(p)$, each pixel contains exactly an isoline which extends inside the pixel.

(ii) $\rightarrow$ (iii)
In the absence of split pixels and having a continuous function on each pixel, the isoline through $p$ divides the pixel into two areas that contain the isolines greater (lesser) than $I(p)$. In Figure~\ref{fig:mix}, we depict in red the greater isolines and in blue the lesser isolines. It follows that (I(0,1)-R(p)) has opposed sign to (I(1,2)-R(p)) and the property (iii) holds.

(iii) $\rightarrow$ (i)
W.l.o.g. let us assume that $(I(0,1)-R(p))>0$.  
We can assign $p_1=(1,0)$,  $p_2=(1,2)$,  $p_3=(0,1)$,  $p_3=(2,1)$ and create $\ell_1$ and $\ell_2$ accordingly. Since each side is linearly interpolated, we have four monotonic lines and a mix point in $p$.
 \qed
\end{proof}

\begin{figure}[ht]
\begin{center}
\includegraphics[width=0.20\textwidth]{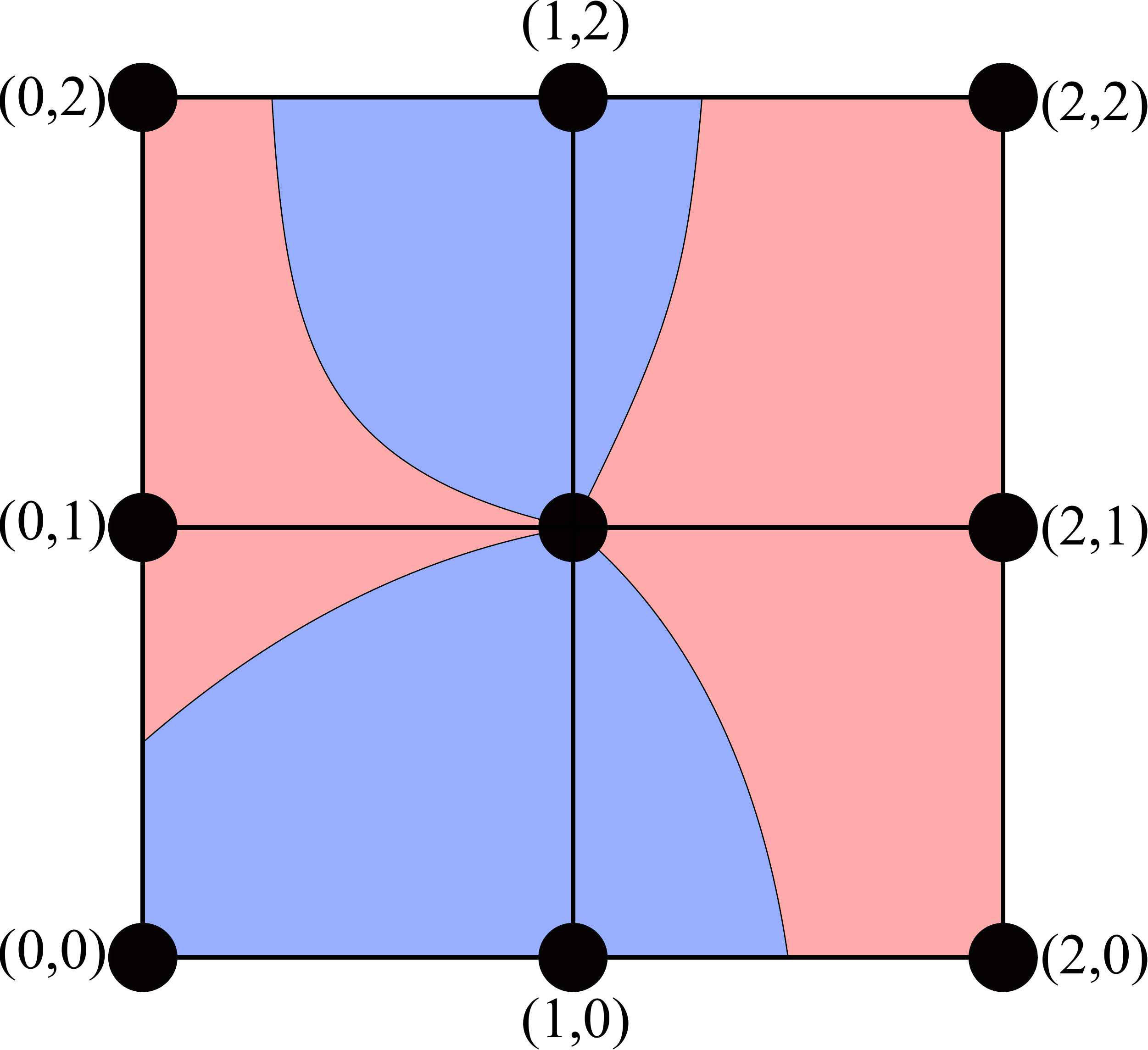}
\caption{Mix point relationships with isolines\label{fig:mix}}
\end{center}
\end{figure}

\begin{lemma}[Non minima joining]
Given a non minimal node $p$ with no incoming edges and the graph built by steps 1,2 and 3, there is a new edge that does not intersect a diagonal edge and that connects the node to a lower valued node.
\end{lemma}

\begin{proof}
Since $p$ is not a local minimum, it exists at least another point $p'\in n8(p)$ such that $p'\prec p$.
Let us consider the case where the segment $(p',p)$ is a diagonal crossing the pixel and the graph already contains the  other green diagonal ($p_1,p_2$) as shown in Figure~\ref{fig:connection} right.
Since $p_2$ was the best choice for $p_1$, it follows that $p\prec p_2$ and $p'\prec p_2$. Moreover, since we could consider a diagonal edge, it follows that the pixel does not contain a split point.
We have two cases: (blue) $p'\prec p_1$. Since there is no split pixel it follows that $p_1\prec p$. If $p_1\prec p'$ (red), since $p'\prec p$, it can not be that $p\prec p_1$, since all values are distinct. It follows that necessarily $p_1\prec p$. In both cases the edge $(p_1,p')$ is a valid candidate.
 The proof for the other diagonal orientation $(p_2,p_1)$ is symmetrical.
 \qed
\end{proof}

\begin{lemma}[DAG]
The graph is a directed acyclic graph. 
\end{lemma}

Each edge $(p_1,p_2)$ in the steepest graph connects two neighbors nodes such that $p_1\prec p_2$. Therefore, no cycles are allowed, since each edge connects different and ordered values nodes. \qed

\begin{lemma}[Planar graph]
The graph is planar
\end{lemma}

\begin{proof}
The nodes embedding in the plane is straightforward, since nodes are associated to discrete pixel points. It is sufficient to show that no edges cross each other. 
The only possible edges intersection can happen on the two diagonals of a pixel. By construction, steps 1 and 2 introduce no diagonals. Step 3 can not introduce two crossing diagonals over the same pixel, since the two different corners on a pixel could not identify different nodes as corresponding best next nodes, since they both include the other pixels corners in their $n8()$ neighborhood. Therefore, it can not be the case that crossing edges are selected.
Step 4 is performed incrementally and guarantees the absence of crossing diagonals.
 \qed
\end{proof}

\begin{lemma}[All saddles are connected]
The graph contains, for each saddle, at least a path that connects it to a local extrema.
\end{lemma}

\begin{proof}
Since split points are not represented by a node in the steepest graph, let us map the property to each of the four corners of a split pixel. In particular, we state that the two greater (lower) than split value nodes are connected to a local maxima (minima). 

The case starting from greater than split value nodes is trivial, since the discrete increasing steepest path reaches a local maxima.
It is possible to recursively expand, in a depth first search style, the reversed edges incoming to the node. If no steepest paths are able to reach a local minimum, the edge addition of step 4 allows to extend the search to other steepest paths. Eventually the additions reach a minimum, since the image is finite and every edge decreases the image value. 

In case of a mix point node, the arguments are equivalent, since there are two greater (lower) valued points in the $n4()$ neighborhood that are added at step 2.
 \qed
\end{proof}

\begin{lemma}[Correct connectivity]\label{lem:connect}
 A path starting from a local minimum and ending to a local maximum in the graph is a monotonic increasing path embedded in R.
 \end{lemma}
\begin{proof}
We inductively proof the statement. Starting from a minimum, this is a collapsed increasing path. Any edge extending the previous path, which already is a monotonic path in R, maintains the property. If the edge $(p_1,p_2)$ covers a pixel side, the bilinear interpolation trivially produces a monotonic increasing path along the corresponding segment in $R$. If the edge covers a diagonal, it follows that there is no split pixel. Moreover, $p_1\prec p_2$. If there is a isoline of value $R(p_1)$ inside the pixel, then the isolines lands on the pixel perimeter and from there, a monotonic increasing path can be found around the perimeter, until $p_2$ is reached. If there is no isoline of value $R(p_1)$ inside the pixel, the segment $(p_1,p_2)$ is a monotonic increasing path in $R$.
 \qed
 \end{proof}
 
\begin{lemma}
A monotonic region contains no self intersecting isolines.
 \end{lemma}

A monotonic region contains no saddles by construction. A non saddle point inside a pixel is traversed by a single isoline. A point on the pixel side can be traversed by two isolines (one on each neighboring pixel) that merge at the side. A non mix-point corner contains two isolines that connect at the corner. Note that three isolines on three out of four pixels containing the corner impose that the fourth pixel contains an isoline (by continuity of the four piecewise bilinear functions).
It follows that no intersecting isolines are present in a monotonic region.

\begin{lemma}
The non collapsed part of a monotonic region has exactly one minimum and one maximum in image values on its perimeter.
 \end{lemma}
\begin{proof}
The perimeter is composed of differently oriented graph edges. Since the region contains no saddles and $R$ is continuous, all isolines in the region connect two specific points on the perimeter.  It follows that there is a total ordering of two subsets of perimeter points. In particular there is a minimum and a maximum on the perimeter, that correspond to collapsed isolines.
 \qed
 \end{proof}

\end{document}